\documentclass[10pt,a4paper]{article}
\usepackage{lipsum}
\usepackage{amsfonts}
\usepackage{graphicx}
\usepackage{epstopdf}
\usepackage{algorithmic}
\ifpdf
  \DeclareGraphicsExtensions{.eps,.pdf,.png,.jpg}
\else
  \DeclareGraphicsExtensions{.eps}
\fi

\title{WPPNets and WPPFlows: The Power of\\Wasserstein Patch Priors
 for Superresolution}

\author{{Fabian Altekr\"uger\footnotemark[2] }\thanks{Department of Mathematics, 
Humboldt-Universit\"at zu Berlin, 
Unter den Linden 6, 
D-10099 Berlin, Germany,
fabian.altekrueger@hu-berlin.de}
\and Johannes Hertrich\thanks{Institute of Mathematics,
Technische Universit\"at Berlin,
Stra{\ss}e des 17. Juni 136, 
D-10623 Berlin, Germany,
j.hertrich@math.tu-berlin.de.}
}

\usepackage{amsopn}

\usepackage[utf8]{inputenc}
\usepackage{amsmath,amssymb,amstext}
\usepackage{mathtools}
\usepackage[ngerman, english]{babel}
\usepackage[numbers]{natbib}
\usepackage{amsthm}
\usepackage{float} 
\usepackage{subcaption}
\usepackage{xcolor}
\usepackage[onehalfspacing]{setspace}
\usepackage{geometry}
\usepackage[hidelinks]{hyperref}
\usepackage[symbol]{footmisc}
\mathtoolsset{showonlyrefs}
\usepackage{animate}
\DeclareMathOperator*{\argmin}{arg\,min}
\DeclareMathOperator*{\argmax}{arg\,max}

\newcommand{\R}{\mathbb{R}}
\newcommand{\C}{\mathbb{C}}

\newcommand{\E}{\mathbb{E}}
\newcommand{\tT}{\mathrm{T}}
\newcommand{\KL}{\mathrm{KL}}

\theoremstyle{plain}
\newtheorem{lemma}{Lemma}
\newtheorem{theorem}[lemma]{Theorem}

\newtheorem{proposition}[lemma]{Proposition}
\newtheorem{remark}[lemma]{Remark}

\theoremstyle{definition}

\begin{document}
\maketitle

\renewcommand{\thefootnote}{\arabic{footnote}}
\begin{abstract}
Exploiting image patches instead of whole images have proved to be 
a powerful approach to tackle various problems in image processing.
Recently, Wasserstein patch priors (WPP), 
which are based on the comparison of the patch distributions of the unknown image
and a reference image,
were successfully used
as data-driven regularizers in the variational formulation
of superresolution.
However, for each input image, this approach requires the solution of a 
non-convex minimization problem  which is computationally costly.
In this paper, we propose to learn two kind of neural networks in an unsupervised way
based on  WPP loss functions.
First, we show how convolutional neural networks (CNNs) can be incorporated.
Once the network, called WPPNet, is learned, it can be very efficiently applied to any input image.
Second, we incorporate conditional normalizing flows to provide a tool for uncertainty quantification.
Numerical examples demonstrate the  very good performance of WPPNets for superresolution 
in various image classes even if the forward operator is known only approximately.
\end{abstract}

\section{Introduction}

In inverse problems, the task is to reconstruct an unknown ground truth $\bar x$ from a noisy observation 
\begin{equation} \label{model}
y = f(\bar x) + \xi,
\end{equation}
where $f$ is an ill-posed forward operator and $\xi$ is the realization of Gaussian noise with distribution $\mathcal N(0,\sigma^2 I)$.
Such problems can be tackled by finding a minimizer of a variational problem
\begin{equation}\label{general_variational_problem}
\mathcal J(x) = \mathcal D(f(x),y) + \lambda\mathcal R(x), \qquad \lambda > 0,
\end{equation}
where $\mathcal D(f(x),y)$ is a data-fidelity term which accounts for the noise 
and $\mathcal R$ a regularizer or image prior.
The concrete form of \eqref{general_variational_problem} is often derived by a Bayesian approach.

In this paper, we focus on the problem of superresolution. 
Here, $\bar x$ is a high-resolution image and $y$ a low-resolution image
obtained by a forward operator $f$, which is usually a composition of a blur operator and a downsampling operator.
Due to its actuality, superresolution
using deep neural networks (NNs) was considered in many papers, 
see e.g.,~\cite{DLHT2015,LTHC2017,LSKNM2017, RIM2017,SHCSFN2021,Tian21,WYWG2018,ZTKZF2018}, see also \cite{WCH2020} for a survey. 
In particular, unrolled approaches like \cite{GL2010,JP2021,MMK2021} yield good results.
However, within their training process, these methods require the 
access to a large amount of registered pairs $(\bar x_i,y_i)$ of high- and low-resolution images.
Only very few NN based approaches incorporate knowledge about the forward operator $f$ and the underlying image domain instead of large training data sets.
Examples are 
the zero-shot superresolution \cite{SCI2018},
the deep image prior \cite{UVL2018} and
Plug-and-Play methods \cite{VBW2013}.
Zero-shot superresolution \cite{SCI2018} exploits the internal image statistics and learns image-specific relations between the low-resolution image and its downscaled versions in order to apply these relations for reconstructing the high-resolution image.
The Deep Image Prior (DIP) \cite{UVL2018} learns a CNN by minimizing the loss function
$$
\mathcal L_{{\text{DIP}}}(\theta) \coloneqq \|f(G_\theta(z))-y\|^2,
$$
where  $z$ is a randomly chosen input and an early-stopping technique is used as regularization.
Then, the reconstructed image is obtained by $x=G_\theta(z)$.
It was shown in \cite{UVL2018} that DIP admits competitive results for many inverse problems. 
In \cite{BLS20} a combination of DIP with the TV regularizer (DIP+TV), leading to the loss function
\begin{equation}\label{eq_DIP:Loss}
\mathcal L_{{\text{DIP+ TV}}}(\theta) \coloneqq \|f(G_\theta(z))-y\|^2 + \lambda \text{TV}(G_\theta(z)), \quad \lambda >0,
\end{equation}
was successfully proposed for computerized tomography.
Note that each reconstruction with DIP+TV requires the training of a NN, which makes the method time consuming.
The idea of Plug-and-Play methods is to consider an optimization algorithm 
from convex analysis for solving \eqref{general_variational_problem}, as, e.g.,\ the 
forward-backward splitting \cite{CW2005} or the alternating direction method of multipliers \cite{CKCH2019,EB1992},
and to replace the proximal operator with respect to the regularizer by a more general denoiser.
For example, the forward backward splitting algorithm for minimizing \eqref{general_variational_problem} is given by
\begin{equation*}
x_{r+1}=\mathrm{prox}_{\eta R}(x_r-\eta\nabla_x \mathcal D(x_r,y))
\end{equation*}
and can be modified by 
\begin{equation}\label{eq_PnP_FBS}
x_{r+1}= \mathcal G(x_r-\eta\nabla_x \mathcal D(x_r,y)),
\end{equation}
where $\mathcal G$ is a neural network trained for denoising natural images as the DRUNet from \cite{ZLZZ2021}, and the hyperparameter $\eta>0$ is the step size.
Plug-and-Play methods were used for several applications in image processing with excellent performance, 
see e.g.,~\cite{CWE2016,GJNMU2018,HNS2021,MMHC2017,O2017,ZLZZ2021}.
Closely related to Plug-and-Play methods are regularizing by denoising (RED) \cite{REM2017}, variational networks \cite{EKKP2020} 
and total deep variation \cite{KEKP2020,PKPE2021}.

On the other hand, powerful methods based on the self-similarity of small patches within
natural images were developed in the last years. 
Such patch-based methods were used, e.g., for denoising \cite{BCM2005,HBD2018,LBM2013,STA2021} 
and superresolution \cite{HNABBSS2020,SJ2016}.
In particular, the expected patch log-likelihood algorithm (EPLL) \cite{PDDN2019,ZW2011} achieved competitive results to 
simple NN-based approaches by using the log-likelihood function of a Gaussian mixture model as regularizer within the 
variational problem \eqref{general_variational_problem}. 
Recently, this approach was also extended to posterior sampling for uncertainty quantification \cite{FW2021}.

In this paper, 
we assume that we are given 
the (approximate) forward operator $f\colon\R^{d_1\times d_2}\to\R^{n_1\times n_2}$ consisting of a blur and a downsampling operator, where $d_1=q n_1$ and $d_2=q n_2$ for some integral magnification factor $q$.
Additionally, we assume that we are given a database of low-resolution images $y_1,...,y_m$ 
and one single high-resolution reference image $\tilde x$.
Further, we assume that the distribution of patches in $\tilde x$ is similar to the patch distribution in the 
unknown high-resolution ground truths $\bar x_1,..., \bar x_m$ corresponding to $y_1,...,y_m$.
This assumption is fulfilled for images from similar image classes as,
e.g., from textures or materials.
The setting is motivated by the analysis of materials microstructures.
Here, it is usually possible to scan a large area of a material with a low-resolution, while the limited amount of time 
and resources forbids to image the same size of a data using a higher resolution.
Further, it is often impossible to scan the same section of one sample twice with different resolutions as in many applications 
destructive imaging processes, e.g.,\ FIB-SEM are used. 
These imaging techniques destroy the scanned section of the material and consequently there is no chance to 
generate paired training data.
Within this setting, Hertrich et al. \cite{Hertrich21} proposed  to solve the variational problem 
\eqref{general_variational_problem} using a so-called Wasserstein patch prior (WPP), which penalizes the quadratic Wasserstein distance between the patch distribution
of $x$ and the reference image $\tilde x$. 
The approach was inspired by the idea that a texture-like image can be represented 
by the distribution of small patches
\cite{EL1999, GRGH2017, Hou21Patch,Hou21, LB2019}, which was used by Guiterrez et al.~\cite{GRGH2017} 
to synthesize textures and was extended by Houdard et al.~\cite{Hou21Patch} to generative texture modelling.
Texture synthesis based on wavelet coefficients was considered in \cite{RPDB2011}.
Unfortunately, the minimization of the variational problem corresponding to the WPP requires several computations and differentiations of Wasserstein distances, which is computationally expensive.
We propose to overcome this computational overhead by exploiting NNs. 
First we introduce WPPNets, which are
NNs trained in an unsupervised way using a new WPP-based loss function.
Then, in order to measure the uncertainty within the reconstructions, 
we train conditional normalizing flows using a WPP-Kullback-Leibler-based loss function and call the resulting model WPPFlows.
Normalizing flows \cite{DSB2017,KD2018,RM2015} are learned diffeomorphisms to model complicated and high-dimensional probability distributions using a much simpler latent distribution. Applications to inverse problems were considered, e.g.,\ in \cite{AFHH2021,Denker21,Bouman20}.
We demonstrate the performance of our methods by various numerical examples on real-world data\footnote{The code is available at \url{https://github.com/FabianAltekrueger/WPPNets}.}.
It turns out that WPPNets outperform several methods such as Plug-and-Play approaches and DIP.
In practice, it is often unrealistic to assume that the forward operator $f$ is known exactly.
Instead, it can be described inaccurately or estimated on synthetic data.
Our numerical examples show that WPPNets are much more stable under such 
perturbations of the operator than other approaches.
Finally, our examples show that WPPFlows are a reasonable method to quantify the uncertainty within the reconstructions as they produce diverse and realistic images.

The paper is organized as follows:
in Section~\ref{Section_WPP}, we briefly review the WPP approach from \cite{Hertrich21}. 
Then, in Section~\ref{Section_WReg}, we introduce WPPNets.
Conditional normalizing flows were incorporated in Section \ref{Sec_UQ} to get WPPFlows.
In Section~\ref{Section_NumResults}, we present numerical examples and compare our results with other methods.
Finally, conclusions are drawn in Section~\ref{Section_conclusion}.

\section{The Wasserstein Patch Prior}\label{Section_WPP}

\begin{figure}
    \centering
    \includegraphics[width=\textwidth]{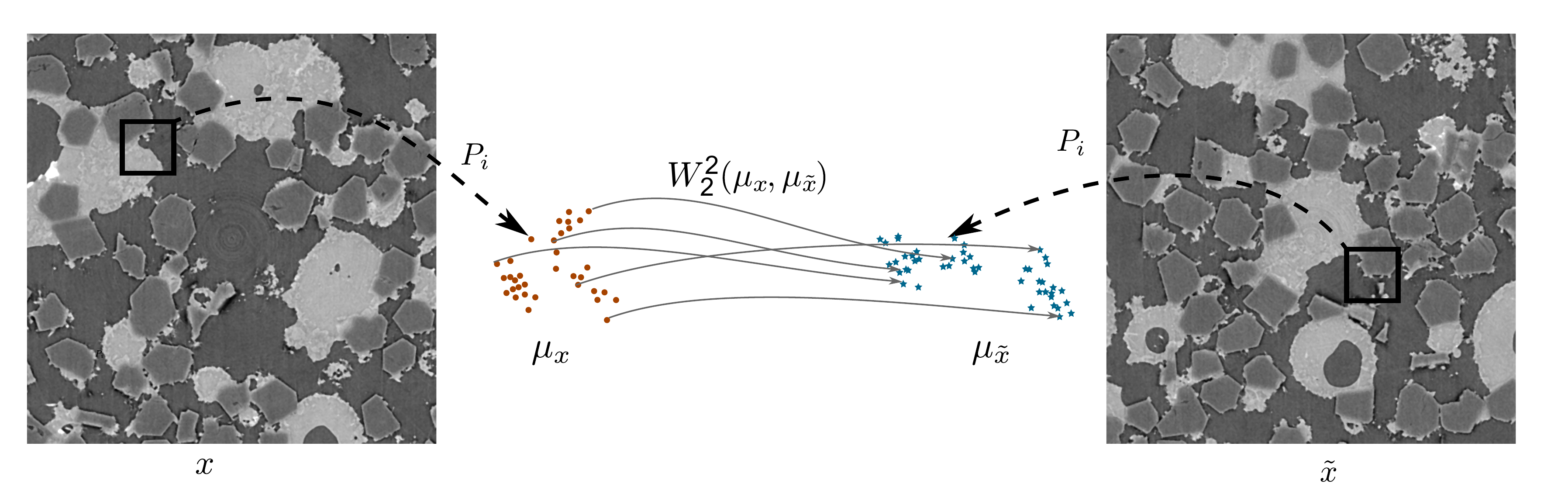}
    \caption{Visualization of the Wasserstein patch prior.}
    \label{fig_vis}
\end{figure}

Let $f\colon\R^{d_1\times d_2}\to\R^{n_1\times n_2}$ be a known forward operator with $d_1=qn_1$ and $d_2=qn_2$ for some integral magnification factor $q$.
Given an observation $y=f(\bar x)+\xi\in\R^{n_1\times n_2}$ for some noise $\xi$ and a reference image $\tilde x$, we aim to reconstruct the unknown ground truth image $\bar x\in\R^{d_1\times d_2}$ assuming that $\bar x$ and $\tilde x$ have a similar patch distribution.

For this purpose, let $P_j\colon \R^{d_1\times d_2}\to\R^{s_1\times s_2}$, $j=1,\ldots,N$ denote the operator which
extracts the $j$-th patch of size $s_1 \times s_2$ with $s_1 \ll d_1$, $s_2 \ll d_2$
from an image  $x \in \R^{d_1 \times d_2}$.
Resampling the images columnwise, we may consider $f\colon\R^d\to\R^n$ and $P_j\colon \R^{d} \to \R^{s}$, where
$d \coloneqq d_1 d_2$, $n\coloneqq n_1 n_2$ and $s \coloneqq s_1 s_2$.
Then we define 
the \emph{empirical patch distribution} $\mu_x$ of  $x \in \R^{d}$ 
by
\begin{align}
\mu_{x} = \frac{1}{N} \sum_{j=1}^{N} \delta_{P_j( x)},
\end{align}
where $\delta$ denotes the Dirac measure.
The main assumption in the following models is that similarly structured images $x$ and $\tilde x$ have also 
similar empirical patch distributions.
Based on the empirical patch distributions of an image $x$ and a reference image $\tilde x$, 
we define the  \emph{Wasserstein Patch Prior} (WPP) as the squared Wasserstein-2 distance of the corresponding empirical patch distributions, i.e.,
$$
W_2^2 (\mu_x,\mu_{\tilde x}) \coloneqq 
\min\limits_{\pi \in \Pi(\mu_x,\mu_{\tilde x})} \sum_{j=1}^{N} \sum_{k=1}^{\tilde N} \|P_j (x) - P_k(\tilde x) \|^2 \pi_{j,k}
$$
where                                                        
$\Pi(\mu_x,\mu_{\tilde x}) \coloneqq \{ \pi = (\pi_{j,k}) \in \R^{N \times \tilde N}_{\ge 0}: 
\sum_{j=1}^N \pi_{j,k} = \frac{1}{\tilde N}, \,
\sum_{k=1}^{\tilde N} \pi_{j,k} = \frac1N \}
$.
The WPP is visualized in Figure~\ref{fig_vis}.
Note that it is not required that the images have the same size. Instead, we only require that all patches  have the same size.
We will frequently use the semi-dual form of the Wasserstein distance, see e.g.,~\cite[Chapter~1]{Santambrogio2015},
\begin{equation} \label{w_dual}
W_2^2 (\mu_x,\mu_{\tilde x}) 
=
\max\limits_{\boldsymbol{\psi} \in \R^{\tilde{N}}} 
\Big( \frac{1}{N} \sum_{j=1}^N \psi^c \left(P_j(x) \right) + \frac{1}{{\tilde N}} \sum_{k=1}^{\tilde N} \psi_k \Big),
\end{equation}
where $\boldsymbol{\psi} \coloneqq (\psi_k)_{k=1}^{\tilde N}$ and
$\psi^c \left( P_j(x) \right) \coloneqq \min\limits_{k \in \{1,...,{\tilde N}\}} \{ \| P_j (x) -  P_k(\tilde x) \|^2 - \psi_k \}$ 
denotes the $c$-transform of $\boldsymbol{\psi}$.
In \cite{Hertrich21}, the WPP is used as regularizer in the variational problem
\begin{align}\label{WPP}
\mathcal J(x) \coloneqq \frac12\|f(x)-y\|^2+\lambda W_2^2(\mu_x,\mu_{\tilde x}).
\end{align}
It was shown that \eqref{WPP} outperforms state-of-the-art methods
for superresolution of material images.
However, it requires the minimization of the functional \eqref{WPP} 
for each high-resolution image $x$ we want to reconstruct
from its low-resolution counterpart,  which is computationally costly.
Therefore, in the next section we propose to learn a NN 
based on the above loss function and then to use this network
to generate high-resolution images in a fast way.

\section{WPPNets} \label{Section_WReg}
We assume that we are given a high-resolution reference image 
$\tilde{x}\in \R^d$
and several low-resolution images $y_i\in\R^n$, $i=1,\ldots,m$, $n < d$.
The corresponding high-resolution ground truth images $\bar x_i\in\R^d$, $i=1,\ldots,m$ are unknown, 
but we assume that the  patch distributions within the reference image $\tilde{x}$ and the ground truths $\bar x_i$ are similar.
Based on \eqref{WPP}
it appears natural to train a CNN $G_\theta\colon\R^n\to\R^d$ using the loss function
\begin{align}\label{eq_loss1}
\mathcal L(\theta) \coloneqq \frac1m\sum_{i=1}^m \|f(G_\theta(y_i))-y_i\|^2 + \lambda W_2^2(\mu_{G_\theta(y_i)},\mu_{\tilde x}).
\end{align}
In other words, the network will be trained to map an observation $y$ 
onto the corresponding WPP reconstruction $x$, which is the minimizer of \eqref{WPP}.
As the CNN $G_\theta$ is fully convolutional, we can apply it for images of arbitrary size.
More precisely, $G_\theta$ maps a low-dimensional input image of size $n_1\times n_2$ to a high-resolution image of size $d_1\times d_2$ with $d_1=qn_1$ and $d_2=qn_2$, where $q$ is the magnification factor.

Since each training image shows only a small part of the considered texture or material, the patch distribution in this part might have a severe bias compared with the patch distribution
in the whole texture or material because of local structures shown in this specific part.
Consequently, the Wasserstein patch prior, which enforces equality of these patch distributions, does not
make sense when it is applied on too small images.
As a remedy, 
we divide our training data into disjoint batches $(B_j)_j$ 
of size $|B_j|=b$, $j=1,...,N_B$ with $\bigcup_{j=1}^{N_B} B_j=\{1,...,m\}$. Then the loss function \eqref{eq_loss1} can be rewritten to
\begin{align}\label{eq_loss2}
\mathcal L(\theta) =\frac1{N_B} \sum_{j=1}^{N_B}\Big(\frac1b\sum_{i\in B_j} \|f(G_\theta(y_i))-y_i\|^2 +\lambda \frac1b\sum_{i\in B_j}  W_2^2(\mu_{G_\theta(y_i)},\mu_{\tilde x})\Big).
\end{align}
Instead of comparing the patch distributions of $\tilde x$ and $G_\theta (y_i)$ by $W_2^2(\mu_{G_\theta(y_i)},\mu_{\tilde x})$ for each $i$ separately, we now compare
the patch distribution of $\tilde x$ with the distribution of all patches within the images of the batch $B_j$. This can be seen as a concatenation of the reconstructed images $G_\theta (y_i)$, $i \in B_j$ for some batch $B_j$,  in order to reduce the bias of the patch distribution of the single reconstructions $G_\theta (y_i)$ compared to the patch distribution of all reconstructions. Formally, this corresponds to replacing in \eqref{eq_loss2} the term
\begin{equation}\label{replace_1}
\frac1{b}\sum_{i\in B_j}W_2^2(\mu_{G_\theta(y_i)},\mu_{\tilde x}) \quad \mathrm{by} \quad
W_2^2\Big(\frac1{b}\sum_{i\in B_j}\mu_{G_\theta(y_i)},\mu_{\tilde x}\Big),
\end{equation}
for $j=1,...,N_B$.
Then, we obtain the loss function
\begin{align}\label{eq_loss}
\mathcal L_\mathrm{WPPNet}(\theta)
\coloneqq
\frac1{N_B} \sum_{j=1}^{N_B}\Big(\frac1b\sum_{i\in B_j} \|f(G_\theta(y_i))-y_i\|^2 +\lambda  W_2^2\Big(\frac1b\sum_{i\in B_j}\mu_{G_\theta(y_i)},\mu_{\tilde x}\Big)\Big).
\end{align}
We call a neural network trained with the loss function \eqref{eq_loss} a \emph{Wasserstein Patch Prior Network} (WPPNet).
\begin{remark}
The required batch size depends on the size of the low-resolution example images $y_i$, $i=1,...,m$ and on the level of homogeneity of the considered texture or material.
In the case of large example images $y_i$ and a very homogeneous texture, we can assume that the patch distribution in each image $y_i$ is representative for the whole texture. Consequently, we can choose batch size $b=1$. In this case, the loss functions \eqref{eq_loss} and \eqref{eq_loss2} coincide.
However, in practice, we have often only access to small low-resolution images and in particular for applications with materials' microstructures the considered images admit a lower homogeneity such that larger batch sizes are necessary.
\end{remark}

In order to minimize the loss function \eqref{eq_loss} with a gradient-based optimization method, 
we need to compute the derivative of 
$W_2^2 (\mu_{G_{\theta}(y)},\mu_{\tilde{x}})$ with respect to $\theta$ for some observation $y\in\mathcal Y$. 

Since the Wasserstein distance is computed iteratively, the application of backpropagation is computationally intractable. 
Instead, we compute the derivative via the semi-dual formulation of the Wasserstein distance.
To this end, recall that the Wasserstein distance reads in its semi-dual form as
\begin{align*}
W_2^2(\mu_{G_{\theta}(y)},\mu_{\tilde{x}}) 
= 
\max_{\psi \in \R^{\tilde{N}}} F (\psi,\theta; y), \hspace{0.3cm}
F(\psi,\theta;y) \coloneqq \frac{1}{N} \sum_{j=1}^{N} \psi^c 
\left( P_j \left( G_{\theta}(y) \right) \right)
+ \frac{1}{\tilde{N}} \sum_{k=1}^{\tilde{N}} \psi_k.
\end{align*}
Then the following well-known theorem provides a connection between the gradient of the Wasserstein distance and those of $F$, 
see e.g., \cite{Hou21Patch}.

\begin{theorem} \label{Theorem_derivativeW2}
Let both
$\theta \mapsto W_2^2 \Big(\frac{1}{N} \sum_{j=1}^N \delta_{P_j \left( G_\theta (y) \right)}, 
\frac{1}{\tilde{N}} \sum_{k=1}^{\tilde{N}} \delta_{P_k (\tilde{x})} \Big)$ 
and $\theta \mapsto F(\psi^{*},\theta;y)$ be differentiable at $\theta_0$ 
with $\psi^{*} \in \argmax_{\psi} F(\psi,\theta_0;y)$. 
Then it holds
\begin{align*}
\nabla_\theta W_2^2 \Big(\frac{1}{N} \sum_{j=1}^N \delta_{P_j \left(G_{\theta_0} (y) \right)}, \frac{1}{\tilde{N}} \sum_{k=1}^{\tilde{N}} 
\delta_{P_k (\tilde{x})} \Big) 
= \nabla_\theta F(\psi^{*},\theta_0;y).
\end{align*}
\end{theorem}

Since for almost every $\theta$ the set of minimizers
\begin{align}\label{eq:kappa}
\kappa_\psi (j) = \argmin_{k \in \{ 1,...,\tilde{N} \}} \big(\| P_j \left(G_\theta (y) \right) - P_k (\tilde{x}) \|^2 - \psi_k \big)
\end{align}
is single-valued, it holds that 
$\psi^c \left(P_j \left(G_{\theta}(y)\right) \right)
= 
\| P_j \left(G_{\theta}(y) \right) - P_{\kappa_\psi (j)} \left(\tilde{x}\right) \|^2 
- \psi_{\kappa_{\psi (j)}}$, 
which is differentiable in $\theta$ if $G_\theta$ is. 
Thus, by Theorem \ref{Theorem_derivativeW2}, the gradient of the Wasserstein distance is given by
\begin{align*}
\nabla_\theta W_2^2 (\mu_{G_\theta (y)},\mu_{\tilde{x}}) 
&= \nabla_\theta F(\psi^* , \theta;y) \\
&= \frac{2}{N} \sum_{j=1}^N 
\Big(\partial_\theta \left(P_j \left(G_\theta (y) \right) \right)^\tT\Big)  
\left(P_j \left(G_\theta (y) \right) - P_{\kappa_{\psi^*} (j)} \left(\tilde{x} \right) \right),\label{eq_wasser_grad}
\end{align*}
where
$
\psi^*\coloneqq\argmax_{\boldsymbol{\psi}\in\R^{\tilde N}} F(\psi,\theta;y).
$
For computing $\psi^*$, we use a (stochastic) gradient ascent algorithm as suggested in
\cite{Hou21Patch}.

\begin{remark}[Computational Complexity]\label{rem_computational_complexity}
The computation of (the gradient of) 
$W_2^2(\mu_{G_{\theta}(y)},\mu_{\tilde x})$ does not scale well with the number of patches $N$ 
and $\tilde N$ of the reconstruction $G_\theta(y)$ and the reference image $\tilde x$.
More precisely, it has complexity $\mathcal O(N\tilde N)$ due to the computation 
of the $\kappa_\psi$ in \eqref{eq:kappa}.
As both $N$ and $\tilde N$ might be large, this leads to an intractable computational effort. 
To speed up the numerical computations, we replace the empirical patch distribution $\mu_{\tilde x}$  
by the distribution of a random subset of all patches, i.e., we redefine $\mu_{\tilde x}$ as
$$
\mu_{\tilde x}=\frac1{|I|}\sum_{k\in I}\delta_{P_k(\tilde x)},
$$
where $I$ is a random subset of $\{1,...,\tilde N\}$ of size $|I|\ll \tilde N$.
Since $\mu_{\tilde x}$ has to be a fixed measure, this set $I$ is chosen once in the beginning and 
kept over the whole training procedure.
Then, the complexity of the computation of (the gradient of) 
$W_2^2(\mu_{G_{\theta}(y)},\mu_{\tilde x})$ reduces to $\mathcal O(N|I|)$.
Note that we cannot subsample the patch distribution of the reconstruction in the same way 
as this would lead to pixels in the reconstructions, which are not effected by the regularizer.
For our numerical experiments we chose $|I|=10000$.
\end{remark}

\begin{remark}[Relation to the deep image prior]
In the case that $m=1$, i.e., when we have given exactly one training image, the loss function of the WPPNet
reads as
$$
\|f(G_\theta(y))-y\|^2+\lambda W_2^2(\mu_{G_\theta(y)},\mu_{\tilde x}).
$$
This is quite similar to the loss function of DIP+TV \eqref{eq_DIP:Loss} with the only difference that the 
TV regularization is replaced by a WPP regularizer.
Consequently, in this case, we can interpret the WPPNet as deep image prior with WPP regularization.
However, it is well known that the reconstructions of the deep image prior are highly adapted to the considered observation and do not generalize to observations which are unseen during the training time.
Therefore, this relation of WPPNets to the deep image prior holds not longer true as soon as the number $m$ of 
low-resolution images is larger than $1$.
Nevertheless, we can hope that the effect of a ``regularization by a CNN architecture'' improves the results of the WPPNet over the variational WPP reconstruction in some cases.
\end{remark}

\section{WPPFlows} \label{Sec_UQ}
For several applications, it is crucial to get not only a realistic reconstruction $x$, 
but also to measure the uncertainty within the reconstruction.
From a mathematical point of view this corresponds to reconstructing the full posterior distribution
$P_{X|Y=y}$ within the Bayesian inverse problem 
\begin{equation}\label{eq_BIP}
Y=f(X)+\Xi,
\end{equation}
where $\Xi\sim\mathcal N(0,\sigma^2 I)$ is independent of $X$,
For this purpose, we will make use of \emph{conditional normalizing flows}.
In particular, this allows  to produce different possible high-resolution reconstructions from the same low-resolution image $y$.

Within the Bayesian inverse problem \eqref{eq_BIP} we need to specify the prior distribution $P_X$.
Here, we assume that it is defined via the Wasserstein patch prior by the density
\begin{equation}\label{eq_prior_px}
p_X(x) \sim \exp(-\rho W_2^2(\mu_x,\mu_{\tilde x})),
\end{equation}
where $\rho>0$ is a hyperparameter.
The following proposition ensures that this defines indeed a probability distribution, 
which is crucial to apply the concepts from Baysian statistics.
\begin{proposition}\label{thm_WPP_dist}
The function
$
\varphi(x)\coloneqq\exp(-\rho W_2^2(\mu_x,\mu_{\tilde x}))
$
is integrable.
\end{proposition}
The proof is given in Appendix~\ref{app_proof_wpp_dist}.

The following remark interprets the variational WPP model \eqref{WPP} from the previous sections as maximum-a-postiori estimator (MAP) within the Bayesian inverse problem \eqref{eq_BIP}.

\begin{remark} \label{map}Under the assumption that $P_{Y|X=x} =\mathcal N(f(x),\sigma^2 I)$, Bayes' theorem implies that maximizing the log-posterior distribution $\log(p_{X|Y=y}(x))$
corresponding to \eqref{eq_BIP} can be written as
\begin{align*}
\argmax_x \{\log (p_{X|Y=y}(x) )\}
&=
\argmax_x \Big\{\log  \Big(\frac{p_{Y|X=x}(y)\, p_X(x)}{p_Y(y)} \Big)\Big\}
\\
&=
\argmax_x \left\{ \log \left( \exp \left( -\|f(x) -y\|^2/(2\sigma^2) \right) \right) + \log \left( p_X(x) \right) \right\}
\\                                                    
&=
\argmin_x \left\{ \tfrac12 \|f(x)-y\|^2 - \sigma^2 \log \left( p_X(x) \right) \right\}.
\end{align*}
For the prior \eqref{eq_prior_px} with
$\rho \coloneqq \frac{\lambda}{\sigma^2}$
this gives the variational WPP model \eqref{WPP}.
\end{remark}

\paragraph{Normalizing Flows}
The aim of normalizing flows is to sample from a complicated probability distribution $P_X$ 
which admits the density function $p_X$.
A \emph{normalizing flow} $\mathcal T=\mathcal T_\theta\colon\R^d\to\R^d$ is an invertible neural network with parameters $\theta$, which is learned to push forward a simple distribution $P_Z$ (usually a standard Gaussian) to $P_X$ such that
$
\mathcal T_\#P_Z = P_Z \circ \mathcal T^{-1} \approx P_X.
$
Here, the symbol $\approx$ means that the distributions are similar in some proper distance or divergence.
Several architectures of normalizing flows were proposed in literature \cite{CBDJ2019,EKS2020,KD2018,RM2015}. Here, we use an adaption of the SRFlow architecture \cite{Lugmayr20} based on affine coupling blocks \cite{AKRK2019,DSB2017}. A detailed description is given in Appendix~\ref{sec_WPPFlow_architecture}.

For applications in inverse problems, normalizing flows were generalized to incorporate a condition \cite{Ardizzone21,HHS2021}. 
More precisely, for approximating all posterior distributions $P_{X|Y=y}$ within the Bayesian inverse problem \eqref{eq_BIP} using a flow model, we learn a mapping $\mathcal T=\mathcal T_\theta\colon\R^d\times\R^n\to\R^d$ such that for all $y\in\R^n$ we have that $\mathcal T(\cdot,y)$ is invertible and that $\mathcal T(\cdot,y)_\#P_Z\approx P_{X|Y=y}$.
Note, that (conditional) normalizing flows can be generalized for the use of non-deterministic transformation, 
see the overview paper \cite{HHS2021tutorial}.

To ensure  $\mathcal T(\cdot,y)_\#P_Z\approx P_{X|Y=y}$, 
we use the expectation on $Y$ of the 
\emph{backward Kullback-Leibler (KL) divergence} 
\begin{equation}\label{eq_backward_KL}
\mathcal L(\theta)\coloneqq\E_{y\sim P_Y}[\KL(\mathcal T(\cdot,y)_\#P_Z,P_{X|Y=y})]
\end{equation}
This was also proposed in \cite{AFHH2021,HHS2021,KDKS2019,Bouman20}. 
The KL divergence is not symmetric and for a discussion 
on forward versus backward Kullback-Leibler divergences we refer to \cite{HHS2021}.

For computing $\mathcal L_\theta$ and taking its derivative, 
we use the following proposition, which is a combination of \cite[Prop. 1]{AFHH2021} and \cite{AKRK2019}. 

\begin{proposition} \label{prop_backwardKL}
Let $X$ and $Y$ be related by the Bayesian inverse problem \eqref{eq_BIP}.
Then, for any $y\in\R^n$, the KL divergence $\KL(\mathcal T(\cdot,y)_\#P_Z,P_{X|Y=y})$ is up to a constant equal to
\begin{align*}
\E_{z\sim P_Z}\Big[\frac1{2\sigma^2}\|f(\mathcal T(z,y))-y\|^2+\log(p_X(\mathcal T(z,y))-\log(|\det(\nabla \mathcal T(z,y))|)\Big]
\end{align*}
\end{proposition}

By discretizing the expectation of $P_Y$ by independent samples $y_1,...,y_m$ of $Y$ and using our definition \eqref{eq_prior_px} of the prior distribution $p_X$, we obtain by Proposition \ref{prop_backwardKL} that the loss function from \eqref{eq_backward_KL} is given by
\begin{align*}
\mathcal L(\theta) &\overset{\sim}{=} \frac1{m} \sum_{i=1}^{m} \E_{z\sim P_Z}\Big[\frac1{2\sigma^2}\|f(\mathcal{T}(z;y_i))-y_i\|^2 +\rho  W_2^2\Big(\mu_{\mathcal{T}(z;y_i)},\mu_{\tilde x}\Big)
\\&\phantom{\coloneqq \frac1{m} \sum_{i=1}^{m} \E_{z\sim P_Z}\Big[} - \log \vert \det \nabla \mathcal{T}(z;y_i) \vert \Big],
\end{align*}
where the symbol $\overset{\sim}{=}$ indicates equality up to a constant.
As in Section~\ref{Section_WReg}, we merge the patch distribution of a batch of images such that we obtain the loss function
\begin{align*}
\mathcal L_\mathrm{WPPFlow}(\theta) &\coloneqq \E_{z_1,...,z_m\sim (P_Z)^m}\Big[\frac1{|B|} \sum_{j=1}^{N_B}\frac{1}{b} \sum_{i\in B_j} \frac{\|f(\mathcal{T}(z_i;y_i))-y_i\|^2}{2\sigma^2}
\\&\phantom{\coloneqq \E_{z_1,...,z_m\sim (P_Z)^m}\Big[}  - \log \vert \det \nabla \mathcal{T}(z_i;y_i) \vert + \rho  W_2^2\Big(\frac1b\sum_{i\in B_j}\mu_{\mathcal{T}(z_i;y_i)},\mu_{\tilde x}\Big)\Big].
\end{align*}
We call a normalizing flow trained by this loss function \emph{Wasserstein Patch Prior Flow} (WPPFlow).

\begin{remark}[MCMC methods]
Using Bayes formula, we can evaluate the density of the log-posterior distribution up to a constant by
\begin{equation}\label{eq_posterior}
\begin{aligned}
\log(p_{X|Y=y}(x))&\overset{\sim}{=}\log(p_{Y|X=x}(y))\log(p_X(x))\\&\overset{\sim}{=}-\tfrac{1}{2\sigma^2}\|f(x)-y\|^2-\rho W_2^2(\mu_x,\mu_{\tilde x}).
\end{aligned}
\end{equation}
A classical possibility for sampling from a probability distribution whose density is known up to 
a constant are Markov chain Monte Carlo (MCMC) methods like the Metropolis Hastings algorithm or the
Langevin dynamics, see \cite{RR2004} for an overview.
However, the methods do not scale very well in high dimensions and for images
usually several millions of evaluations of the target density are required.
This makes the application of MCMC methods impossible when the evaluation of the target density is costly.
In our case the posterior density \eqref{eq_posterior} includes the computation of 
the Wasserstein distance of the empirical patch distributions.
Even though there exist efficient algorithms for computing Wasserstein distances,
such a large number of evaluations leads to an intractable computational effort.
For a detailed comparison of MCMC methods and normalizing flows on a smaller problem, we refer to \cite{AFHH2021}.
\end{remark}

\section{Numerical Results} \label{Section_NumResults}
In this section we demonstrate the good performance of our methods.
In Subsection~\ref{sec_textures} and \ref{sec_synchrotron}, we assume 
that the forward operator $f$, consisting of a blur operator and a downsampling operator,
is given exactly.
Then, in Subsection \ref{sec_inaccurate}, we demonstrate the robustness of WPPNets by assuming that the
operator knowledge is inaccurate.
In the first example, 
we generate the low-resolution data by a slightly different operator than we use for learning the WPPNet.
In the second example, 
we consider real data, where the forward operator $f$ is unknown.
We estimate the forward operator $f$ based on one pair of registered images and we use this 
(inaccurate) estimation of the operator for superresolution.
Finally, in Subsection~\ref{Sec_results_uq} we consider the uncertainty in reconstructing from a low-resolution image using WPPFlows.
Details on the architecture and the experimental setup are given in Appendix~\ref{sec_implementation_details}.

We compare our WPPNet with the following methods:
\begin{itemize}
\item \textbf{bicubic interpolation} \cite{K1981}. 
\item \textbf{Plug-and-Play Forward Backward Splitting with DRUNet} (PnP-DRUNet):
we use the DRUNet from \cite{ZLZZ2021} as denoiser $\mathcal G$ in \eqref{eq_PnP_FBS} and run 100 iterations. 
Even with the learned NN, PnP-DRUNet is significantly slower than a simple evaluation of the WPPNet.
\item \textbf{Deep Image Prior with TV regularization} (DIP+TV)
\cite{UVL2018}\footnote{We use the original implementation from \cite{UVL2018} available at\\ \url{https://github.com/DmitryUlyanov/deep-image-prior}}:
note that each reconstruction with the DIP+TV requires the training of a neural network. Thus, the reconstruction time of
DIP+TV is much slower than for the WPPNet.
\item \textbf{ACNN} trained on natural images:
a natural way to overcome the issue of missing (paired) training data 
could be to train a CNN onto natural images and
to hope that it generalizes to the special structured test set.
Following this approach, we compare our results with an asymmetric CNN (ACNN) \cite{Tian21} trained
on the 400 training images from the BSDS500 dataset \cite{MCTM2001}.
As loss function, we use the standard $L^2$-loss.
Afterwards we apply these networks to our special structured test set.
Training and reconstruction time of a ACNN is comparable with those of a WPPNet. 
Note that the architecture we use here is the same as for the WPPNet.
\item \textbf{WPP} from \eqref{WPP}.
\end{itemize}
We compare the WPPFlow with the following method:
\begin{itemize}
\item \textbf{SRFlow} trained on natural images: we train a modified version of the SRFlow \cite{Lugmayr20}.
The loss function used in \cite{Lugmayr20} is the negative log-likelihood which is equivalent to interchanging the two arguments in the loss function \eqref{eq_backward_KL}.
Note that this leads to a loss function which requires a large database of paired training data.
As we assume that such a database is not available, we use the DIV2K images \cite{AT2017} as training images.
 Note that the architecture we use here is the same as for the WPPFlow.
\end{itemize}

We have to emphasize that, except the WPP, these methods do not include some prior information 
and thus have some weaker assumptions than WPPNets. 
Nevertheless, there are not many comparison methods which include only one high-resolution image. 
The intention is to demonstrate the impact of including a comparably small knowledge about the underlying image domain.

To evaluate the quality of our results, we use different quality measures:
\begin{itemize}
\item \textbf{PSNR.} For two images $x$ and $y$ on $[0,1]^{m\times n}$, the peak-signal to noise ration is defined as
$$
\mathrm{PSNR}(x,y)=-10\log_{10}(\tfrac{1}{mn}\|x-y\|^2).
$$
Larger PSNR values correspond to a better reconstruction.
It is well-known that the PSNR prefers very smooth reconstructions which does in general not coincide with the visual
impression. 
\item \textbf{Blur effect \cite{CDLN2007}.} 
This metric is based on comparing an input image $x$ with a blurred version $x_{\mathrm{blur}}$. 
For sharp images $x$, the difference should be very pronounced while it will be small for blurred $x$.
The blur effect is normalized to $[0,1]$, where a small blur effect indicates that $x$ is very sharp while a large blur effect means that $x$ is very blurry.
\item \textbf{LPIPS \cite{ZIESW2018}\footnote{We use the implementation \url{https://github.com/richzhang/PerceptualSimilarity}, version 0.1.}}.
The basic idea of the learned perceptual image patch similarity is to compare the feature maps extracted from some deep neural network that is trained for some classical imaging task which is not necessarily related to our original problem.
A small value of LPIPS indicates a high perceptual similarity.
\item \textbf{SSIM \cite{WBSS04}.}
The structural similarity index measure compares the overall image structure of two images $x$ and $y$ on $[0,1]^{n \times n}$. It is computed by moving a local window at $M$ locations
\begin{align*}
\text{SSIM}(x,y) \coloneqq \frac{1}{M} \sum_{i=1}^M \frac{(2 \mu_x^{(i)} \mu_y^{(i)} + C_1)(2 \sigma_{xy}^{(i)} + C_2)}{((\mu_x^{(i)})^2 + (\mu_y^{(i)})^2 + C_1)((\sigma_x^{(i)})^2 + (\sigma_y^{(i)})^2 + C_2)},    
\end{align*}
where $\mu_x^{(i)}$ and $\mu_y^{(i)}$ are the mean intensity, $\sigma_x^{(i)}$ and $\sigma_y^{(i)}$ are the standard deviation and $\sigma_{xy}^{(i)}$ is the covariance of $x$ and $y$ at the local window $i$. The local window is chosen to be $7\times7$ and the constants $C_1 = (K_1 L)^2$, $C_2 = (K_2 L)^2$ ensure stability, where $K_1 = 0.01$, $K_2 = 0.03$ and $L = \max(x) - \max(y)$ is the data range as in \cite{WBSS04}.
\item \textbf{FSIM \cite{ZZMZ2011}.} The basic idea of the feature-based similarity index (FSIM) is not to compare the raw pixel values, but to compare the similarity of certain feature maps extracted from the images. We use the same feature maps as in the original paper \cite{ZZMZ2011}.
\end{itemize}
Since we do not want to consider boundary effects, we do not consider a boundary of $40$ pixels when evaluating the quality measures.
All figures show the full images and zoom-in parts below them.

\begin{figure}[t]
\centering
\begin{subfigure}{0.197\textwidth}
  \centering
  \includegraphics[width=\linewidth]{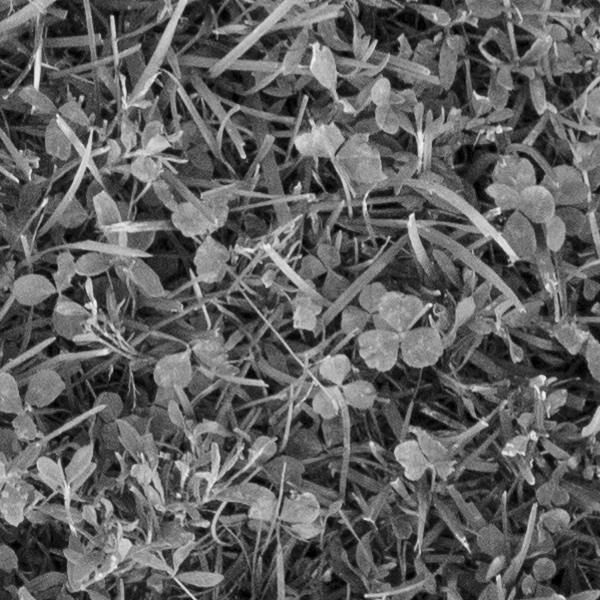}
  \caption{Grass} 
  \label{fig_ref_grass}
\end{subfigure}%
\hfill
\begin{subfigure}{0.197\textwidth}
  \centering
  \includegraphics[width=\linewidth]{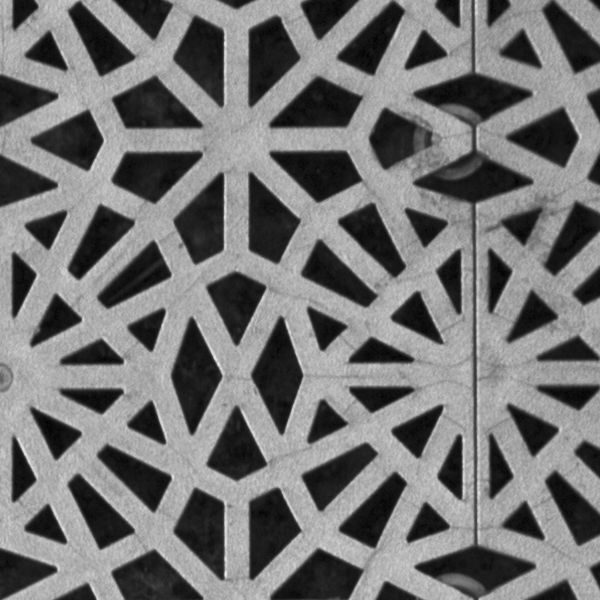}
  \caption{Floor} 
  \label{fig_ref_floor}
\end{subfigure}%
\hfill
\begin{subfigure}{0.197\textwidth}
  \centering
  \includegraphics[width=\linewidth]{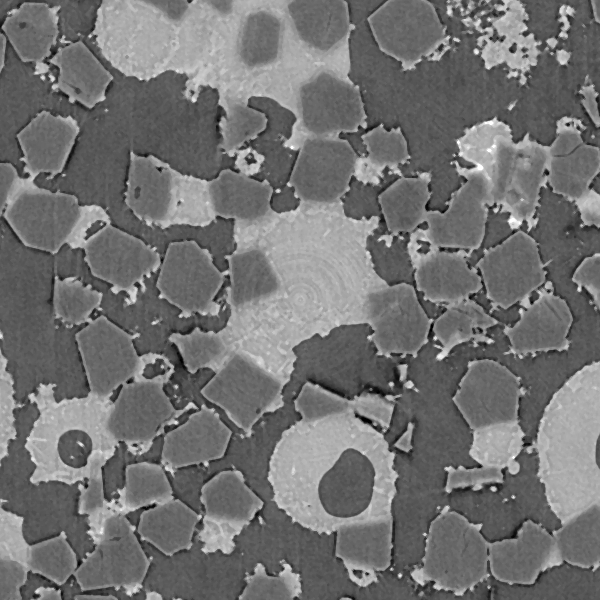}
  \caption{SiC}
  \label{fig_ref_SiC}
\end{subfigure}%
\hfill
\begin{subfigure}{0.197\textwidth}
  \centering
  \includegraphics[width=\linewidth]{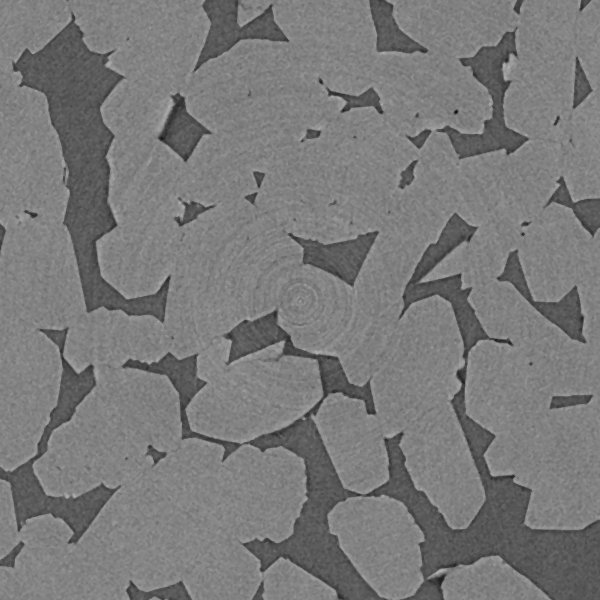}
  \caption{FS} 
  \label{fig_ref_FS}
\end{subfigure}%
\hfill
\begin{subfigure}{0.197\textwidth}
  \centering
  \includegraphics[width=\linewidth]{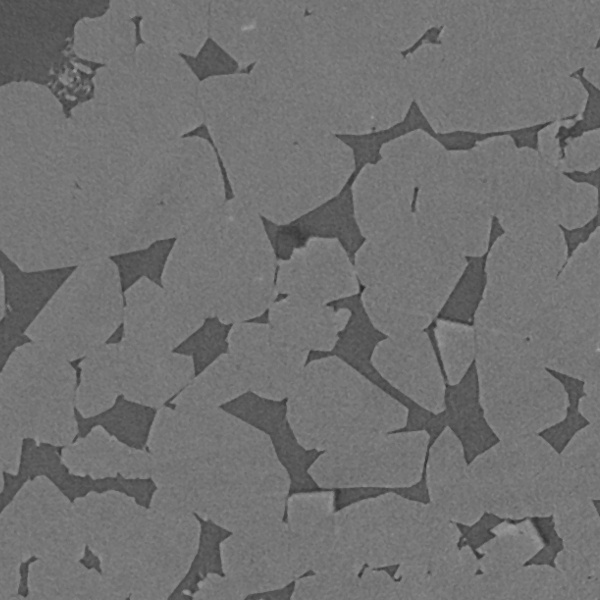}
  \caption{Registered FS} 
  \label{fig_ref_RegFS}
\end{subfigure}%
\caption{Reference images used for the numerical examples.} \label{fig_ground_truth}
\end{figure}

\subsection{Texture Superresolution}\label{sec_textures}

First, we consider the Kylberg texture dataset \cite{K2011}\footnote{available at \url{https://kylberg.org/kylberg-texture-dataset-v-1-0}}.
Here, we use the textures ``Grass'' and ``Floor''. The high-resolution ground truth and the reference image are different $600\times 600$ sections
cropped from the original texture images, see Figures~\ref{fig_ref_grass} and \ref{fig_ref_floor}. Similarly, the low-resolution training data is generated by cropping $100\times 100$
sections from the texture images and artificially downsampling it by a predefined forward operator $f$.
The forward operator $f$ is a convolution with a $16 \times 16$ 
Gaussian blur kernel with standard deviation $2$, 
stride $4$ and $\xi \sim \mathcal{N}(0,0.01^2)$ is some noise. Note that the stride determines the subsampling factor in each direction.
To keep the dimensions consistent, we use zero-padding. As weighting parameter in the WPP-loss \eqref{eq_loss} we used $\lambda = 12.5$.
\footnote{The implementation of this example is available online at \url{https://github.com/FabianAltekrueger/WPPNets}.}

The resulting quality measures are given in Table~\ref{table_errorMeasures_textures} and 
the reconstructions are shown in Figure~\ref{Comp_HRLRPred_texture_grass} and \ref{Comp_HRLRPred_texture_floor}, respectively.
We observe that the WPPNet and WPP lead to significantly sharper and visually better results than the other methods.
However, the WPP requires the minimization of the functional \eqref{WPP} for each reconstruction, which is computationally costly.
Also DIP+TV requires for any reconstruction the training of a NN and for PnP-DRUNet, we have to compute the
iteration \eqref{eq_PnP_FBS} several times.
Thus, the reconstruction time for WPP, DIP+TV and PnP-DRUNet is significantly larger than for WPPNet and ACNN.

Further, we observe that PnP-DRUNet and DIP+TV have a better PSNR value than the WPPNet for the texture ``Floor''.
However, it is well-known that the PSNR as quality measure prefers smooth images.
In terms of the blur effect, LPIPS, FSIM and the visual impression, the WPPNet and WPP are clearly better than the other methods.
Considering the results, we can see that PnP-DRUNet, DIP+TV and ACNN tend to generate oversmoothed images, while WPP and WPPNet tend to oversharpen the reconstruction.
This hypothesis can be underlined, by the fact that the PSNR and SSIM values of the WPPNet reconstruction
can be significantly improved by applying a Gaussian blur filter with standard deviation $0.7$ for grass and $1.0$ for floor on the reconstruction, see Table~\ref{table_errorMeasures_textures} (right).

\begin{figure}[t!]
\centering
\begin{subfigure}[t]{.33\textwidth}
  \includegraphics[width=\linewidth]{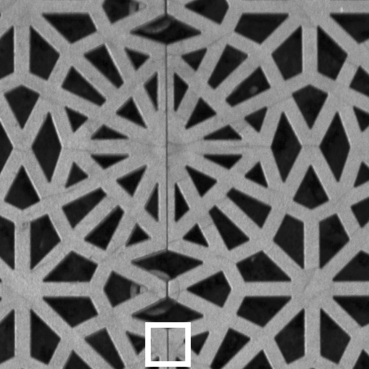}
\end{subfigure}%
\hfill
\begin{subfigure}[t]{.33\textwidth}
  \includegraphics[width=\linewidth]{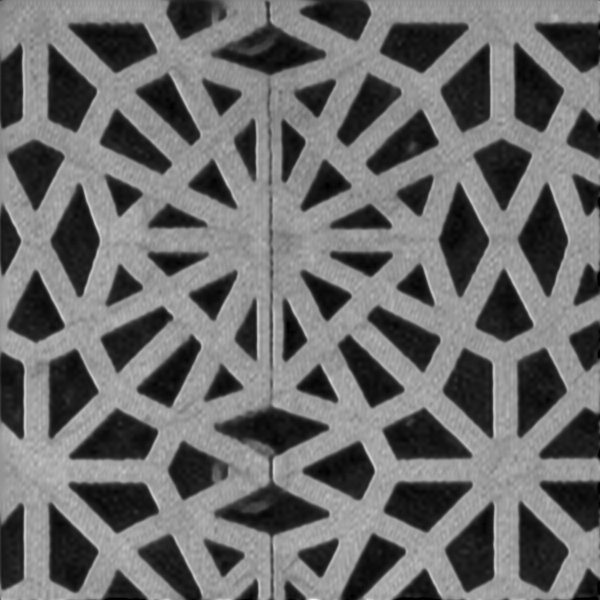}
\end{subfigure}%
\hfill
\begin{subfigure}[t]{.33\textwidth}
  \includegraphics[width=\linewidth]{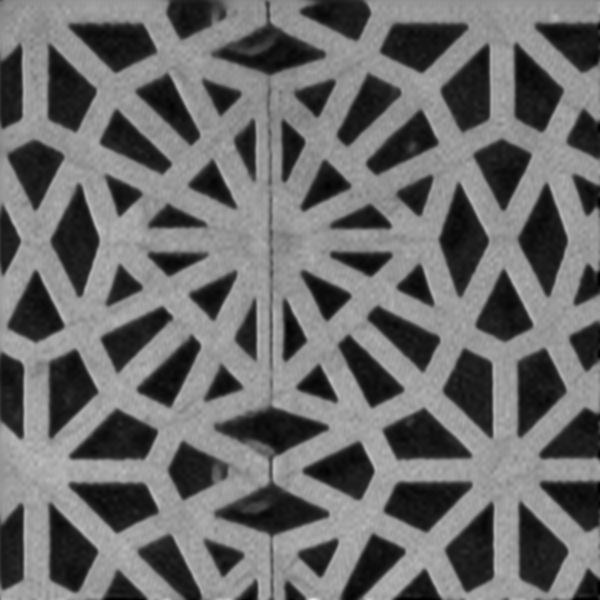}
\end{subfigure}%

\begin{subfigure}[t]{.33\textwidth}
  \includegraphics[width=\linewidth]{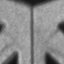}
  \caption*{HR image}
\end{subfigure}%
\hfill
\begin{subfigure}[t]{.33\textwidth}
  \includegraphics[width=\linewidth]{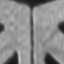}
  \caption*{unblurred}
\end{subfigure}%
\hfill
\begin{subfigure}[t]{.33\textwidth}
  \includegraphics[width=\linewidth]{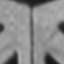}
  \caption*{blurred}
\end{subfigure}%
\caption{Comparison of unblurred and blurred WPPNet reconstruction. The zoomed-in part is marked with a white box in the HR image.} 
\label{Comp_blur_unblur}
\end{figure}

Within the WPPNet reconstruction of the ``Grass'' texture, we observe slightly structured noise. This can be explained by the fact that the reference image is very noisy, see Figure~\ref{fig_ref_grass}.

\begin{table}
\begin{center}
\scalebox{.70}{
\begin{tabular}[t]{c|c|cccccc|c} 
             &             & bicubic & PnP-DRUNet& DIP+TV    & ACNN  & WPP    & WPPNet  & WPPNet blur \\
\hline
             & PSNR        & 22.68   & 24.71  & \textbf{24.97}  & \underline{\textbf{25.06}}  & 24.61 & 24.79 & 24.95 \\ 
Grass        & Blur Effect & 0.5980  & 0.4649 & 0.4410 & 0.4307  & \underline{\textbf{0.4153}} & \textbf{0.4219} & 0.4495\\
             & LPIPS       & 0.4891  & 0.4074 & \textbf{0.2287} & 0.2403  & \underline{\textbf{0.1777}} & 0.2353 & 0.2729 \\
             & SSIM        & 0.6442  & 0.7529 & \textbf{0.7626} & \underline{\textbf{0.7683}} & 0.7473 & 0.7581 & 0.7602\\
             & FSIM        & 0.8820  & 0.9188 & 0.9364 & 0.9328 & \underline{\textbf{0.9420}} & \textbf{0.9378} & 0.9340\\
\hline
Time         &Training     &    -     &  -\footnotemark[6] 
                                                               & -      & 3h       & -        &  8h & \\
             &Reconstruction& 0.0003s &  51.36s     &  114.42s    & 0.05s    &  477.06s     & 0.05s & \\
\hline
\hline
             & PSNR        & 29.03   & \underline{\textbf{32.96}}  & \textbf{32.90}  & 30.35  & 30.39 & 30.99 & 32.86 \\ 
Floor        & Blur Effect & 0.7469  & 0.6803 & 0.6690 & 0.5977  & \underline{\textbf{0.5218}} & \textbf{0.5407} & 0.6187 \\
             & LPIPS       & 0.2568  & 0.2584 & 0.2462 & 0.2795  & \underline{\textbf{0.1647}} & \textbf{0.1705} & 0.1800\\
             & SSIM        & 0.8091  & \underline{\textbf{0.8552}} & \textbf{0.8507} & 0.8345 & 0.7850 & 0.7969 & 0.8413\\
             & FSIM        & 0.9435  & 0.9773 & \textbf{0.9776} & 0.9722 & 0.9752 & \underline{\textbf{0.9796}} & 0.9878\\
\hline
Time         &Training     &    -     &  -\footnotemark[6] 
                                                               & -      &   3h     & -        & 9.5h &  \\
             &Reconstruction& 0.0003s &  51.36s     &  114.42s    & 0.05s    &  477.06s     & 0.05s & \\
\end{tabular}}
\caption{Comparison of superresolution results for the textures ``Grass'' and ``Floor'' (stride 4). 
The best two values are marked in bold, the best one is additionally underlined.}                    
\label{table_errorMeasures_textures}
\end{center}
\end{table}

\footnotetext[6]{For the PnP-DRUNet we used a pretrained denoiser.}

\begin{figure}[t!]
\centering
\begin{subfigure}[t]{.2\textwidth}
  \includegraphics[width=\linewidth]{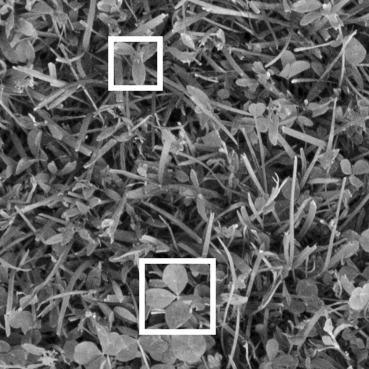}
\end{subfigure}%
\hspace{0.05cm}
\begin{subfigure}[t]{.2\textwidth}
  \includegraphics[width=\linewidth]{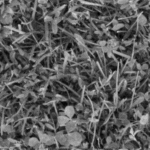}
\end{subfigure}%
\hspace{0.05cm}
\begin{subfigure}[t]{.2\textwidth}
  \includegraphics[width=\linewidth]{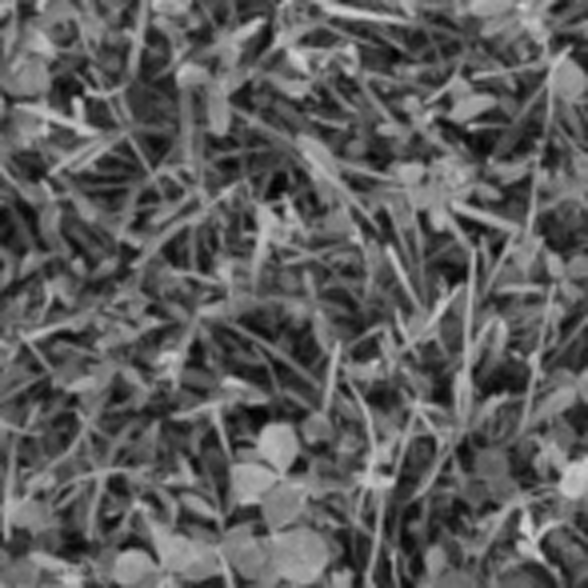}
\end{subfigure}%
\hspace{0.05cm}
\begin{subfigure}[t]{.2\textwidth}
  \includegraphics[width=\linewidth]{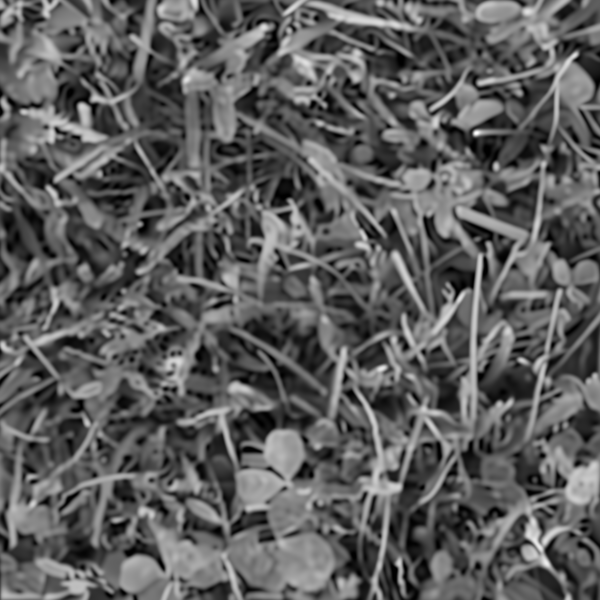}
\end{subfigure}%

\begin{subfigure}[t]{.2\textwidth}
  \includegraphics[width=\linewidth]{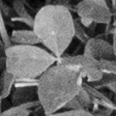}
\end{subfigure}%
\hspace{0.05cm}
\begin{subfigure}[t]{.2\textwidth}
  \includegraphics[width=\linewidth]{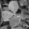}
\end{subfigure}%
\hspace{0.05cm}
\begin{subfigure}[t]{.2\textwidth}
  \includegraphics[width=\linewidth]{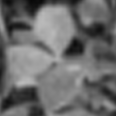}
\end{subfigure}%
\hspace{0.05cm}
\begin{subfigure}[t]{.2\textwidth}
  \includegraphics[width=\linewidth]{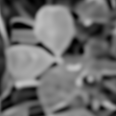}
\end{subfigure}%

\begin{subfigure}[t]{.2\textwidth}
  \includegraphics[width=\linewidth]{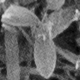}
  \caption*{HR image}
\end{subfigure}%
\hspace{0.05cm}
\begin{subfigure}[t]{.2\textwidth}
  \includegraphics[width=\linewidth]{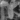}
  \caption*{LR image}
\end{subfigure}%
\hspace{0.05cm}
\begin{subfigure}[t]{.2\textwidth}
  \includegraphics[width=\linewidth]{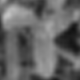}
  \caption*{bicubic}
\end{subfigure}%
\hspace{0.05cm}
\begin{subfigure}[t]{.2\textwidth}
  \includegraphics[width=\linewidth]{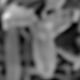}
  \caption*{PnP-DRUNet}
\end{subfigure}%


\begin{subfigure}[t]{.2\textwidth}
  \includegraphics[width=\linewidth]{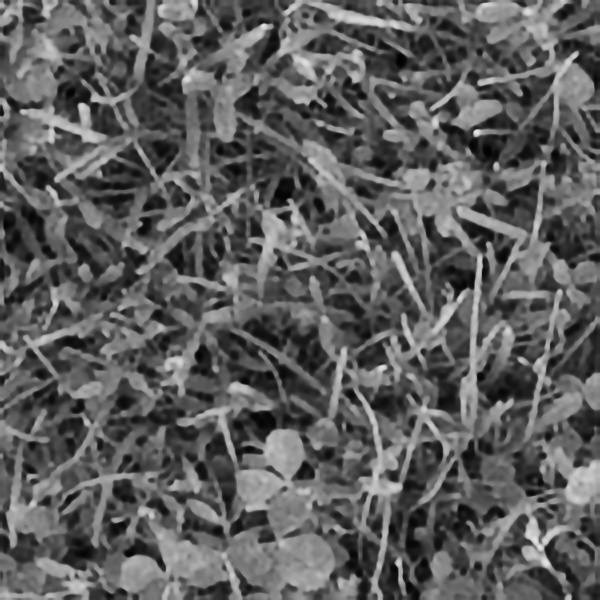}
\end{subfigure}%
\hspace{0.05cm}
\begin{subfigure}[t]{.2\textwidth}
  \includegraphics[width=\linewidth]{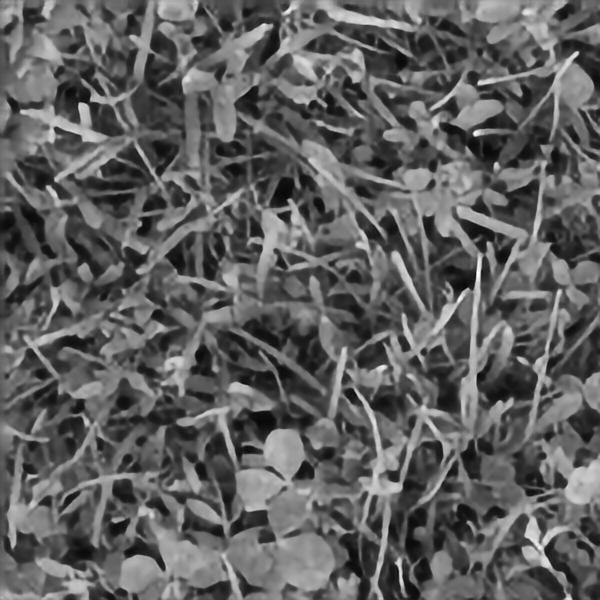}
\end{subfigure}%
\hspace{0.05cm}
\begin{subfigure}[t]{.2\textwidth}
  \includegraphics[width=\linewidth]{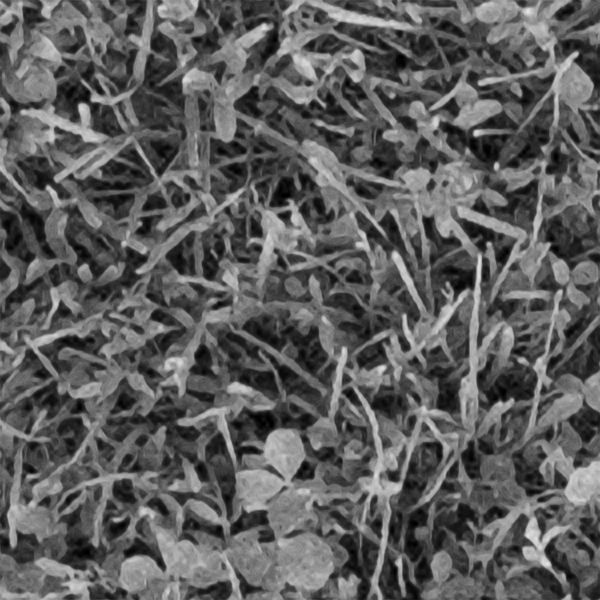}
\end{subfigure}%
\hspace{0.05cm}
\begin{subfigure}[t]{.2\textwidth}
  \includegraphics[width=\linewidth]{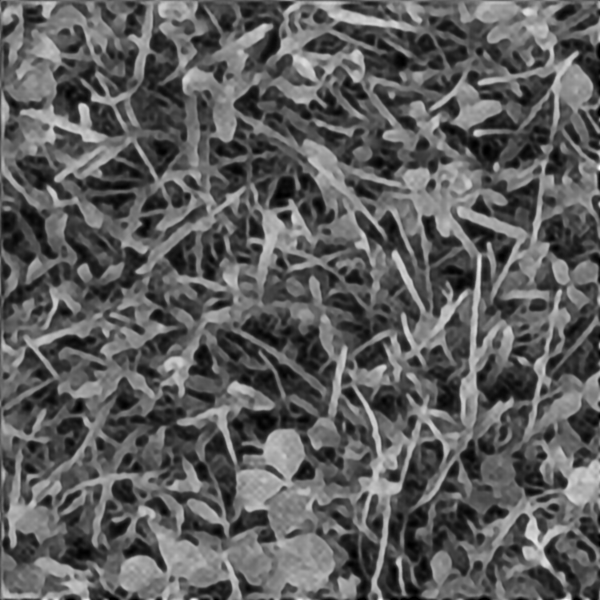}
\end{subfigure}%

\begin{subfigure}[t]{.2\textwidth}
  \includegraphics[width=\linewidth]{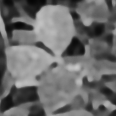}  
\end{subfigure}%
\hspace{0.05cm}
\begin{subfigure}[t]{.2\textwidth}
  \includegraphics[width=\linewidth]{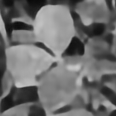}
\end{subfigure}%
\hspace{0.05cm}
\begin{subfigure}[t]{.2\textwidth}
  \includegraphics[width=\linewidth]{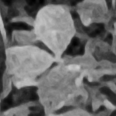}
\end{subfigure}%
\hspace{0.05cm}
\begin{subfigure}[t]{.2\textwidth}
  \includegraphics[width=\linewidth]{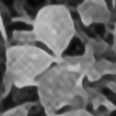}
\end{subfigure}%

\begin{subfigure}[t]{.2\textwidth}
  \includegraphics[width=\linewidth]{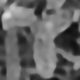}
  \caption*{DIP+TV}
\end{subfigure}%
\hspace{0.05cm}
\begin{subfigure}[t]{.2\textwidth}
  \includegraphics[width=\linewidth]{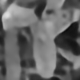}
  \caption*{ACNN}
\end{subfigure}%
\hspace{0.05cm}
\begin{subfigure}[t]{.2\textwidth}
  \includegraphics[width=\linewidth]{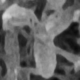}
  \caption*{WPP}
\end{subfigure}%
\hspace{0.05cm}
\begin{subfigure}[t]{.2\textwidth}
  \includegraphics[width=\linewidth]{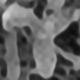}
  \caption*{WPPNet}
\end{subfigure}%
\caption{Comparison of superresolution  for the texture ``Grass'' with stride 4. The zoomed-in parts are marked with a white box in the HR image.} 
\label{Comp_HRLRPred_texture_grass}
\end{figure}
\begin{figure}[t!]
\centering
\begin{subfigure}[t]{.2\textwidth}
  \includegraphics[width=\linewidth]{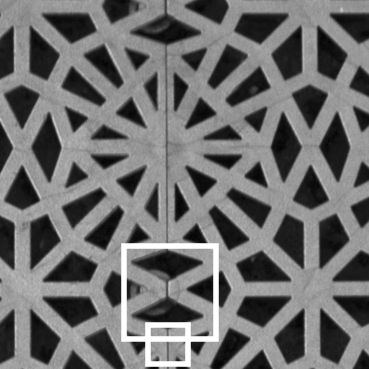}
\end{subfigure}%
\hspace{0.05cm}
\begin{subfigure}[t]{.2\textwidth}
  \includegraphics[width=\linewidth]{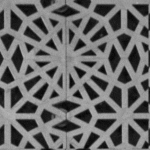}
\end{subfigure}%
\hspace{0.05cm}
\begin{subfigure}[t]{.2\textwidth}
  \includegraphics[width=\linewidth]{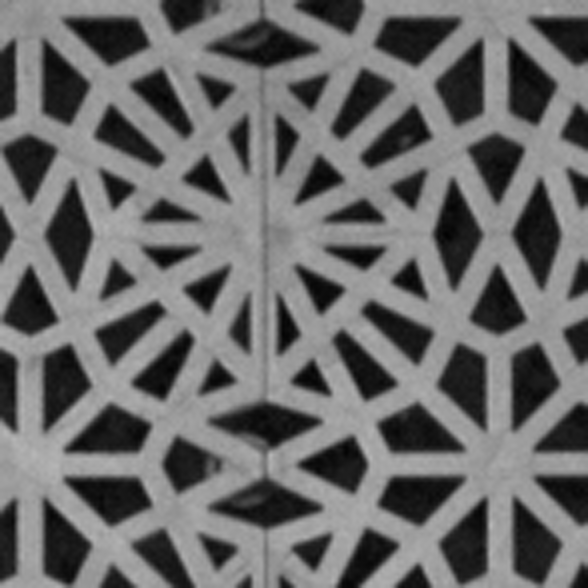}
\end{subfigure}%
\hspace{0.05cm}
\begin{subfigure}[t]{.2\textwidth}
  \includegraphics[width=\linewidth]{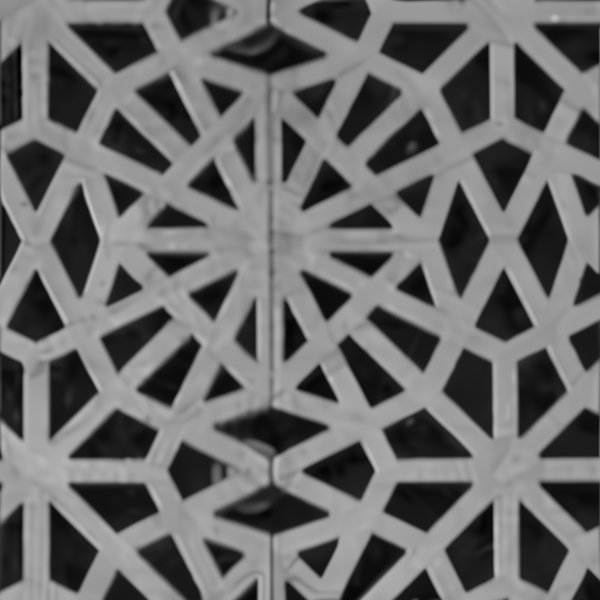}
\end{subfigure}%

\begin{subfigure}[t]{.2\textwidth}
  \includegraphics[width=\linewidth]{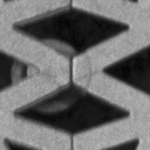}
\end{subfigure}%
\hspace{0.05cm}
\begin{subfigure}[t]{.2\textwidth}
  \includegraphics[width=\linewidth]{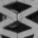}
\end{subfigure}%
\hspace{0.05cm}
\begin{subfigure}[t]{.2\textwidth}
  \includegraphics[width=\linewidth]{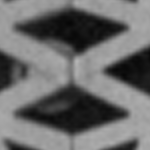}
\end{subfigure}%
\hspace{0.05cm}
\begin{subfigure}[t]{.2\textwidth}
  \includegraphics[width=\linewidth]{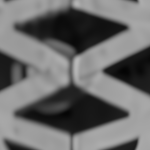}
\end{subfigure}%

\begin{subfigure}[t]{.2\textwidth}
  \includegraphics[width=\linewidth]{Results/Texture_Floor/hr_zoom2.png}
  \caption*{HR image}
\end{subfigure}%
\hspace{0.05cm}
\begin{subfigure}[t]{.2\textwidth}
  \includegraphics[width=\linewidth]{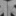}
  \caption*{LR image}
\end{subfigure}%
\hspace{0.05cm}
\begin{subfigure}[t]{.2\textwidth}
  \includegraphics[width=\linewidth]{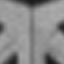}
  \caption*{bicubic}
\end{subfigure}%
\hspace{0.05cm}
\begin{subfigure}[t]{.2\textwidth}
  \includegraphics[width=\linewidth]{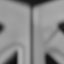}
  \caption*{PnP-DRUNet}
\end{subfigure}%


\begin{subfigure}[t]{.2\textwidth}
  \includegraphics[width=\linewidth]{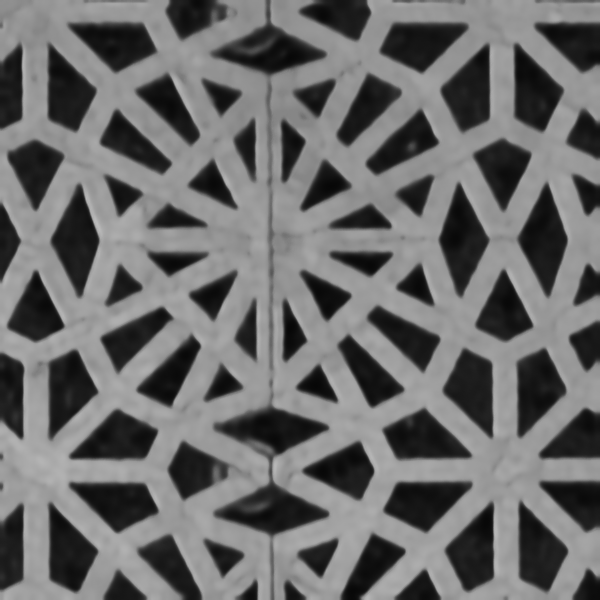}
\end{subfigure}%
\hspace{0.05cm}
\begin{subfigure}[t]{.2\textwidth}
  \includegraphics[width=\linewidth]{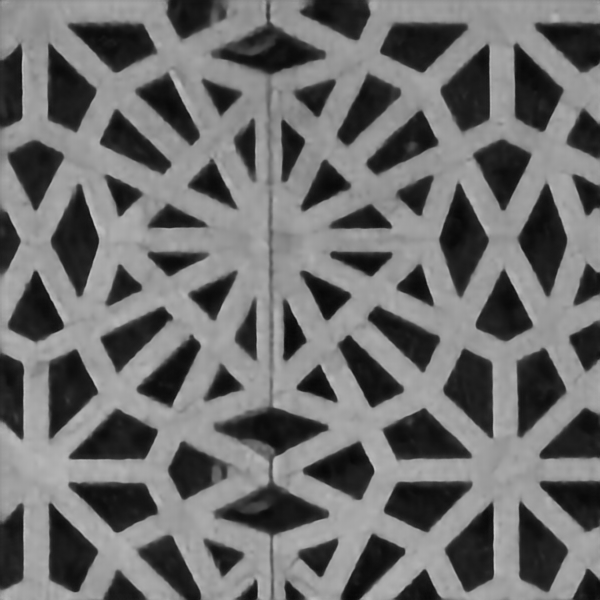}
\end{subfigure}%
\hspace{0.05cm}
\begin{subfigure}[t]{.2\textwidth}
  \includegraphics[width=\linewidth]{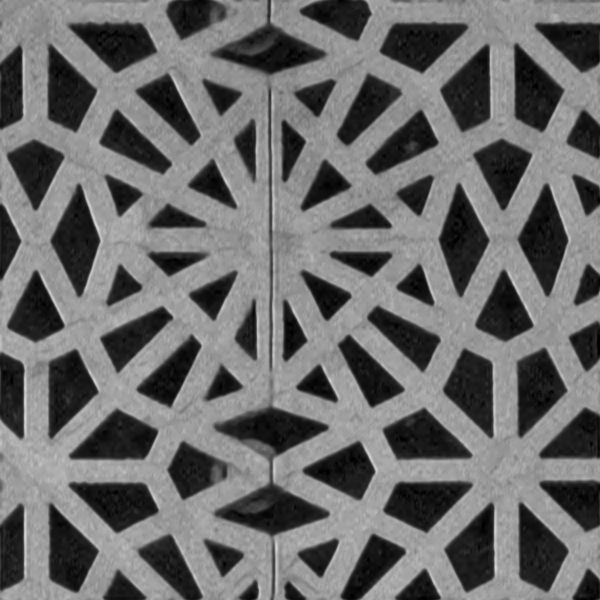}
\end{subfigure}%
\hspace{0.05cm}
\begin{subfigure}[t]{.2\textwidth}
  \includegraphics[width=\linewidth]{Results/Texture_Floor/W2_270.png}
\end{subfigure}%

\begin{subfigure}[t]{.2\textwidth}
  \includegraphics[width=\linewidth]{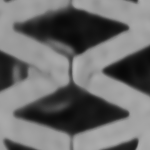}  
\end{subfigure}%
\hspace{0.05cm}
\begin{subfigure}[t]{.2\textwidth}
  \includegraphics[width=\linewidth]{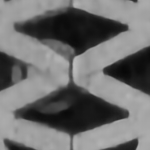}
\end{subfigure}%
\hspace{0.05cm}
\begin{subfigure}[t]{.2\textwidth}
  \includegraphics[width=\linewidth]{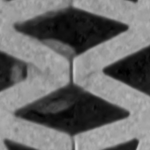}
\end{subfigure}%
\hspace{0.05cm}
\begin{subfigure}[t]{.2\textwidth}
  \includegraphics[width=\linewidth]{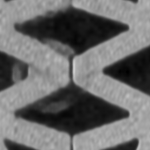}
\end{subfigure}%

\begin{subfigure}[t]{.2\textwidth}
  \includegraphics[width=\linewidth]{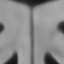}
  \caption*{DIP+TV}
\end{subfigure}%
\hspace{0.05cm}
\begin{subfigure}[t]{.2\textwidth}
  \includegraphics[width=\linewidth]{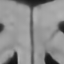}
  \caption*{ACNN}
\end{subfigure}%
\hspace{0.05cm}
\begin{subfigure}[t]{.2\textwidth}
  \includegraphics[width=\linewidth]{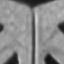}
  \caption*{WPP}
\end{subfigure}%
\hspace{0.05cm}
\begin{subfigure}[t]{.2\textwidth}
  \includegraphics[width=\linewidth]{Results/Texture_Floor/w2_zoom2.png}
  \caption*{WPPNet}
\end{subfigure}%
\caption{Comparison of superresolution for the texture ``Floor'' with stride 4. The zoomed-in parts are marked with a white box in the HR image.}                    
\label{Comp_HRLRPred_texture_floor}
\end{figure}

\subsection{Synchrotron Computed Tomography Data}\label{sec_synchrotron}

Next we consider material data which was also used in \cite{Hertrich21,HNABBSS2020}.
A series of multi-scale 3D images has been acquired by synchrotron micro-computed tomography at the SLS beamline TOMCAT. Samples of two materials were selected to provide 3D images having different levels of complexity, namely
\begin{itemize}
\item[-] ``SiC Diamonds'' obtained by microwave sintering of silicon and diamonds, see \cite{vaucher2007line}.
\item[-] ``FS'' (Fontainebleau sandstone), a rather homogeneous natural rock that is commonly used in the oil industry for flow experiments.
\end{itemize}
In our experiments we consider a voxel spacing of $1.625$ {\textmu}m.
From this 3D image we extract 2D slices of size $600\times 600$ and use them as ground truth and reference images for our experiments, see Figures~\ref{fig_ref_SiC} and \ref{fig_ref_FS}.
Since we require that the forward operator $f$ is known (same $f$ as in Section~\ref{sec_textures}), 
we generate the low-resolution images artificially 
by extracting 2D slices from our 3D image and downsample it using the known predefined forward operator $f$.
In this way, we generate a set of $1000$ low-resolution images of size $25 \times 25$ for training the WPPNet.

The resulting quality measures are given in Table~\ref{table_errorMeasures} and 
the reconstructions are shown in Figure~\ref{Comp_HRLRPred_SiC} and \ref{Comp_HRLRPred_FS}.
Similar as in Subsection \ref{sec_textures}, we observe that the reconstructions with WPPNet and WPP are significantly sharper and visually better than the other methods.
Again, the PSNR prefers in some cases the much smoother reconstructions of DIP+TV and ACNN.
However, the results of WPPNet and WPP look visually much better which is also quantified by smaller values for
LPIPS, FSIM and blur effect.

We apply all methods onto a larger test set in Appendix~\ref{sec_further_exp}.

\begin{table}
\begin{center}
\scalebox{.85}{
\begin{tabular}[t]{c|c|cccccc} 
             &             & bicubic & PnP-DRUNet& DIP+TV   & ACNN  & WPP    & WPPNet \\
\hline
             & PSNR        &  25.34  & 27.43  & \textbf{27.81}  & 27.66   & 27.57  &  \underline{\textbf{27.83}} \\ 
SiC & Blur Effect &  0.5794 & 0.4405 & 0.4046 & 0.4076  & \underline{\textbf{0.3743}} &  \textbf{0.3810} \\
             & LPIPS       &  0.4216 & 0.3133 & 0.2076 & 0.2441  & \underline{\textbf{0.1627}} &  \textbf{0.1819} \\
             & SSIM        & 0.7247  & \textbf{0.7770} & 0.7756 & \underline{\textbf{0.7842}} & 0.7555 & 0.7678 \\
             & FSIM        & 0.8792  & 0.8911 & 0.9344 & 0.9130 & \textbf{0.9443} & \underline{\textbf{0.9449}} \\
\hline
Time         &Training     &    -     &  -    & -      & 6.5h       & -        &  16h  \\
             &Reconstruction& 0.0003s       &  51.36s     &  114.42s    & 0.05s    &  477.06s     & 0.05s  \\
\hline
\hline
             & PSNR        &  29.19  & 31.05  & \textbf{31.35}  & \underline{\textbf{31.68}}   & 31.07  &  30.96 \\ 
FS           & Blur Effect &  0.4856 & 0.4936 & 0.3724 & 0.4276  & \underline{\textbf{0.3225}} &  \textbf{0.3305} \\
             & LPIPS       &  0.3524 & 0.3565 & 0.2162 & 0.2881  & \underline{\textbf{0.1630}} &  \textbf{0.1823} \\
             & SSIM        & 0.7049  & 0.7436 & \textbf{0.7495} & \underline{\textbf{0.7615}} & 0.7275 & 0.7297 \\
             & FSIM        & 0.8561  & 0.7888 & 0.9171 & 0.8429 & \underline{\textbf{0.9505}} & \textbf{0.9439} \\
\hline
Time         &Training     &    -     &  -\footnotemark[6] 
                                                               & -      &   6.5h     & -        &  15h  \\
             &Reconstruction& 0.0003s &  51.36s     &  114.42s    & 0.05s    &  477.06s     & 0.05s  \\
\end{tabular}}
\caption{Comparison of superresolution results for material images with stride 4. 
The best two values are marked in bold, the best one is additionally underlined.}                    
\label{table_errorMeasures}
\end{center}
\end{table}

\begin{figure}
\centering
\begin{subfigure}[t]{.2\textwidth}
  \includegraphics[width=\linewidth]{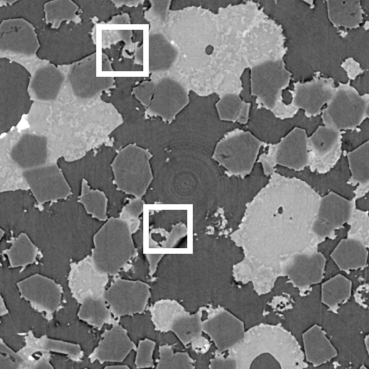}
\end{subfigure}%
\hspace{0.05cm}
\begin{subfigure}[t]{.2\textwidth}
  \includegraphics[width=\linewidth]{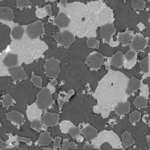}
\end{subfigure}%
\hspace{0.05cm}
\begin{subfigure}[t]{.2\textwidth}
  \includegraphics[width=\linewidth]{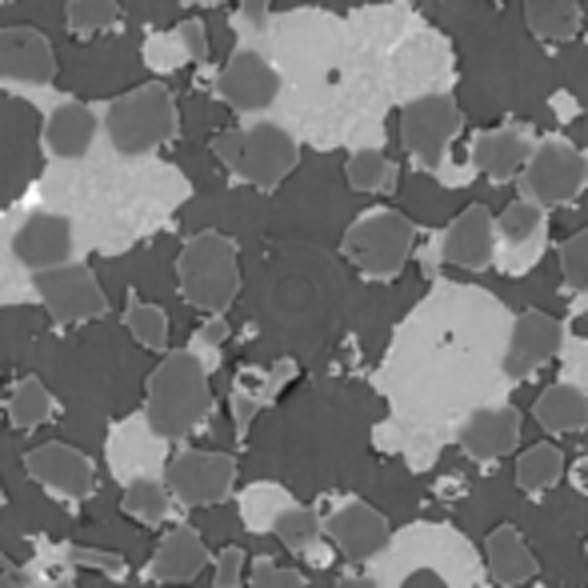}
\end{subfigure}%
\hspace{0.05cm}
\begin{subfigure}[t]{.2\textwidth}
  \includegraphics[width=\linewidth]{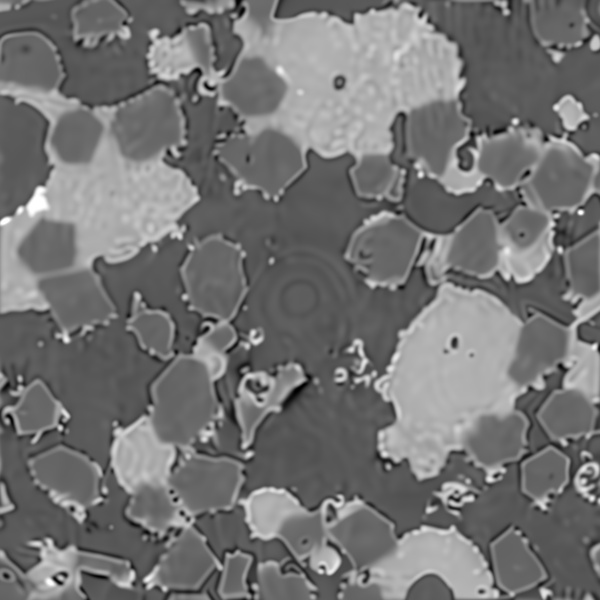}
\end{subfigure}%
\hspace{0.05cm}

\begin{subfigure}[t]{.2\textwidth}
  \includegraphics[width=\linewidth]{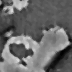}
\end{subfigure}%
\hspace{0.05cm}
\begin{subfigure}[t]{.2\textwidth}
  \includegraphics[width=\linewidth]{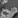}
\end{subfigure}%
\hspace{0.05cm}
\begin{subfigure}[t]{.2\textwidth}
  \includegraphics[width=\linewidth]{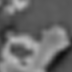}
\end{subfigure}%
\hspace{0.05cm}
\begin{subfigure}[t]{.2\textwidth}
  \includegraphics[width=\linewidth]{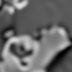}
\end{subfigure}%
\hspace{0.05cm}

\begin{subfigure}[t]{.2\textwidth}
  \includegraphics[width=\linewidth]{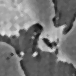}
  \caption*{HR image}
\end{subfigure}%
\hspace{0.05cm}
\begin{subfigure}[t]{.2\textwidth}
  \includegraphics[width=\linewidth]{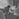}
  \caption*{LR image}
\end{subfigure}%
\hspace{0.05cm}
\begin{subfigure}[t]{.2\textwidth}
  \includegraphics[width=\linewidth]{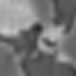}
  \caption*{bicubic}
\end{subfigure}%
\hspace{0.05cm}
\begin{subfigure}[t]{.2\textwidth}
  \includegraphics[width=\linewidth]{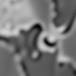}
  \caption*{PnP-DRUNet}
\end{subfigure}%


\begin{subfigure}[t]{.2\textwidth}
  \includegraphics[width=\linewidth]{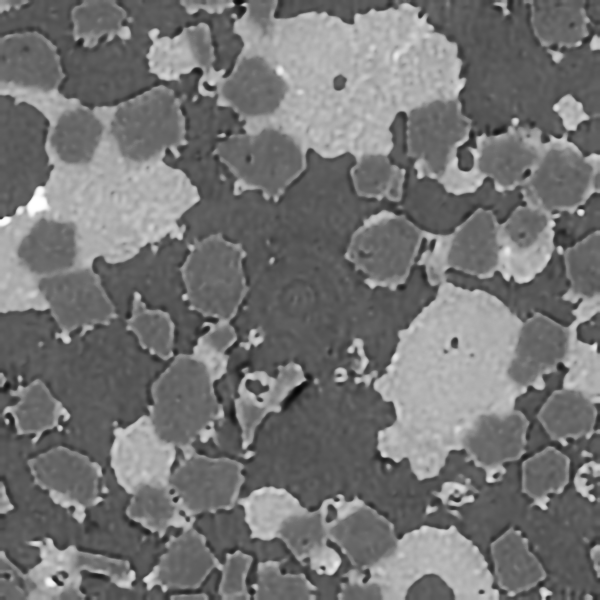}
\end{subfigure}%
\hspace{0.05cm}
\begin{subfigure}[t]{.2\textwidth}
  \includegraphics[width=\linewidth]{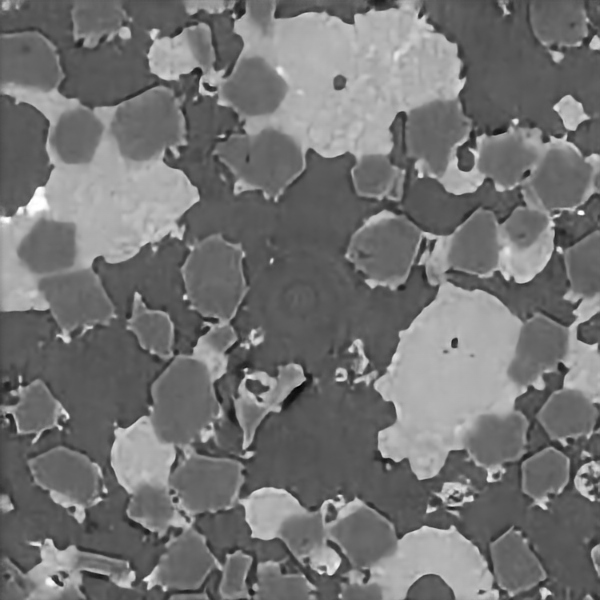}
\end{subfigure}%
\hspace{0.05cm}
\begin{subfigure}[t]{.2\textwidth}
  \includegraphics[width=\linewidth]{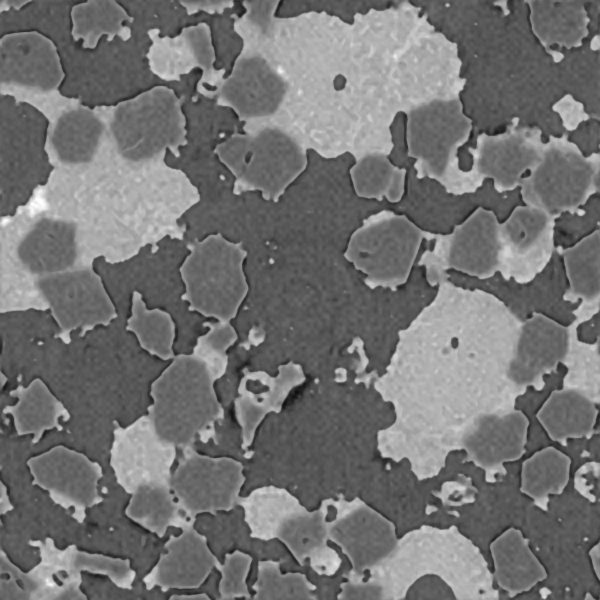}
\end{subfigure}%
\hspace{0.05cm}
\begin{subfigure}[t]{.2\textwidth}
  \includegraphics[width=\linewidth]{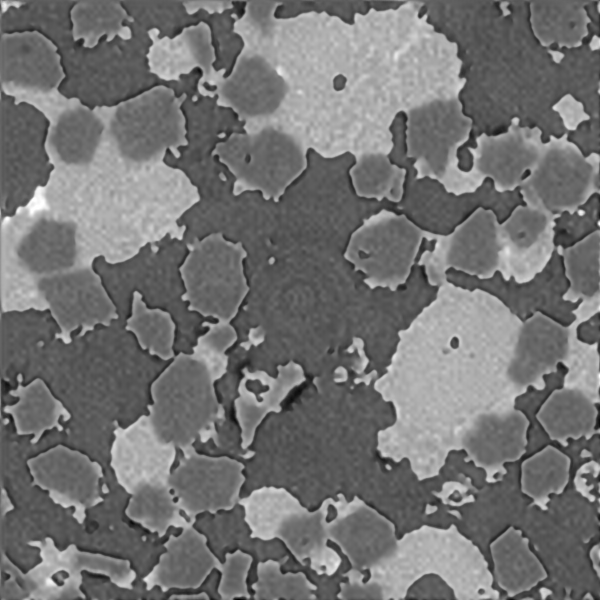}
\end{subfigure}%
\hspace{0.05cm}

\begin{subfigure}[t]{.2\textwidth}
  \includegraphics[width=\linewidth]{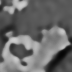}  
\end{subfigure}%
\hspace{0.05cm}
\begin{subfigure}[t]{.2\textwidth}
  \includegraphics[width=\linewidth]{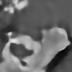}
\end{subfigure}%
\hspace{0.05cm}
\begin{subfigure}[t]{.2\textwidth}
  \includegraphics[width=\linewidth]{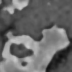}
\end{subfigure}%
\hspace{0.05cm}
\begin{subfigure}[t]{.2\textwidth}
  \includegraphics[width=\linewidth]{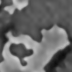}
\end{subfigure}%

\begin{subfigure}[t]{.2\textwidth}
  \includegraphics[width=\linewidth]{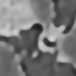}
  \caption*{DIP+TV}
\end{subfigure}%
\hspace{0.05cm}
\begin{subfigure}[t]{.2\textwidth}
  \includegraphics[width=\linewidth]{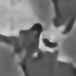}
  \caption*{ACNN}
\end{subfigure}%
\hspace{0.05cm}
\begin{subfigure}[t]{.2\textwidth}
  \includegraphics[width=\linewidth]{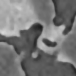}
  \caption*{WPP}
\end{subfigure}%
\hspace{0.05cm}
\begin{subfigure}[t]{.2\textwidth}
  \includegraphics[width=\linewidth]{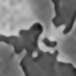}
  \caption*{WPPNet}
\end{subfigure}%
\caption{Comparison of superresolution  for the material ``SiC Diamonds'' with stride~4. The zoomed-in parts are marked with a white box in the HR image.} \label{Comp_HRLRPred_SiC}
\end{figure}

\begin{figure}[t!]
\centering
\begin{subfigure}[t]{.2\textwidth}
  \includegraphics[width=\linewidth]{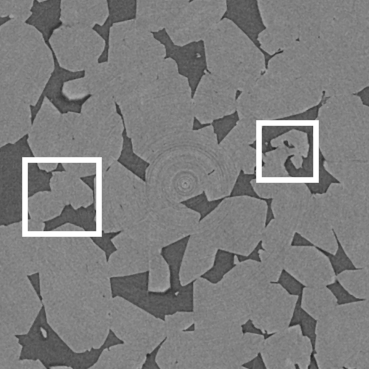}
\end{subfigure}%
\hspace{0.05cm}
\begin{subfigure}[t]{.2\textwidth}
  \includegraphics[width=\linewidth]{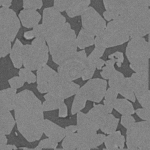}
\end{subfigure}%
\hspace{0.05cm}
\begin{subfigure}[t]{.2\textwidth}
  \includegraphics[width=\linewidth]{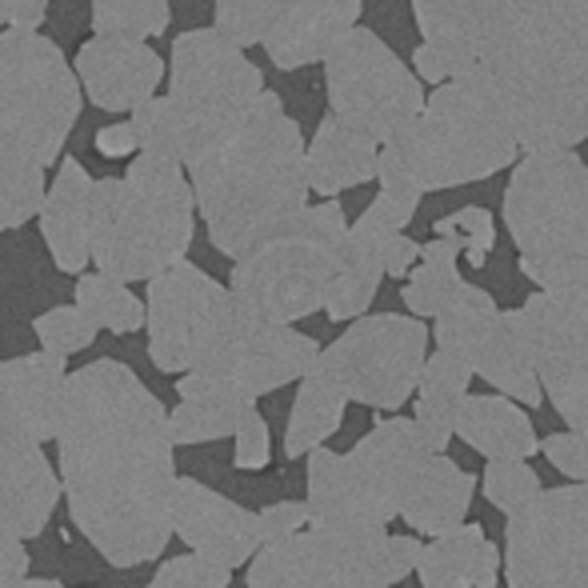}
\end{subfigure}%
\hspace{0.05cm}
\begin{subfigure}[t]{.2\textwidth}
  \includegraphics[width=\linewidth]{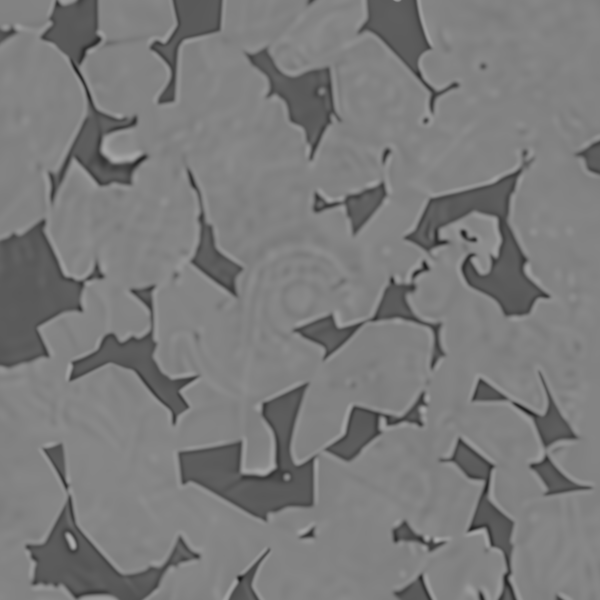}
\end{subfigure}%
\hspace{0.05cm}

\begin{subfigure}[t]{.2\textwidth}
  \includegraphics[width=\linewidth]{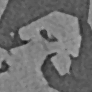}
\end{subfigure}%
\hspace{0.05cm}
\begin{subfigure}[t]{.2\textwidth}
  \includegraphics[width=\linewidth]{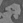}
\end{subfigure}%
\hspace{0.05cm}
\begin{subfigure}[t]{.2\textwidth}
  \includegraphics[width=\linewidth]{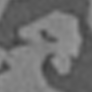}
\end{subfigure}%
\hspace{0.05cm}
\begin{subfigure}[t]{.2\textwidth}
  \includegraphics[width=\linewidth]{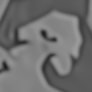}
\end{subfigure}%
\hspace{0.05cm}

\begin{subfigure}[t]{.2\textwidth}
  \includegraphics[width=\linewidth]{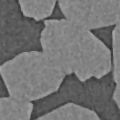}
  \caption*{HR image}
\end{subfigure}%
\hspace{0.05cm}
\begin{subfigure}[t]{.2\textwidth}
  \includegraphics[width=\linewidth]{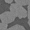}
  \caption*{LR image}
\end{subfigure}%
\hspace{0.05cm}
\begin{subfigure}[t]{.2\textwidth}
  \includegraphics[width=\linewidth]{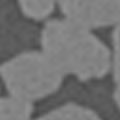}
  \caption*{bicubic}
\end{subfigure}%
\hspace{0.05cm}
\begin{subfigure}[t]{.2\textwidth}
  \includegraphics[width=\linewidth]{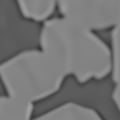}
  \caption*{PnP-DRUNet}
\end{subfigure}%
\hspace{0.05cm}

\begin{subfigure}[t]{.2\textwidth}
  \includegraphics[width=\linewidth]{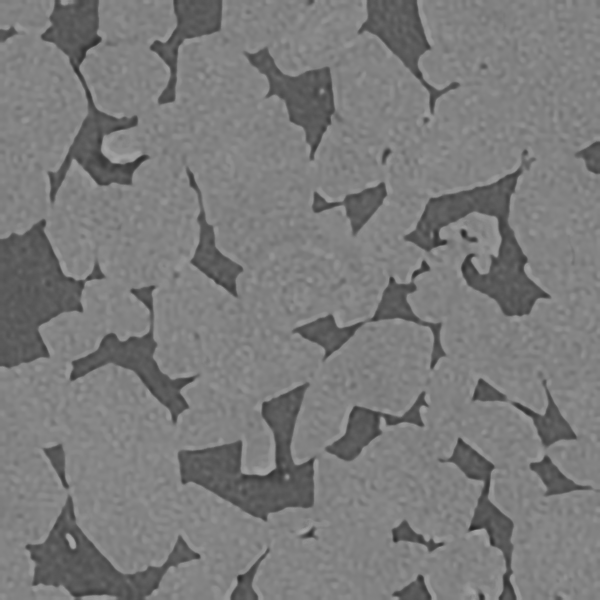}
\end{subfigure}%
\hspace{0.05cm}
\begin{subfigure}[t]{.2\textwidth}
  \includegraphics[width=\linewidth]{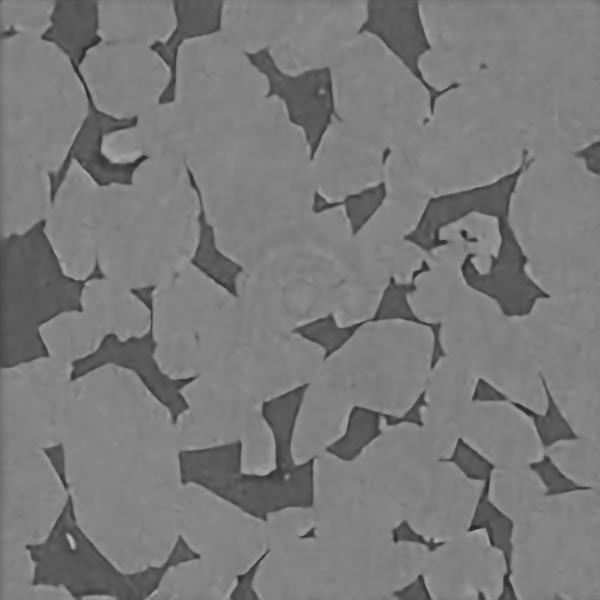}
\end{subfigure}%
\hspace{0.05cm}
\begin{subfigure}[t]{.2\textwidth}
  \includegraphics[width=\linewidth]{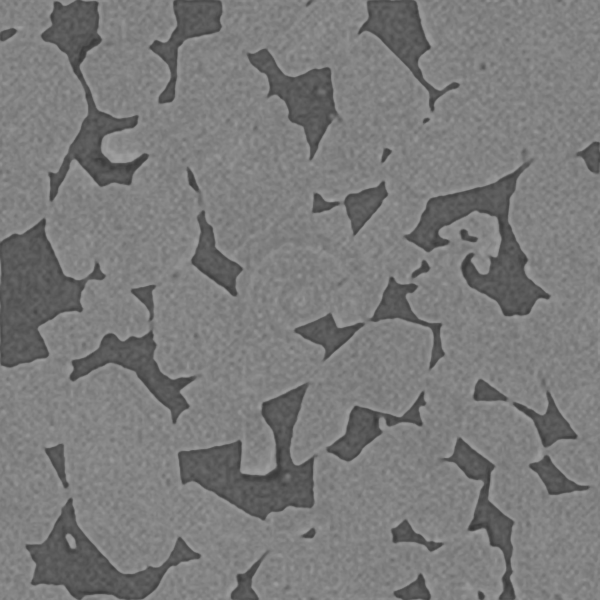}
\end{subfigure}%
\hspace{0.05cm}
\begin{subfigure}[t]{.2\textwidth}
  \includegraphics[width=\linewidth]{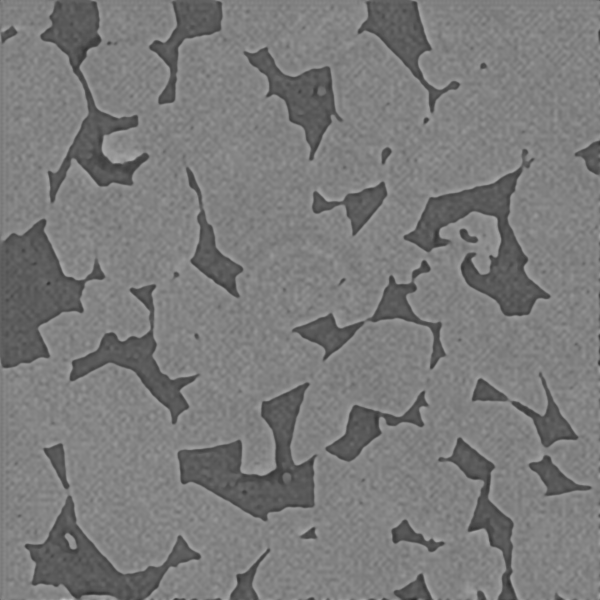}
\end{subfigure}%

\begin{subfigure}[t]{.2\textwidth}
  \includegraphics[width=\linewidth]{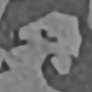}  
\end{subfigure}%
\hspace{0.05cm}
\begin{subfigure}[t]{.2\textwidth}
  \includegraphics[width=\linewidth]{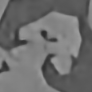}
\end{subfigure}%
\hspace{0.05cm}
\begin{subfigure}[t]{.2\textwidth}
  \includegraphics[width=\linewidth]{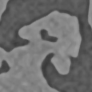}
\end{subfigure}%
\hspace{0.05cm}
\begin{subfigure}[t]{.2\textwidth}
  \includegraphics[width=\linewidth]{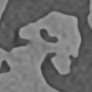}
\end{subfigure}%

\begin{subfigure}[t]{.2\textwidth}
  \includegraphics[width=\linewidth]{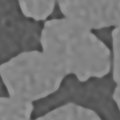}
  \caption*{DIP+TV}
\end{subfigure}%
\hspace{0.05cm}
\begin{subfigure}[t]{.2\textwidth}
  \includegraphics[width=\linewidth]{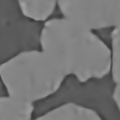}
  \caption*{ACNN}
\end{subfigure}%
\hspace{0.05cm}
\begin{subfigure}[t]{.2\textwidth}
  \centering
  \includegraphics[width=\linewidth]{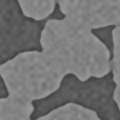}
  \caption*{WPP}
\end{subfigure}%
\hspace{0.05cm}
\begin{subfigure}[t]{.2\textwidth}
  \includegraphics[width=\linewidth]{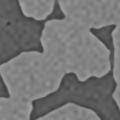}
  \caption*{WPPNet}
\end{subfigure}%
\caption{Comparison of superresolution  for the material ``FS'' with stride 4. The zoomed-in parts are marked with a white box in the HR image.} \label{Comp_HRLRPred_FS}
\end{figure}

\paragraph{Higher Magnification Factor}

Additionally to the previous examples, we apply WPPNets for superresolution with a magnification factor of 6 onto the SiC Diamonds image.
Here, the forward operator $f$ is given by a convolution with a $16 \times 16$ Gaussian blur kernel with standard deviation 3, 
stride 6 and zero-padding. As before, we set the noise to $\xi \sim \mathcal{N}(0,0.01^2)$. 
We use the same ground truth and reference image as before, illustrated in Figure \ref{fig_ground_truth}. 
The resulting quality measures are given in Table~\ref{table_errors_magnif6} and the reconstructions are shown in Figure~\ref{fig_magnificationx6}. 

\begin{table}[b!]
\begin{center}
\scalebox{.85}{
\begin{tabular}[t]{c|c|cccccc} 
             &              & bicubic & PnP-DRUNet& DIP+TV    & ACNN  & WPP    & WPPNet \\
\hline
             & PSNR         &  22.98  & 24.36  & \underline{\textbf{24.82}}  & 24.20   & \textbf{24.51}  &  24.44 \\ 
SiC & Blur Effect  &  0.6875 & 0.6139 & 0.5036 & 0.4476  & \underline{\textbf{0.4081}} &  \textbf{0.4151} \\
             & LPIPS        &  0.5988 & 0.4987 & 0.3602 & 0.3306  & \underline{\textbf{0.2446}} &  \textbf{0.2636} \\
             & SSIM        & 0.5844  & 0.6471 & \textbf{0.6500} & \underline{\textbf{0.6515}} & 0.6488 & 0.6471 \\
             & FSIM        & 0.8371  & 0.8169 & 0.8550 & 0.8371 & \underline{\textbf{0.8653}} & \textbf{0.8641} \\
\hline
Time         &Training     &    -     &  -\footnotemark[6] 
                                                               & -      &  5.5h      & -        & 28h   \\
             &Reconstruction& 0.0003s &  51.36s     &  114.42s    & 0.05s    &  477.06s     & 0.05s  \\
\end{tabular}}
\caption{Comparison of superresolution results for the material  ``SiC'' with  stride 6. The best two values are marked in bold, the best one is additionally underlined.}                    
\label{table_errors_magnif6}
\end{center}
\end{table}

\begin{figure}[t!]
\centering
\begin{subfigure}[t]{.2\textwidth}
  \includegraphics[width=\linewidth]{Results/SiC/img_hr_rectangle.png}
\end{subfigure}%
\hspace{0.05cm}
\begin{subfigure}[t]{.2\textwidth}
  \includegraphics[width=\linewidth]{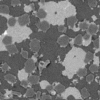}
\end{subfigure}%
\hspace{0.05cm}
\begin{subfigure}[t]{.2\textwidth}
  \includegraphics[width=\linewidth]{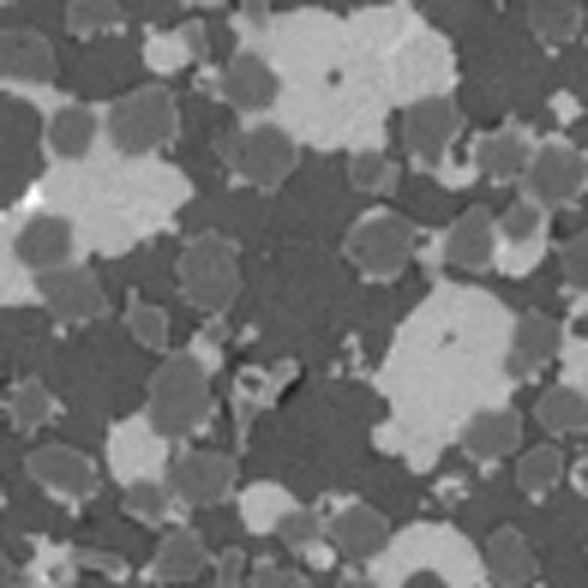}
\end{subfigure}%
\hspace{0.05cm}
\begin{subfigure}[t]{.2\textwidth}
  \includegraphics[width=\linewidth]{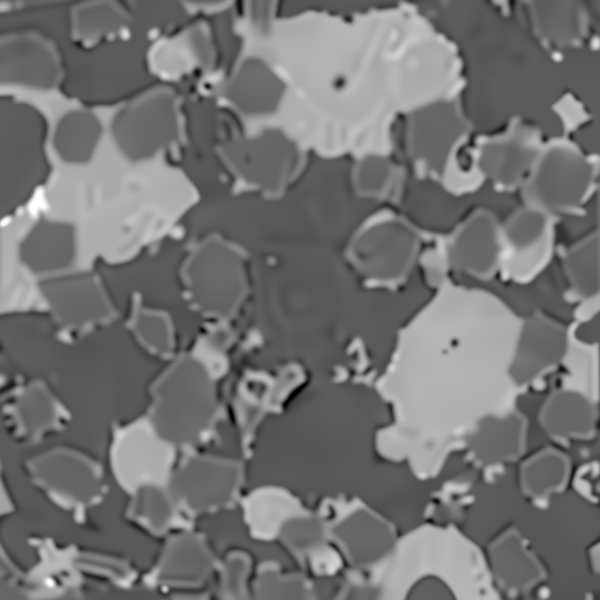}
\end{subfigure}%

\begin{subfigure}[t]{.2\textwidth}
  \includegraphics[width=\linewidth]{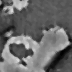}
\end{subfigure}%
\hspace{0.05cm}
\begin{subfigure}[t]{.2\textwidth}
  \includegraphics[width=\linewidth]{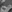}
\end{subfigure}%
\hspace{0.05cm}
\begin{subfigure}[t]{.2\textwidth}
  \includegraphics[width=\linewidth]{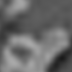}
\end{subfigure}%
\hspace{0.05cm}
\begin{subfigure}[t]{.2\textwidth}
  \includegraphics[width=\linewidth]{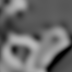}
\end{subfigure}%

\begin{subfigure}[t]{.2\textwidth}
  \includegraphics[width=\linewidth]{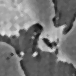}
  \caption*{HR image}
\end{subfigure}%
\hspace{0.05cm}
\begin{subfigure}[t]{.2\textwidth}
  \includegraphics[width=\linewidth]{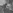}
  \caption*{LR image}
\end{subfigure}%
\hspace{0.05cm}
\begin{subfigure}[t]{.2\textwidth}
  \includegraphics[width=\linewidth]{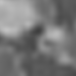}
  \caption*{bicubic}
\end{subfigure}%
\hspace{0.05cm}
\begin{subfigure}[t]{.2\textwidth}
  \includegraphics[width=\linewidth]{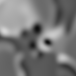}
  \caption*{PnP-DRUNet}
\end{subfigure}%

\begin{subfigure}[t]{.2\textwidth}
  \includegraphics[width=\linewidth]{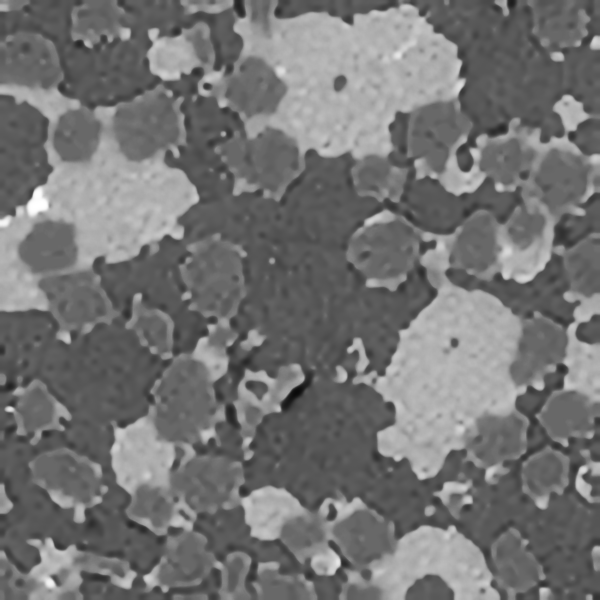}
\end{subfigure}%
\hspace{0.05cm}
\begin{subfigure}[t]{.2\textwidth}
  \includegraphics[width=\linewidth]{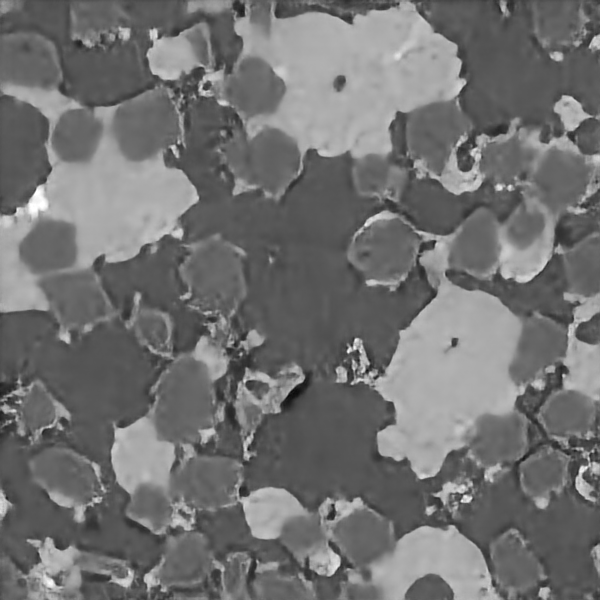}
\end{subfigure}%
\hspace{0.05cm}
\begin{subfigure}[t]{.2\textwidth}
  \includegraphics[width=\linewidth]{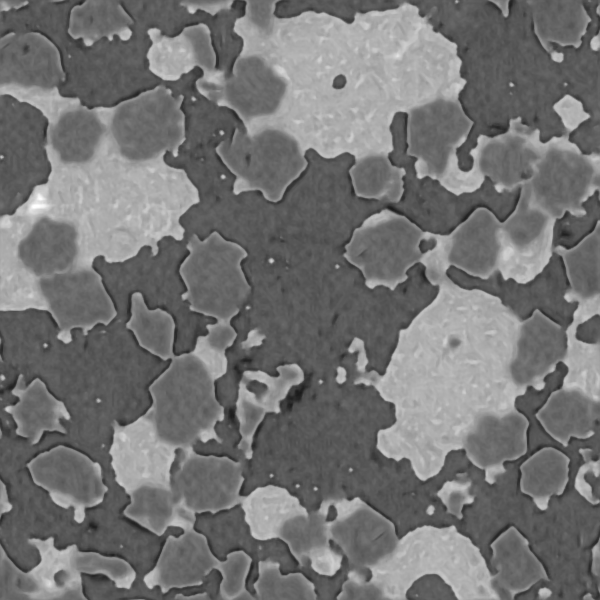}
\end{subfigure}%
\hspace{0.05cm}
\begin{subfigure}[t]{.2\textwidth}
  \includegraphics[width=\linewidth]{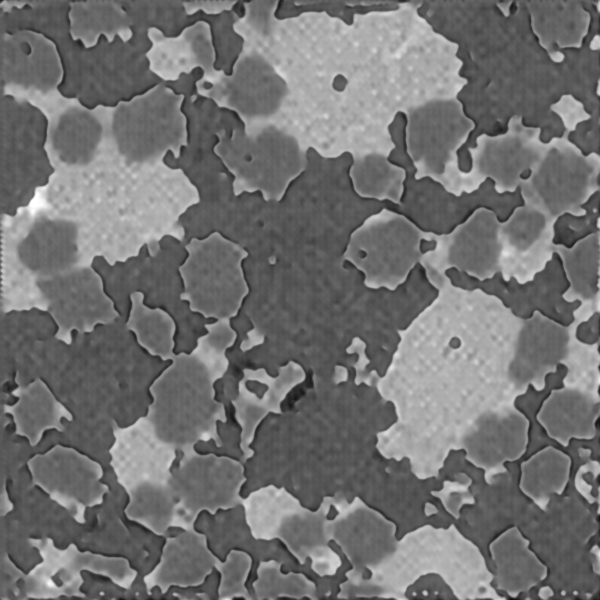}
\end{subfigure}%

\begin{subfigure}[t]{.2\textwidth}
  \includegraphics[width=\linewidth]{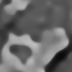}  
\end{subfigure}%
\hspace{0.05cm}
\begin{subfigure}[t]{.2\textwidth}
  \includegraphics[width=\linewidth]{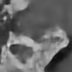}
\end{subfigure}%
\hspace{0.05cm}
\begin{subfigure}[t]{.2\textwidth}
  \includegraphics[width=\linewidth]{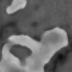}
\end{subfigure}%
\hspace{0.05cm}
\begin{subfigure}[t]{.2\textwidth}
  \includegraphics[width=\linewidth]{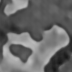}
\end{subfigure}%

\begin{subfigure}[t]{.2\textwidth}
  \includegraphics[width=\linewidth]{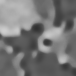}
  \caption*{DIP+TV}
\end{subfigure}%
\hspace{0.05cm}
\begin{subfigure}[t]{.2\textwidth}
  \includegraphics[width=\linewidth]{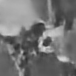}
  \caption*{ACNN}
\end{subfigure}%
\hspace{0.05cm}
\begin{subfigure}[t]{.2\textwidth}
  \centering
  \includegraphics[width=\linewidth]{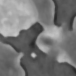}
  \caption*{WPP}
\end{subfigure}%
\hspace{0.05cm}
\begin{subfigure}[t]{.2\textwidth}
  \includegraphics[width=\linewidth]{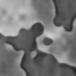}
  \caption*{WPPNet}
\end{subfigure}%
\caption{Comparison of superresolution for the material ``SiC'' with  stride 6. The zoomed-in parts are marked with a white box in the HR image.} \label{fig_magnificationx6}
\end{figure} 

\subsection{Stability under Inaccurate Operators}\label{sec_inaccurate}

In this subsection, we demonstrate the robustness of WPPNets against inaccurate knowledge of the forward operator $f$.
Our data is given by the SiC Diamonds image.

\subsubsection{Inaccurate Forward Operator}

In this example, we generate the low-resolution observations using the forward operator $f_\mathrm{true}$ given by a strided
convolution with a $16\times 16$ Gaussian blur kernel with standard deviation $2$, stride $4$ and zero padding as in the previous 
subsection.
On the other hand, we train the WPPNet and ACNN with an inaccurate forward operator $f_\mathrm{inacc}$, which we also use
for the reconstruction with DIP+TV, PnP-DRUNet and WPP. 
The operator $f_\mathrm{inacc}$ is given in the same way as $f_\mathrm{true}$ with the only difference, that we use a different standard deviation of $1.0$, $1.5$, $2.5$ and $3.0$.

The resulting quality measures are given in Table~\ref{table_errorMeasures_inaccurate} and the reconstructions in 
Figure~\ref{Fig_inaccurateforward}.
We observe that WPPNet and WPP are much more robust against the inaccurate operator than all of the comparison methods. 
If the standard deviation of the inaccurate forward operator is smaller than the true one, we can observe a blur in the reconstructions of PnP-DRUNet, DIP+TV and ACNN, while the WPP and WPPNet reconstructions yield sharp edges. 
For larger standard deviations, there appear artifacts in the reconstructions of DIP+TV and ACNN, while the PnP-DRUNet reconstruction is still blurry.
On the other hand, the reconstruction using WPPNet or WPP is still close to the results using the accurate operator
from the previous section.
Note that for large standard deviation ACNN and DIP+TV have a very small blur effect due to the large number of high-frequency artifacts.

\begin{table}[t]
\begin{center}
\scalebox{.8}{
\begin{tabular}[t]{c|c|ccccc} 
Standard deviation &        & PnP-DRUNet& DIP+TV    & ACNN  & WPP    & WPPNet \\
\hline
             & PSNR         & 24.62  & \textbf{26.15}  & \underline{\textbf{26.17}}  & 25.73  & 26.03 \\ 
1.0          & Blur Effect  & 0.5769 & 0.4860 & 0.5005  & \underline{\textbf{0.3776}} &  \textbf{0.3841} \\
             & LPIPS        & 0.4511 & 0.2969 & 0.2846 & \underline{\textbf{0.1785}} &  \textbf{0.1862} \\ 
             & SSIM        & 0.6555  & \textbf{0.7347} & \underline{\textbf{0.7375}} & 0.7081 & 0.7326  \\
             & FSIM        & 0.8553  & \underline{\textbf{0.9182}} & 0.9074 & \textbf{0.9097} & 0.9087 \\
\hline
Time         &Training       &  -\footnotemark[6]    & -      &  1h      & -     & 17h      \\
             &Reconstruction&  51.36s     &  114.42s    & 0.05s    &  477.06s     & 0.05s  \\
\hline             
\hline
             & PSNR         & 26.13  & \textbf{26.94}  & 26.89  & 26.89  &  \underline{\textbf{27.00}} \\ 
1.5          & Blur Effect  & 0.5379 & 0.4631 & 0.4524  & \underline{\textbf{0.3808}} &  \textbf{0.3891} \\
             & LPIPS        & 0.3907 & 0.2639 & 0.2665 & \underline{\textbf{0.1669}} &  \textbf{0.1851} \\ 
             & SSIM        & 0.7169  & \textbf{0.7571} & \underline{\textbf{0.7613}} & 0.7426 & 0.7538  \\
             & FSIM        & 0.8635  & 0.9112 & 0.9097 & \underline{\textbf{0.9345}} & \textbf{0.9305} \\
\hline
Time         &Training       &  -\footnotemark[6]    & -      &  2h      & -     &  14h     \\
             &Reconstruction&  51.36s     &  114.42s    & 0.05s    &  477.06s     & 0.05s  \\
\hline          
\hline
             & PSNR         & 27.43  & \textbf{27.81}  & 27.66   & 27.57  &  \underline{\textbf{27.83}} \\ 
2.0          & Blur Effect  & 0.4405 & 0.4046 & 0.4076  & \underline{\textbf{0.3743}} &  \textbf{0.3810} \\
(correct operator)             & LPIPS        & 0.3133 & 0.2076 & 0.2441  & \underline{\textbf{0.1627}} &  \textbf{0.1819} \\
             & SSIM         & \textbf{0.7770} & 0.7756 & \underline{\textbf{0.7842}} & 0.7555 & 0.7678 \\
             & FSIM         & 0.8911 & 0.9344 & 0.9130 & \textbf{0.9443} & \underline{\textbf{0.9449}} \\
\hline
Time         &Training       &  -    & -      & 6.5h       & -        &  16h  \\
             &Reconstruction &  51.36s     &  114.42s    & 0.05s    &  477.06s     & 0.05s  \\
\hline
\hline
             & PSNR         & 26.58  & 26.69  & 25.84  & \textbf{27.45}  &  \underline{\textbf{27.48}} \\ 
2.5          & Blur Effect  & 0.4367 & \textbf{0.3766} & \underline{\textbf{0.3445}}  & 0.3876 &  0.3964 \\
             & LPIPS        & 0.3563 & 0.2266 & 0.2420  & \underline{\textbf{0.1725}} &  \textbf{0.2055} \\ 
             & SSIM        & 0.7561  & \underline{\textbf{0.7659}} & 0.7574 & \textbf{0.7617} & 0.7616  \\
             & FSIM        & 0.8751  & 0.9188 & 0.9134 & \textbf{0.9405} & \underline{\textbf{0.9411}} \\
\hline
Time         &Training     &  -\footnotemark[6]   & -      &    2h  & -    & 15h        \\
             &Reconstruction &  51.36s     &  114.42s    & 0.05s    &  477.06s     & 0.05s  \\
\hline             
\hline
             & PSNR         & 24.99  & 23.52  & 22.34  & \textbf{26.73}  &  \underline{\textbf{26.91}} \\ 
3.0          & Blur Effect  & 0.4649 & \textbf{0.3587} & \underline{\textbf{0.3196}}  & 0.4117 &  0.4232 \\
             & LPIPS        & 0.4066 & 0.2688 & 0.2806  & \underline{\textbf{0.2011}} &  \textbf{0.2355} \\ 
             & SSIM        & 0.7049  & 0.7119 & 0.6506 & \textbf{0.7496} & 0.7539  \\
             & FSIM        & 0.8359  & 0.9014 & 0.9031 & \textbf{0.9220} & \underline{\textbf{0.9269}} \\
\hline
Time         &Training     &  -\footnotemark[6]   & -      &  2h  & -    & 12h        \\
             &Reconstruction &  51.36s     &  114.42s    & 0.05s    &  477.06s     & 0.05s  
\end{tabular}}
\caption{Comparison of superresolution results for ``SiC'' with an inaccurate forward blur operator and stride 4. 
The best two values are marked in bold, the best one is additionally underlined. The true forward operator has standard deviation 2.0, the values here are the same as in Table~\ref{table_errorMeasures}.}                   
\label{table_errorMeasures_inaccurate}
\end{center}
\end{table}

\begin{figure}[t!]
\begin{subfigure}[t]{.16\textwidth}
  \includegraphics[width=\linewidth]{Results/SiC/img_hr_rectangle.png}
\end{subfigure}%
\hfill
\begin{subfigure}[t]{.16\textwidth}
  \centering
  \includegraphics[width=\linewidth]{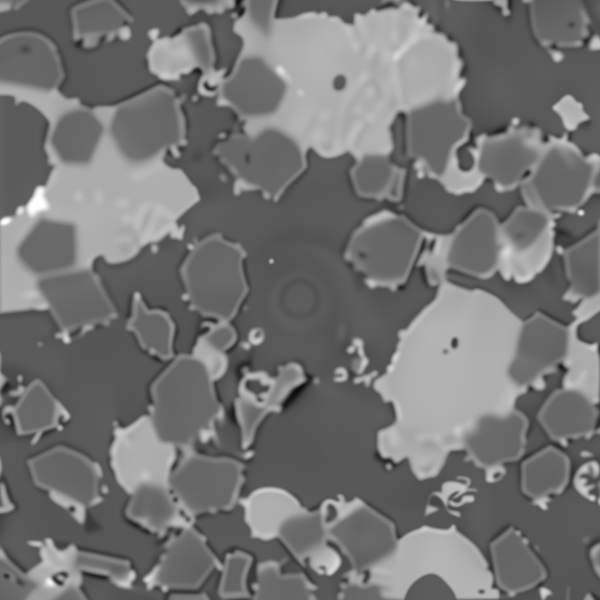}
\end{subfigure}%
\hfill
\begin{subfigure}[t]{.16\textwidth}
  \centering
  \includegraphics[width=\linewidth]{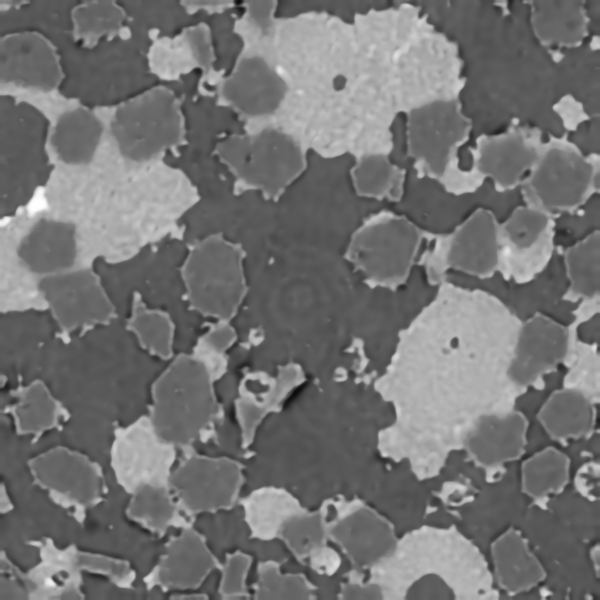}
\end{subfigure}%
\hfill
\begin{subfigure}[t]{.16\textwidth}
  \centering
  \includegraphics[width=\linewidth]{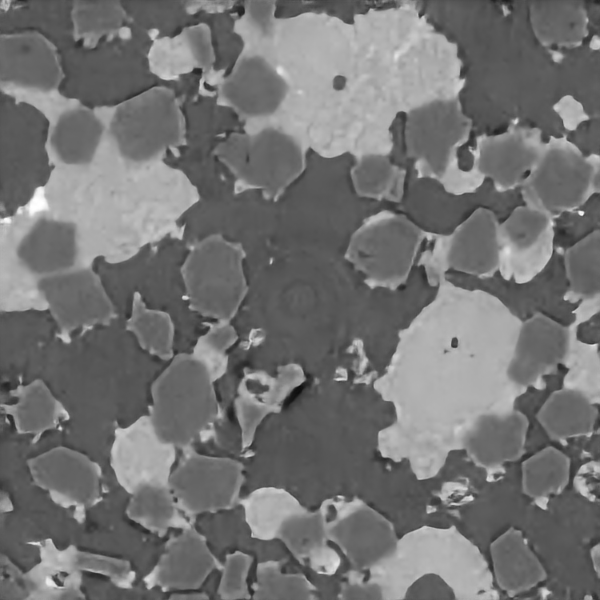}
\end{subfigure}%
\hfill
\begin{subfigure}[t]{.16\textwidth}
  \centering
  \includegraphics[width=\linewidth]{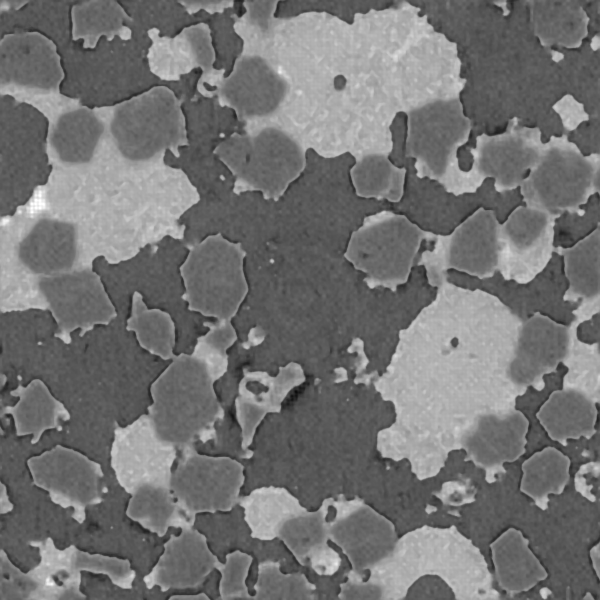}
\end{subfigure}%
\hfill
\begin{subfigure}[t]{.16\textwidth}
  \centering
  \includegraphics[width=\linewidth]{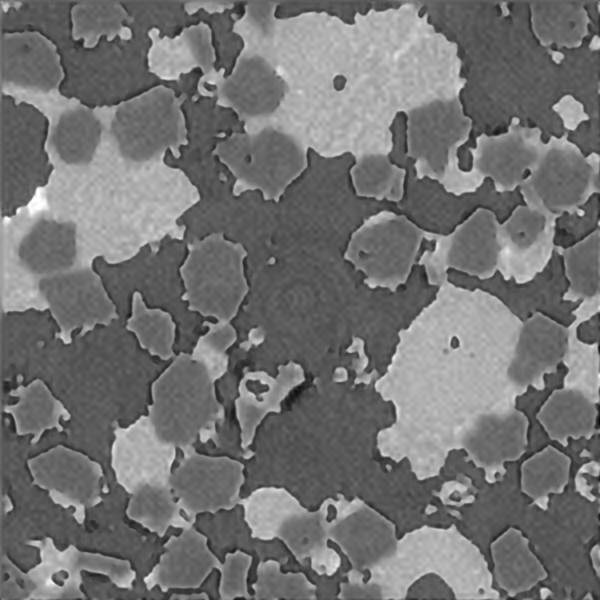}
\end{subfigure}%

\begin{subfigure}[t]{.16\textwidth}
  \centering
  \includegraphics[width=\linewidth]{Results/SiC/hr_zoom1.png}
\end{subfigure}%
\hfill
\begin{subfigure}[t]{.16\textwidth}
  \centering
  \includegraphics[width=\linewidth]{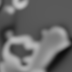}
\end{subfigure}%
\hfill
\begin{subfigure}[t]{.16\textwidth}
  \centering
  \includegraphics[width=\linewidth]{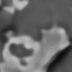}
\end{subfigure}%
\hfill
\begin{subfigure}[t]{.16\textwidth}
  \centering
  \includegraphics[width=\linewidth]{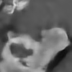}
\end{subfigure}%
\hfill
\begin{subfigure}[t]{.16\textwidth}
  \centering
  \includegraphics[width=\linewidth]{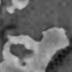}  
\end{subfigure}%
\hfill
\begin{subfigure}[t]{.16\textwidth}
  \centering
  \includegraphics[width=\linewidth]{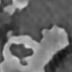}  
\end{subfigure}%

\begin{subfigure}[t]{.16\textwidth}
  \centering
  \includegraphics[width=\linewidth]{Results/SiC/hr_zoom2.png}
  \caption*{HR image}
\end{subfigure}%
\hfill
\begin{subfigure}[t]{.16\textwidth}
  \centering
  \includegraphics[width=\linewidth]{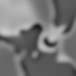}
  \caption*{PnP-DRUNet}
\end{subfigure}%
\hfill
\begin{subfigure}[t]{.16\textwidth}
  \centering
  \includegraphics[width=\linewidth]{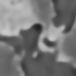}
  \caption*{DIP+TV}
\end{subfigure}%
\hfill
\begin{subfigure}[t]{.16\textwidth}
  \centering
  \includegraphics[width=\linewidth]{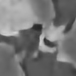}
  \caption*{ACNN}
\end{subfigure}%
\hfill
\begin{subfigure}[t]{.16\textwidth}
  \centering
  \includegraphics[width=\linewidth]{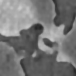}  
  \caption*{WPP}
\end{subfigure}%
\hfill
\begin{subfigure}[t]{.16\textwidth}
  \centering
  \includegraphics[width=\linewidth]{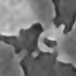}  
  \caption*{WPPNet}
\end{subfigure}


\centering
\begin{subfigure}[t]{.16\textwidth}
  \includegraphics[width=\linewidth]{Results/SiC/img_hr_rectangle.png}
\end{subfigure}%
\hfill
\begin{subfigure}[t]{.16\textwidth}
  \centering
  \includegraphics[width=\linewidth]{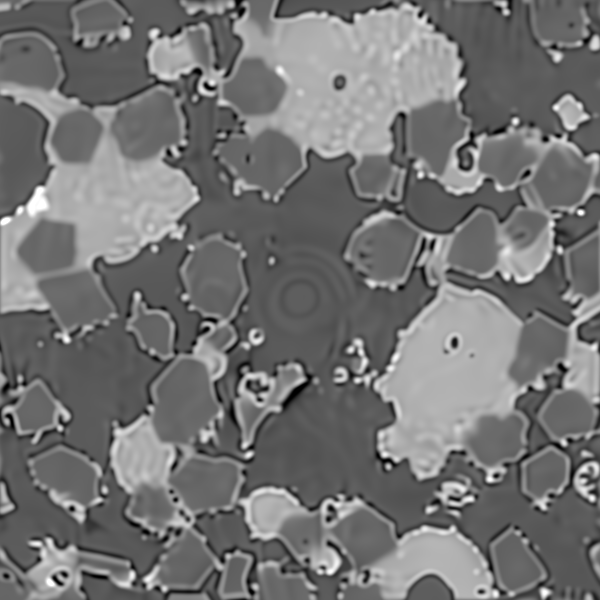}
\end{subfigure}%
\hfill
\begin{subfigure}[t]{.16\textwidth}
  \centering
  \includegraphics[width=\linewidth]{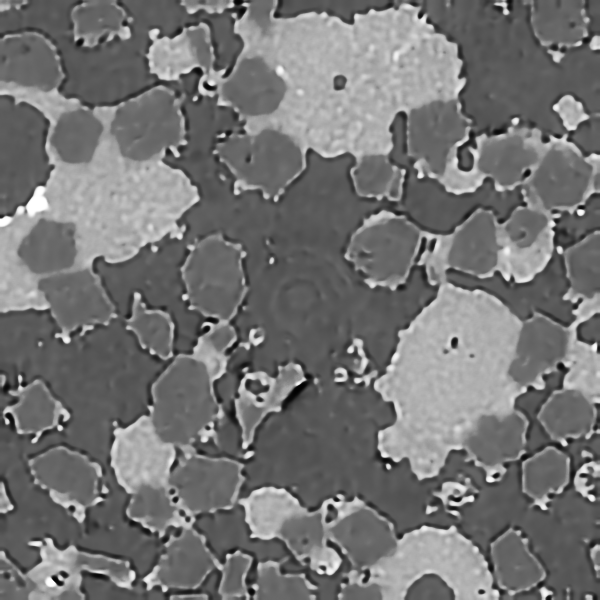}
\end{subfigure}%
\hfill
\begin{subfigure}[t]{.16\textwidth}
  \centering
  \includegraphics[width=\linewidth]{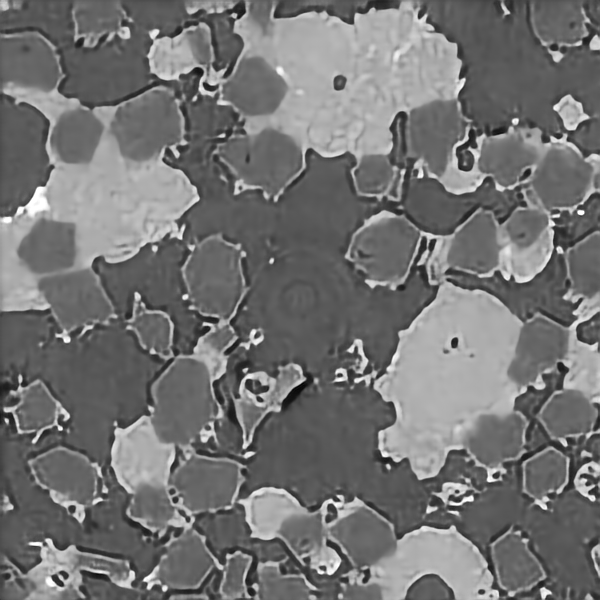}
\end{subfigure}%
\hfill
\begin{subfigure}[t]{.16\textwidth}
  \centering
  \includegraphics[width=\linewidth]{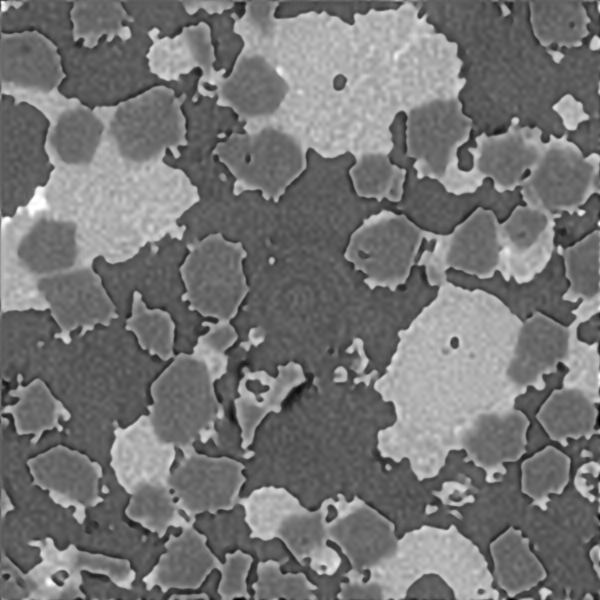}  
\end{subfigure}%
\hfill
\begin{subfigure}[t]{.16\textwidth}
  \centering
  \includegraphics[width=\linewidth]{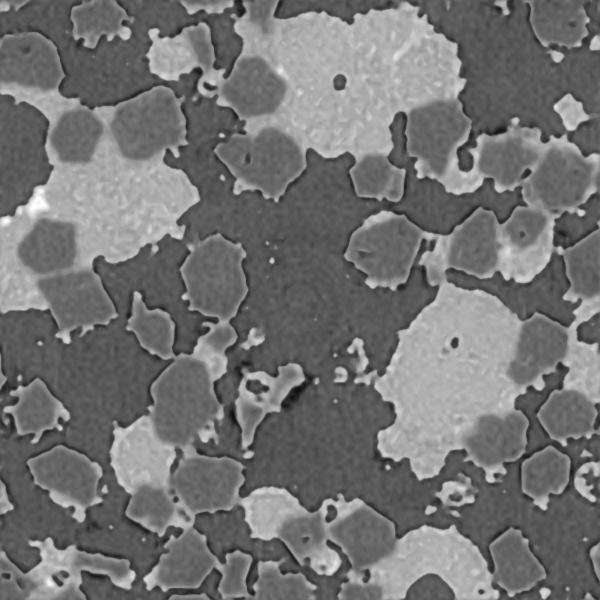}  
\end{subfigure}%

\begin{subfigure}[t]{.16\textwidth}
  \centering
  \includegraphics[width=\linewidth]{Results/SiC/hr_zoom1.png}
\end{subfigure}%
\hfill
\begin{subfigure}[t]{.16\textwidth}
  \centering
  \includegraphics[width=\linewidth]{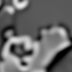}
\end{subfigure}%
\hfill
\begin{subfigure}[t]{.16\textwidth}
  \centering
  \includegraphics[width=\linewidth]{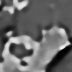}
\end{subfigure}%
\hfill
\begin{subfigure}[t]{.16\textwidth}
  \centering
  \includegraphics[width=\linewidth]{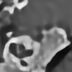}
\end{subfigure}%
\hfill
\begin{subfigure}[t]{.16\textwidth}
  \centering
  \includegraphics[width=\linewidth]{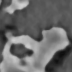}  
\end{subfigure}%
\hfill
\begin{subfigure}[t]{.16\textwidth}
  \centering
  \includegraphics[width=\linewidth]{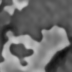}  
\end{subfigure}%

\begin{subfigure}[t]{.16\textwidth}
  \centering
  \includegraphics[width=\linewidth]{Results/SiC/hr_zoom2.png}
  \caption*{HR image}
\end{subfigure}%
\hfill
\begin{subfigure}[t]{.16\textwidth}
  \centering
  \includegraphics[width=\linewidth]{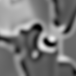}
  \caption*{PnP-DRUNet}
\end{subfigure}%
\hfill
\begin{subfigure}[t]{.16\textwidth}
  \centering
  \includegraphics[width=\linewidth]{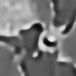}
  \caption*{DIP+TV}
\end{subfigure}%
\hfill
\begin{subfigure}[t]{.16\textwidth}
  \centering
  \includegraphics[width=\linewidth]{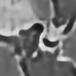}
  \caption*{ACNN}
\end{subfigure}%
\hfill
\begin{subfigure}[t]{.16\textwidth}
  \centering
  \includegraphics[width=\linewidth]{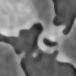}  
  \caption*{WPP}
\end{subfigure}%
\hfill
\begin{subfigure}[t]{.16\textwidth}
  \centering
  \includegraphics[width=\linewidth]{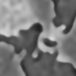}  
  \caption*{WPPNet}
\end{subfigure}%
\caption{Comparison of superresolution for the material ``SiC'' with an inaccurate forward blur operator (standard deviation $1.5$ (top) and $2.5$ (bottom)) and stride 4. The zoomed-in parts are marked with a white box in the HR image.} 
\label{Fig_inaccurateforward}
\end{figure} 

\subsubsection{Estimated Forward Operator}

Finally, we aim to evaluate the performance of WPPNets in a real-world setting motivated by the imaging of material microstructures.
We assume that we have scanned a large area from a materials microstructure using a low-resolution. 
Due to the limited amount of time and resources it is not possible to scan the same area with a higher resolution. 
On the other hand, we assume that we are given a high-resolution image of some small part of this area.

In this setting, we aim to generate a high-resolution correspondence for the whole low-resolution image. We proceed in two steps.
First, we estimate the forward operator of the superresolution problem using the small high-resolution part.
Second, we reconstruct the high-resolution image using the estimated forward operator.

\begin{figure}
\centering
\begin{subfigure}[t]{.2\textwidth}
  \includegraphics[width=\linewidth]{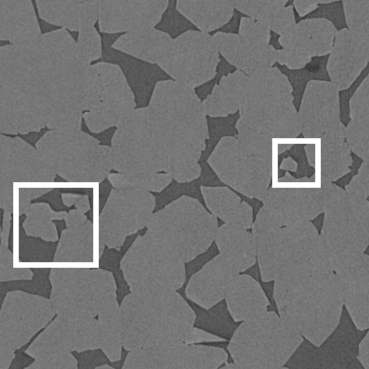}
\end{subfigure}%
\hspace{0.05cm}
\begin{subfigure}[t]{.2\textwidth}
  \includegraphics[width=\linewidth]{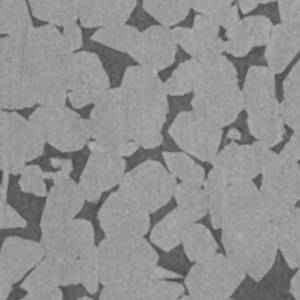}
\end{subfigure}%
\hspace{0.05cm}
\begin{subfigure}[t]{.2\textwidth}
  \includegraphics[width=\linewidth]{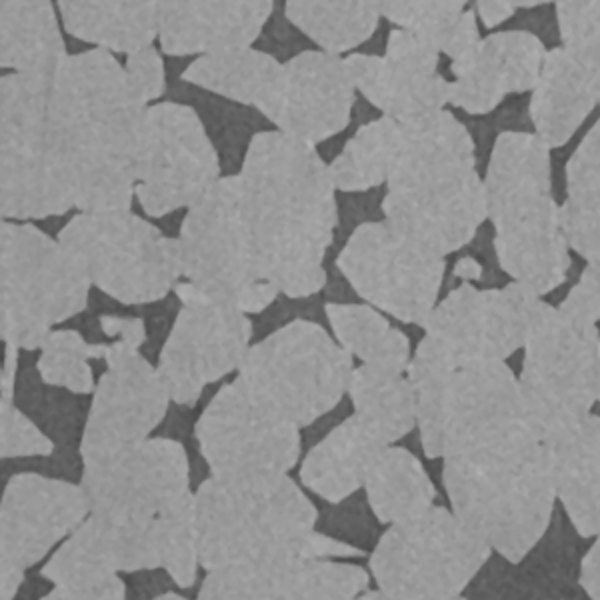}
\end{subfigure}%
\hspace{0.05cm}
\begin{subfigure}[t]{.2\textwidth}
  \includegraphics[width=\linewidth]{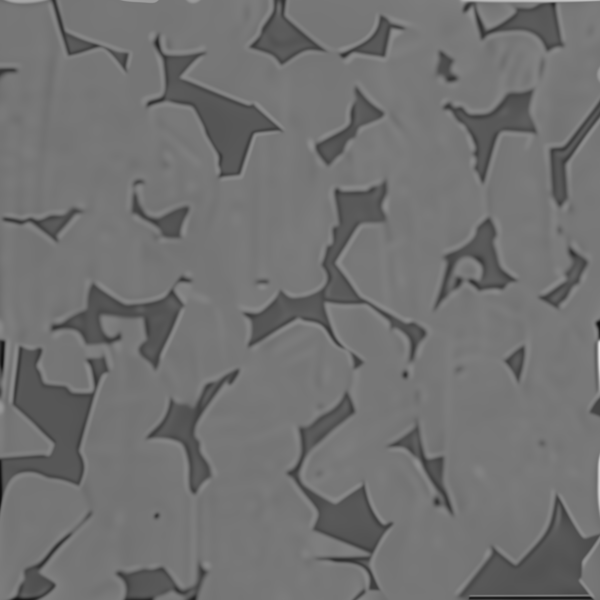}
\end{subfigure}%

\begin{subfigure}[t]{.2\textwidth}
  \includegraphics[width=\linewidth]{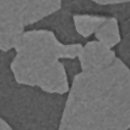}
\end{subfigure}%
\hspace{0.05cm}
\begin{subfigure}[t]{.2\textwidth}
  \includegraphics[width=\linewidth]{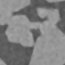}
\end{subfigure}%
\hspace{0.05cm}
\begin{subfigure}[t]{.2\textwidth}
  \includegraphics[width=\linewidth]{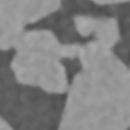}
\end{subfigure}%
\hspace{0.05cm}
\begin{subfigure}[t]{.2\textwidth}
  \includegraphics[width=\linewidth]{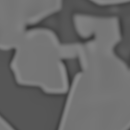}
\end{subfigure}%

\begin{subfigure}[t]{.2\textwidth}
  \includegraphics[width=\linewidth]{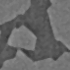}
  \caption*{HR image}
\end{subfigure}%
\hspace{0.05cm}
\begin{subfigure}[t]{.2\textwidth}
  \includegraphics[width=\linewidth]{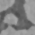}
  \caption*{LR image}
\end{subfigure}%
\hspace{0.05cm}
\begin{subfigure}[t]{.2\textwidth}
  \includegraphics[width=\linewidth]{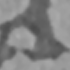}
  \caption*{bicubic}
\end{subfigure}%
\hspace{0.05cm}
\begin{subfigure}[t]{.2\textwidth}
  \includegraphics[width=\linewidth]{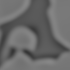}
  \caption*{PnP-DRUNet}
\end{subfigure}%

\begin{subfigure}[t]{.2\textwidth}
  \includegraphics[width=\linewidth]{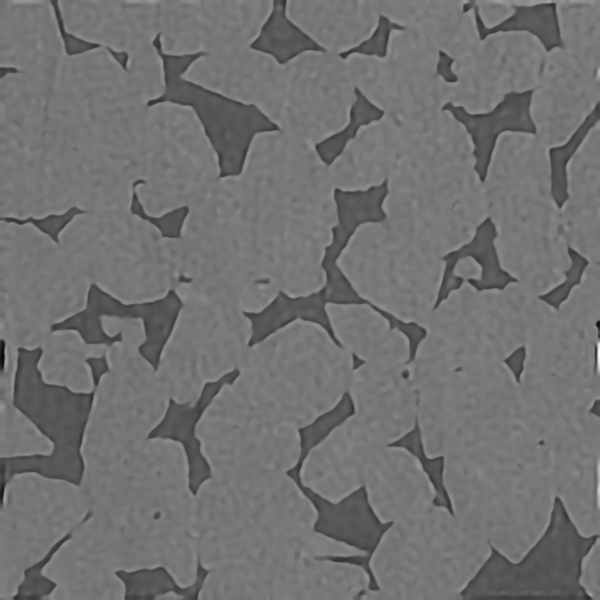}
\end{subfigure}%
\hspace{0.05cm}
\begin{subfigure}[t]{.2\textwidth}
  \includegraphics[width=\linewidth]{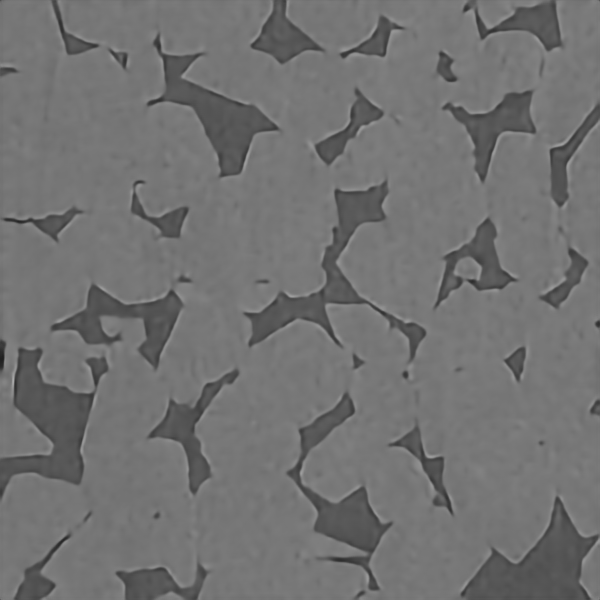}
\end{subfigure}%
\hspace{0.05cm}
\begin{subfigure}[t]{.2\textwidth}
  \includegraphics[width=\linewidth]{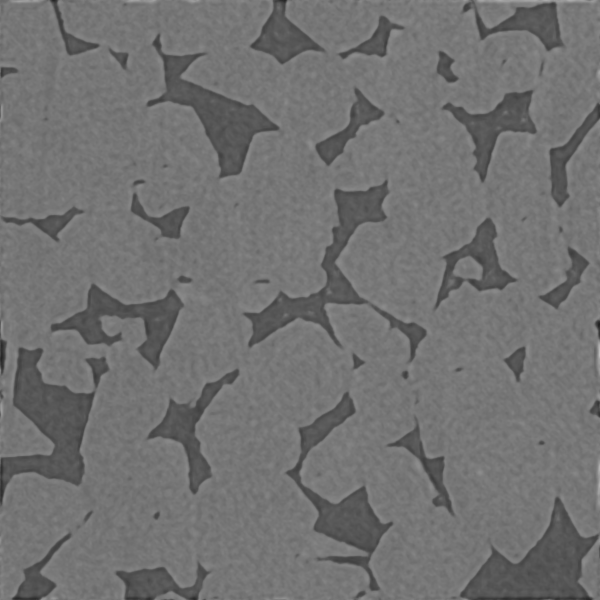}
\end{subfigure}%
\hspace{0.05cm}
\begin{subfigure}[t]{.2\textwidth}
  \includegraphics[width=\linewidth]{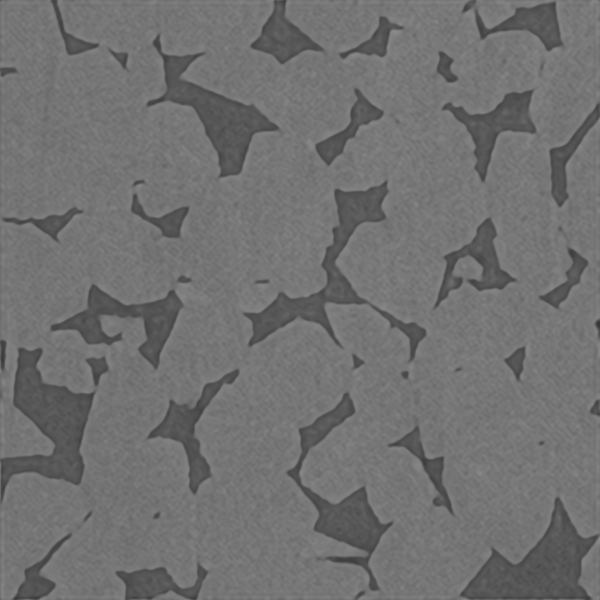}
\end{subfigure}%

\begin{subfigure}[t]{.2\textwidth}
  \includegraphics[width=\linewidth]{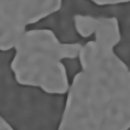}  
\end{subfigure}%
\hspace{0.05cm}
\begin{subfigure}[t]{.2\textwidth}
  \includegraphics[width=\linewidth]{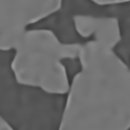}
\end{subfigure}%
\hspace{0.05cm}
\begin{subfigure}[t]{.2\textwidth}
  \includegraphics[width=\linewidth]{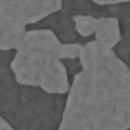}
\end{subfigure}%
\hspace{0.05cm}
\begin{subfigure}[t]{.2\textwidth}
  \includegraphics[width=\linewidth]{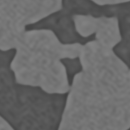}
\end{subfigure}%

\begin{subfigure}[t]{.2\textwidth}
  \includegraphics[width=\linewidth]{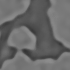}
  \caption*{DIP+TV}
\end{subfigure}%
\hspace{0.05cm}
\begin{subfigure}[t]{.2\textwidth}
  \includegraphics[width=\linewidth]{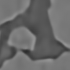}
  \caption*{ACNN}
\end{subfigure}%
\hspace{0.05cm}
\begin{subfigure}[t]{.2\textwidth}
  \centering
  \includegraphics[width=\linewidth]{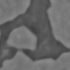}
  \caption*{WPP}
\end{subfigure}%
\hspace{0.05cm}
\begin{subfigure}[t]{.2\textwidth}
  \includegraphics[width=\linewidth]{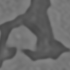}
  \caption*{WPPNet}
\end{subfigure}%
\caption{Reconstruction of the high-resolution image ``Fontainebleau sandstone'' assuming the estimated forward operator. The zoomed-in parts are marked with a white box in the HR image.
} \label{fig_estimatedforward}
\end{figure} 

\paragraph{Data generation and estimation of the Forward Operator}
Now, we consider the synchrotron-computed tomography data from Section~\ref{sec_synchrotron} and consider 
the images with voxel spacings $1.625$ {\textmu}m and $3.25$ {\textmu}m.

Here, we extract a large database of low-resolution images (2D slices of the image with voxel spacing $3.25$ {\textmu}m).
Further, we extract two pairs 
$(\tilde x_1,\tilde y_1)$ and $(\tilde x_2,\tilde y_2)$ of 2D-slices from the images with voxel spacings $1.625$ {\textmu}m and $3.25$ {\textmu}m showing the same area of the
material. Here, the $\tilde x_i$ are extracted from the image with voxel spacing $1.625$ {\textmu}m and the $\tilde y_i$ come from
the image with voxel spacing $3.25$ {\textmu}m.
The pair $(\tilde x_1,\tilde y_1)$ will be used for estimating the operator and the image $\tilde x_1$ (see Figure~\ref{fig_ref_RegFS}) will serve as a reference
image for WPPNet and WPP.
Further, we use the pair $(\tilde x_2,\tilde y_2)$ for evaluating our results, where $\tilde y_2$ is the low-resolution observation,
while $\tilde x_2$ is the high-resolution ground truth.

Note that $(\tilde x_i,\tilde y_i)$ are real-world data. Thus, we register the images $(\tilde x_i,\tilde y_i)$ in a preprocessing
step using the scale-invariant feature transform (SIFT) \cite{L1999}.

In practice, the forward operator could also be estimated from synthetic data (see e.g., \cite{Hertrich21}), which circumvents
the need of the given registered pairs $(\tilde x_i,\tilde y_i)$. However, generating synthetic data for the synchrotron-computed tomography data is
out of scope of our paper.

We estimate the forward operator from the registered pair $(\tilde x_1,\tilde y_1)$ in the same way as proposed in 
\cite{Hertrich21}. For completeness, we describe this procedure in Appendix~\ref{sec_estimation_of_forward}.

\paragraph{Results}
The resulting quality measures are given in Table~\ref{table_errors_estimatedforward} and the reconstructions are shown in
Figure~\ref{fig_estimatedforward}.
We observe that WPPNet and WPP produce significantly sharper and visually better results than the comparisons. This is also
captured by the quality measures.
Note, that there is a change of contrast and brightness between the high-resolution and the low-resolution image. 
As the bicubic interpolation does not consider the operator, this results in a brightness difference of the bicubic interpolation
to the ground truth.

\begin{table}[t]
\begin{center}
\scalebox{.85}{
\begin{tabular}[t]{c|c|cccccc} 
           &              & bicubic & PnP-DRUNet    & DIP +TV   & ACNN  & WPP    & WPPNet \\
\hline
           & PSNR         &  17.67  & \textbf{32.91}  & 32.25 & \underline{\textbf{32.97}}   & 32.68 & 32.86  \\ 
FS         & Blur Effect  &  0.4817 & 0.4900 & 0.3900 & 0.4248  & \textbf{0.3330} & \underline{\textbf{0.3197}}  \\
           & LPIPS        &  0.2805 & 0.3442 & 0.2389 & 0.2882  & \underline{\textbf{0.1362}} & \textbf{0.1411}  \\
           & SSIM         & 0.7280  & \textbf{0.8151} & 0.7850 & \underline{\textbf{0.8196}} & 0.7702 & 0.7724  \\
           & FSIM        & 0.8596  & 0.7903 & 0.8788 & 0.8267 & \textbf{0.9309} & \underline{\textbf{0.9365}} \\
\hline
Time         &Training     &    -     &  -\footnotemark[6] 
                                                               & -      &  2h      & -        & 7h   \\
             &Reconstruction& 0.0003s &  51.36s     &  114.42s    & 0.05s    &  477.06s     & 0.05s  \\
\hline           
\end{tabular}} 
\caption{Comparison of superresolution results using the estimated forward operator. The best two values are marked in bold, the best one is additionally underlined.}
\label{table_errors_estimatedforward}
\end{center}
\end{table}

\subsection{Uncertainty Quantification} \label{Sec_results_uq}

Finally, we aim to detect the uncertainties within the reconstructions with the WPP.
To this end, we consider agian the SiC Diamonds images.

\paragraph{Magnification Factor 4}

Here we use the same forward operator $f$ as in Section \ref{sec_synchrotron}, thus the same 1000 low-resolution images of size $25 \times 25$ are used for training the WPPFlow.

In Figure \ref{Fig_WPPSRFlow_x4} (top) we show three different high-resolution predictions with a magnification factor 4 given the same low-resolution image. Note that here we considered the same reference image $\tilde{x}$ as in Figure \ref{Comp_HRLRPred_SiC}. We computed 100 high-resolution predictions and the resulting standard deviation is given in the right part. Here the brighter a pixel is, the less secure is the WPPFlow in its prediction; the maximal pixel-wise standard deviation is 0.08.  As expected, we have a high uncertainty on the edges, while there is nearly no uncertainty in the other regions. 

As a comparison, in Figure \ref{Fig_WPPSRFlow_x4} (bottom) we show three different high-resolution predictions of the SRFlow. It is trained on DIV2K \cite{AT2017} for 1500 epochs using the negative log-likelihood which is equivalent to interchanging the two arguments within \eqref{eq_backward_KL}. The reconstructions admit a lot of artifacts and the pixel-wise standard deviation is not only visible on the edges. 
In Figure \ref{Fig_WPPSRFlow_x4} (right) we compare the mean of 100 reconstructions of WPPFlow and SRFlow. Although each reconstruction of SRFlow admits artifacts, the mean image yields a good reconstruction. This can be also seen in Table \ref{table_errors_NF}. Here we compare the quality measures of the mean images (left) and the averaged quality measures of the reconstructions (right). Whereas the WPPFlow reconstructions are much better, visually and in terms of the quality measures, for the mean image the SRFlow yield a better reconstruction.

\begin{figure}[t!]
\centering
\begin{subfigure}[t]{.14\textwidth}
  \includegraphics[width=\linewidth]{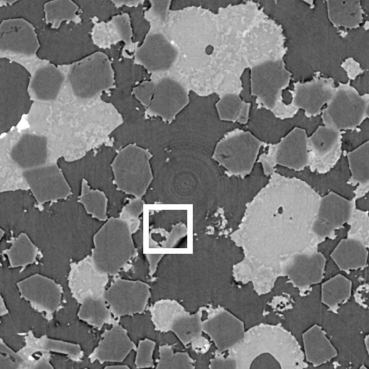}
\end{subfigure}%
\hfill
\begin{subfigure}[t]{.14\textwidth}
  \includegraphics[width=\linewidth]{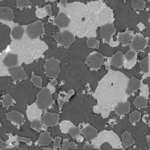}
\end{subfigure}%
\hfill
\begin{subfigure}[t]{.14\textwidth}
  \includegraphics[width=\linewidth]{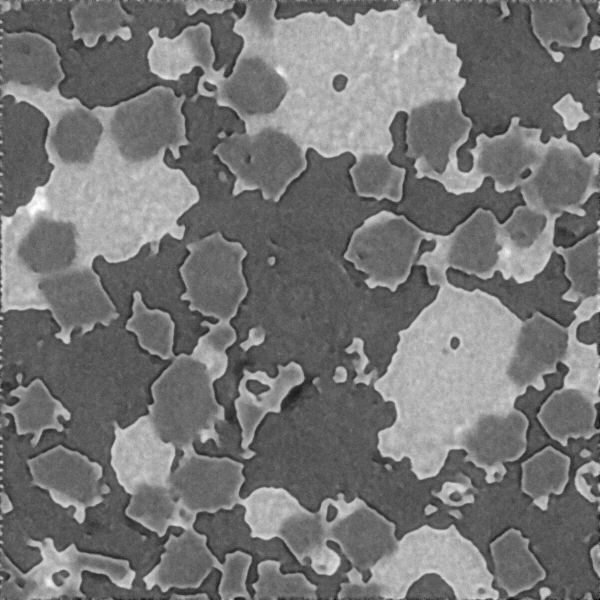}
\end{subfigure}%
\hfill
\begin{subfigure}[t]{.14\textwidth}
  \includegraphics[width=\linewidth]{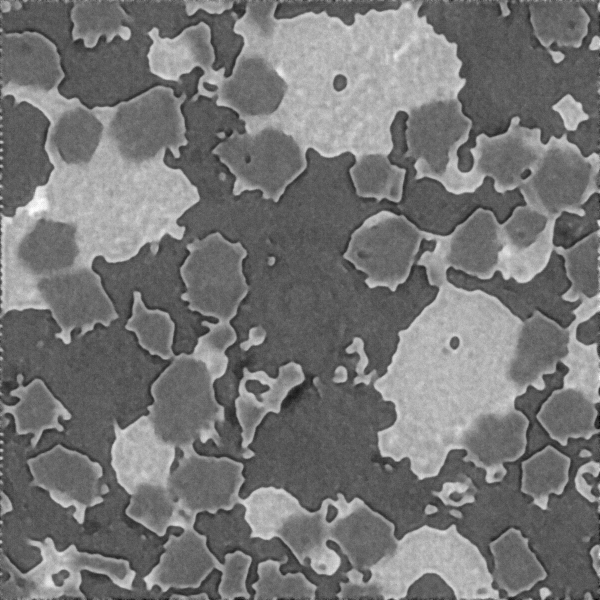}
\end{subfigure}%
\hfill
\begin{subfigure}[t]{.14\textwidth}
  \includegraphics[width=\linewidth]{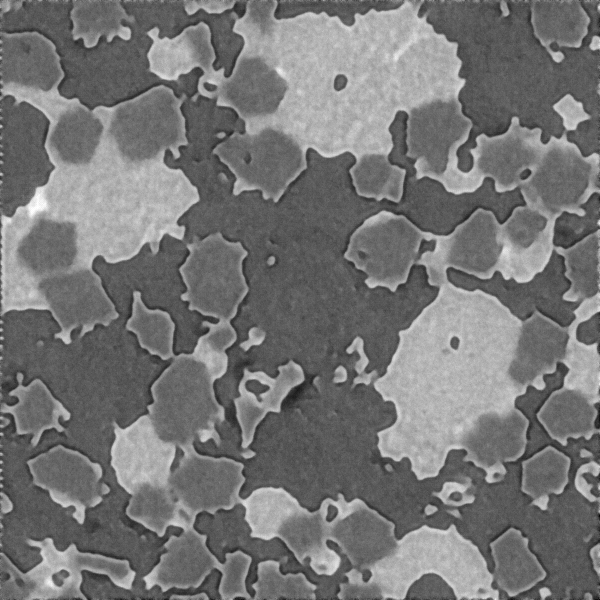}
\end{subfigure}%
\hfill
\begin{subfigure}[t]{.14\textwidth}
  \includegraphics[width=\linewidth]{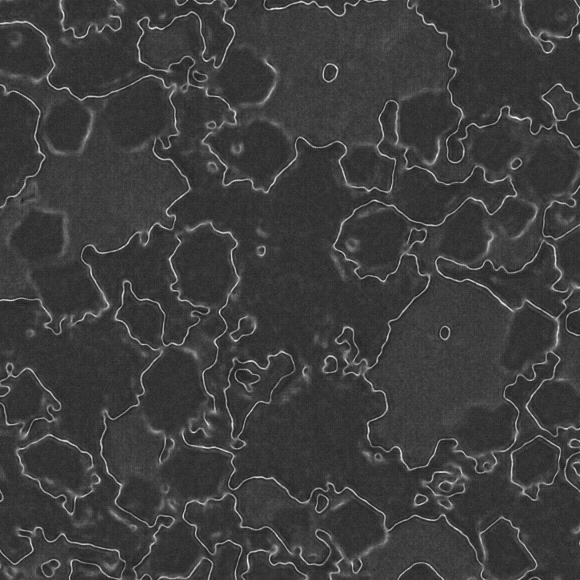}
\end{subfigure}%
\hfill
\begin{subfigure}[t]{.14\textwidth}
  \includegraphics[width=\linewidth]{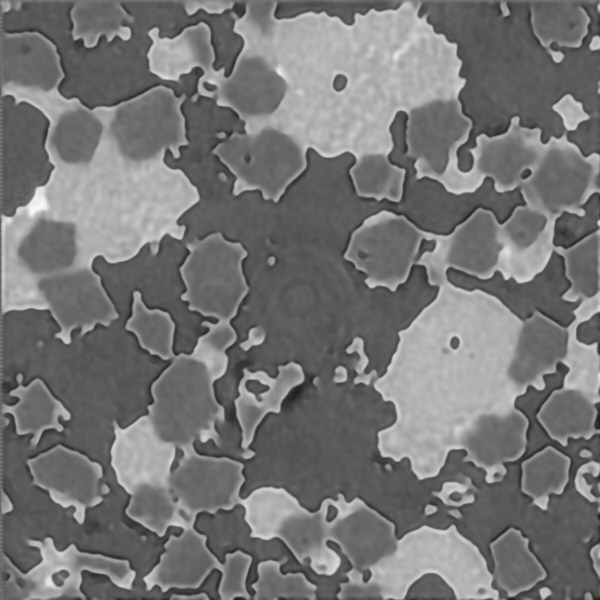}
\end{subfigure}%

\begin{subfigure}[t]{.14\textwidth}
  \includegraphics[width=\linewidth]{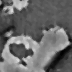}
    \caption*{HR image}
\end{subfigure}%
\hfill
\begin{subfigure}[t]{.14\textwidth}
  \includegraphics[width=\linewidth]{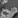}
    \caption*{LR image}
\end{subfigure}%
\hfill
\begin{subfigure}[t]{.14\textwidth}
  \includegraphics[width=\linewidth]{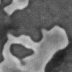}
  \caption*{Prediction 1}
\end{subfigure}%
\hfill
\begin{subfigure}[t]{.14\textwidth}
  \includegraphics[width=\linewidth]{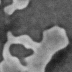}
  \caption*{Prediction 2}
\end{subfigure}%
\hfill
\begin{subfigure}[t]{.14\textwidth}
  \includegraphics[width=\linewidth]{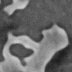}
  \caption*{Prediction 3}
\end{subfigure}%
\hfill
\begin{subfigure}[t]{.14\textwidth}
  \includegraphics[width=\linewidth]{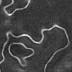}
  \caption*{Standard deviation}
\end{subfigure}%
\hfill
\begin{subfigure}[t]{.14\textwidth}
  \includegraphics[width=\linewidth]{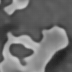}
    \caption*{Mean img}
\end{subfigure}%

\begin{subfigure}[t]{.14\textwidth}
  \includegraphics[width=\linewidth]{Results/UQ/x4_sic/img_hr_rectangle2.png}
\end{subfigure}%
\hfill
\begin{subfigure}[t]{.14\textwidth}
  \includegraphics[width=\linewidth]{Results/UQ/x4_sic/img_hr_LRversion.png}
\end{subfigure}%
\hfill
\begin{subfigure}[t]{.14\textwidth}
  \includegraphics[width=\linewidth]{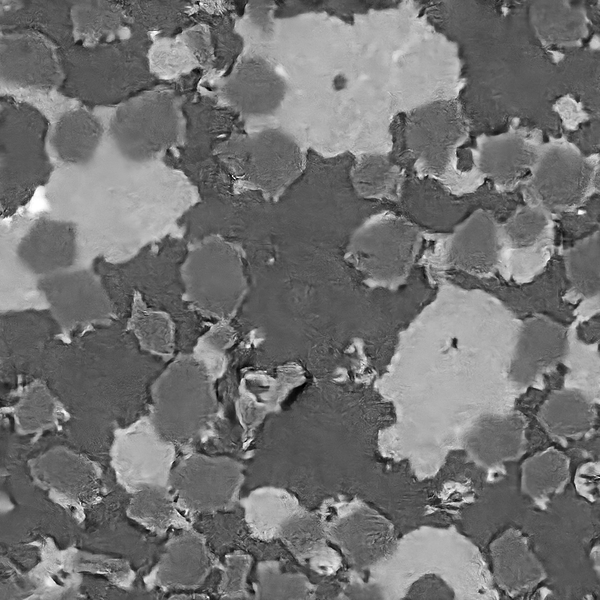}
\end{subfigure}%
\hfill
\begin{subfigure}[t]{.14\textwidth}
  \includegraphics[width=\linewidth]{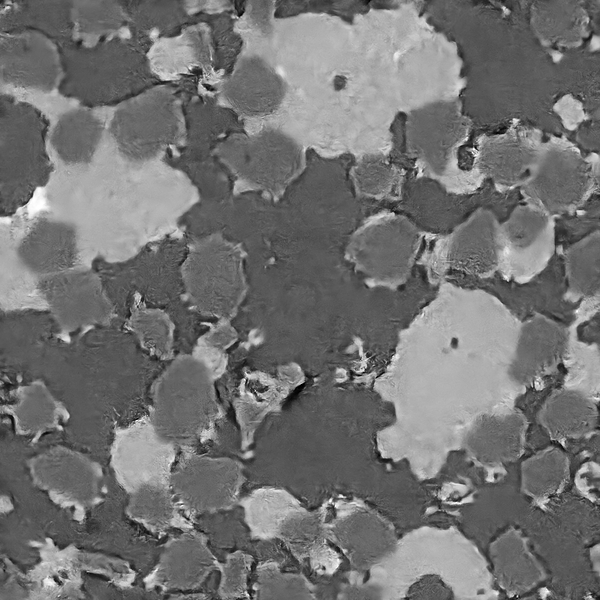}
\end{subfigure}%
\hfill
\begin{subfigure}[t]{.14\textwidth}
  \includegraphics[width=\linewidth]{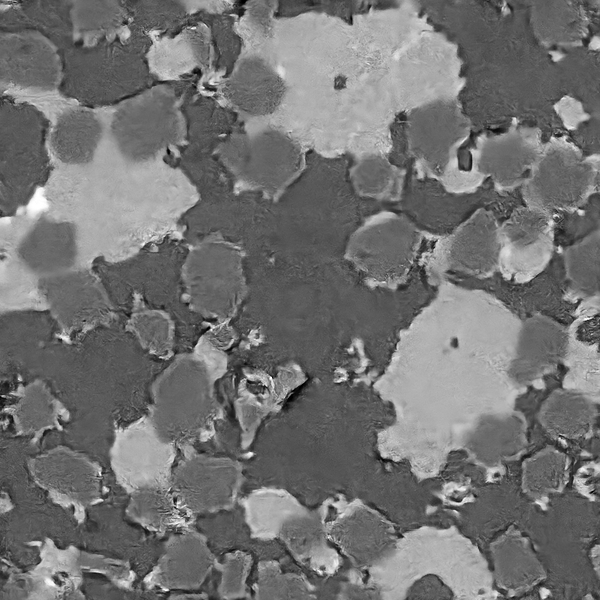}
\end{subfigure}%
\hfill
\begin{subfigure}[t]{.14\textwidth}
  \includegraphics[width=\linewidth]{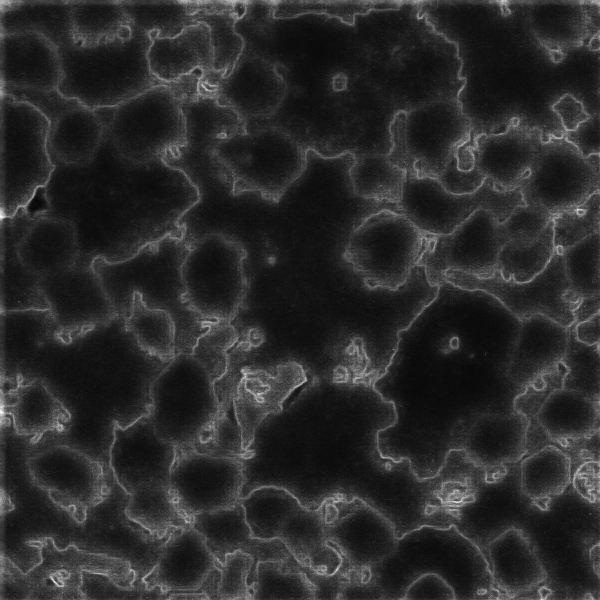}
\end{subfigure}%
\hfill
\begin{subfigure}[t]{.14\textwidth}
  \includegraphics[width=\linewidth]{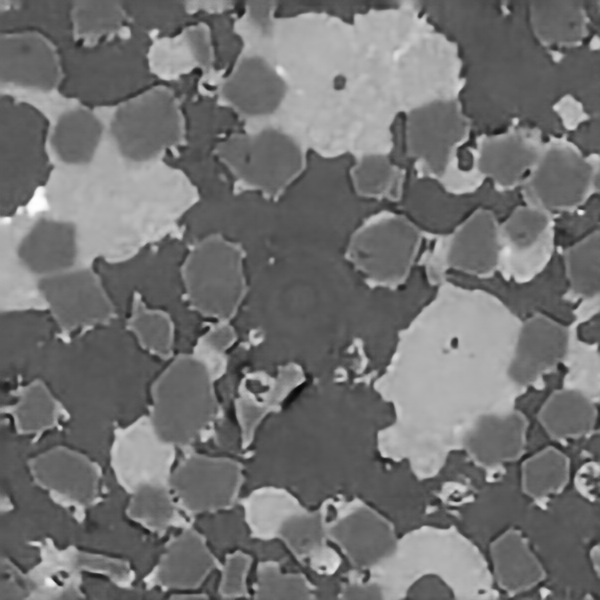}
\end{subfigure}%

\begin{subfigure}[t]{.14\textwidth}
  \includegraphics[width=\linewidth]{Results/UQ/x4_sic/zoom/hr_zoom1.png}
    \caption*{HR image}
\end{subfigure}%
\hfill
\begin{subfigure}[t]{.14\textwidth}
  \includegraphics[width=\linewidth]{Results/UQ/x4_sic/zoom/lr_zoom1.png}
    \caption*{LR image}
\end{subfigure}%
\hfill
\begin{subfigure}[t]{.14\textwidth}
  \includegraphics[width=\linewidth]{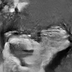}
  \caption*{Prediction 1}
\end{subfigure}%
\hfill
\begin{subfigure}[t]{.14\textwidth}
  \includegraphics[width=\linewidth]{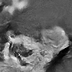}
  \caption*{Prediction 2}
\end{subfigure}%
\hfill
\begin{subfigure}[t]{.14\textwidth}
  \includegraphics[width=\linewidth]{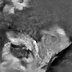}
  \caption*{Prediction 3}
\end{subfigure}%
\hfill
\begin{subfigure}[t]{.14\textwidth}
  \includegraphics[width=\linewidth]{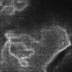}
  \caption*{Standard deviation}
\end{subfigure}%
\hfill
\begin{subfigure}[t]{.14\textwidth}
  \includegraphics[width=\linewidth]{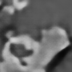}
    \caption*{Mean img}
\end{subfigure}%
\caption{Different WPPFlow (top) and SRFlow (bottom) reconstructions of the ground truth image with stride 4. \textit{Right}: The normalized standard deviation and the mean image.
The zoomed-in part is marked with a white box in the HR image.
} \label{Fig_WPPSRFlow_x4}
\end{figure}

\begin{table}[t]
\begin{center}
\scalebox{.85}{
\begin{tabular}[t]{c|c|cc|cccc} 
           &              & Mean img WPPFlow & Mean img SRFlow    & WPPFlow   & SRFlow \\
\hline
           & PSNR         &  26.68  & \textbf{27.20}  & \textbf{25.74} & 24.79    \\ 
x4         & Blur Effect  &  \textbf{0.4225} & 0.4586 & 0.3539 & \textbf{0.3419}   \\
           & LPIPS        &  \textbf{0.2613} & 0.2868 & \textbf{0.2138} & 0.3090    \\
           & SSIM         & \textbf{0.7662}  & 0.7642 & \textbf{0.6567} & 0.6272  \\
           & FSIM        & 0.9023  & \textbf{0.9027} & \textbf{0.9110} & 0.9064  \\
\hline
Time         &Training     &   -     &  -   & 25h   &  50h    \\
             &Reconstruction& - &  -     &  0.06s    &   0.06s     \\
\hline        
\hline
           & PSNR         &  23.13  & \textbf{23.28}  & \textbf{21.76} & 19.53    \\ 
x8         & Blur Effect  &  \textbf{0.4930} & 0.5977 & 0.3409 & \textbf{0.2743}   \\
           & LPIPS        &  \textbf{0.4046} & 0.4753 & \textbf{0.3495} & 0.4351    \\
           & SSIM         & \textbf{0.6281}  & 0.5715 & \textbf{0.4481} & 0.3106  \\
           & FSIM        & 0.7859  & \textbf{0.8313} & \textbf{0.7687} & 0.7231  \\
\hline
Time         &Training     &    -     &  -   &  12h  & 42h       \\
             &Reconstruction& - &  -     &  0.06s    & 0.06s     \\
\hline       
\end{tabular}} 
\caption{Comparison of superresolution results using the normalizing flows. The best value is marked in bold.}
\label{table_errors_NF}
\end{center}
\end{table}

\paragraph{Magnification Factor 8}
Additionally, we consider the superresolution task with a magnification factor 8. Here, the forward operator $f$ is given by a convolution with a $16 \times 16$ Gaussian blur kernel with standard deviation 4, 
stride 8 and zero-padding to keep the dimensions consistent. Again, we set the noise to $\xi \sim \mathcal{N}(0,0.01^2)$.
For training we use a generated set consisting of 270 low-resolution images of size $20\times20$.

In Figure \ref{Fig_cINN_x8} (top) we show different reconstructions with a magnification factor 8 of a given low-resolution image. Of course, here the differences between the predictions are much larger and we have a higher standard deviation; the maximal pixel-wise standard deviation is 0.167. In contrast to magnification factor 4, where the uncertainties are mainly visible on the edges, in the case of  magnification factor 8 the uncertainties are visible on a much larger area. 
In particular, we observe that WPPFlows are able to detect uncertainties in the topological structure of the data, e.g., if two certain structures are connected or not.

In Figure \ref{Fig_cINN_x8} (bottom) we show three different high-resolution predictions of the SRFlow. It is trained on DIV2K as before for 1800 epochs. Again, the reconstructions admit a lot of artifacts and the pixel-wise standard deviation is not only visible on the edges. 
The mean of 100 reconstructions of WPPFlow and SRFlow is shown in Figure \ref{Fig_cINN_x8} (right). In contrast to magnification factor 4, here the mean image of the WPPFlow is visually better. For the quality measues, see again Table \ref{table_errors_NF}.

\begin{figure}[t!]
\centering
\begin{subfigure}[t]{.14\textwidth}
  \includegraphics[width=\linewidth]{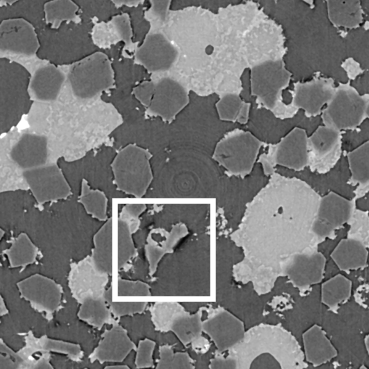}
\end{subfigure}%
\hfill
\begin{subfigure}[t]{.14\textwidth}
  \includegraphics[width=\linewidth]{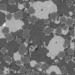}
\end{subfigure}%
\hfill
\begin{subfigure}[t]{.14\textwidth}
  \includegraphics[width=\linewidth]{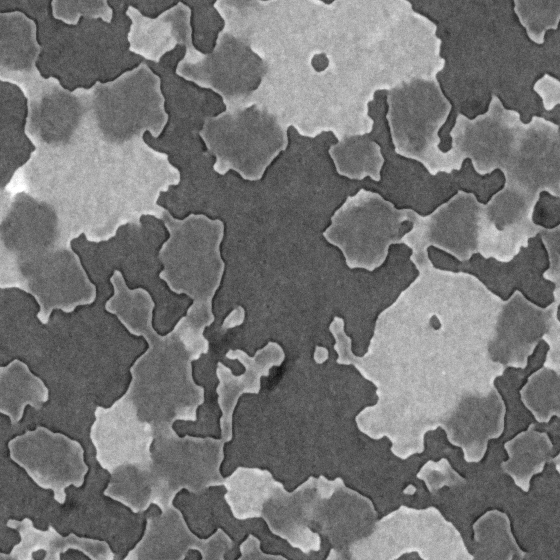}
\end{subfigure}%
\hfill
\begin{subfigure}[t]{.14\textwidth}
  \includegraphics[width=\linewidth]{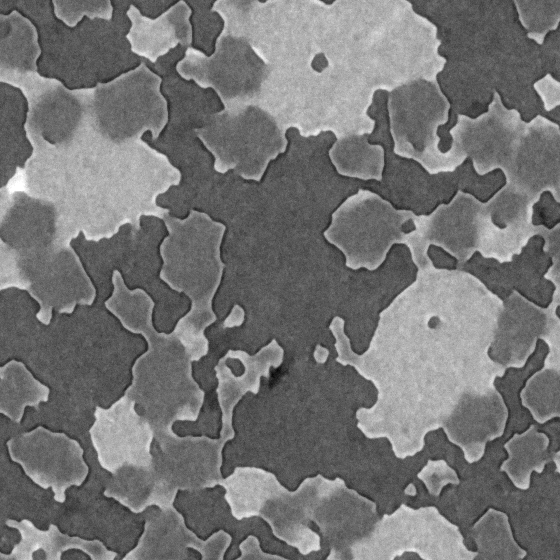}
\end{subfigure}%
\hfill
\begin{subfigure}[t]{.14\textwidth}
  \includegraphics[width=\linewidth]{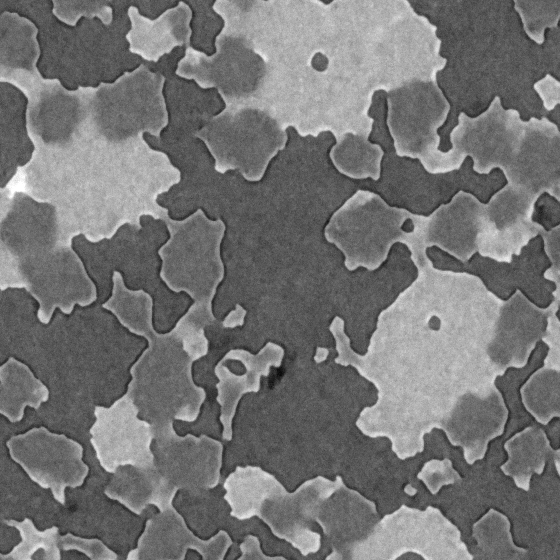}
\end{subfigure}%
\hfill
\begin{subfigure}[t]{.14\textwidth}
  \includegraphics[width=\linewidth]{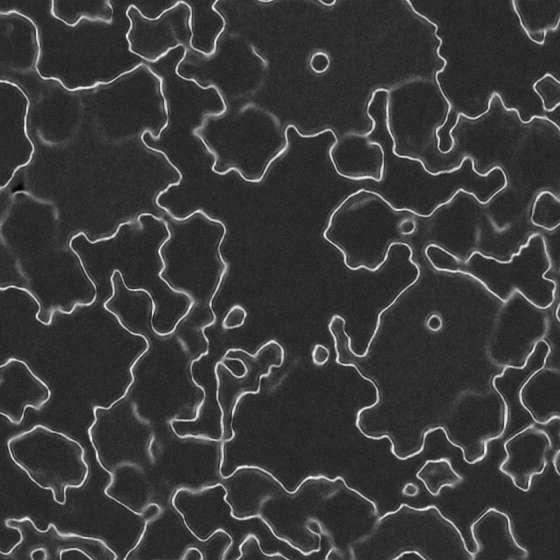}
\end{subfigure}%
\begin{subfigure}[t]{.14\textwidth}
  \includegraphics[width=\linewidth]{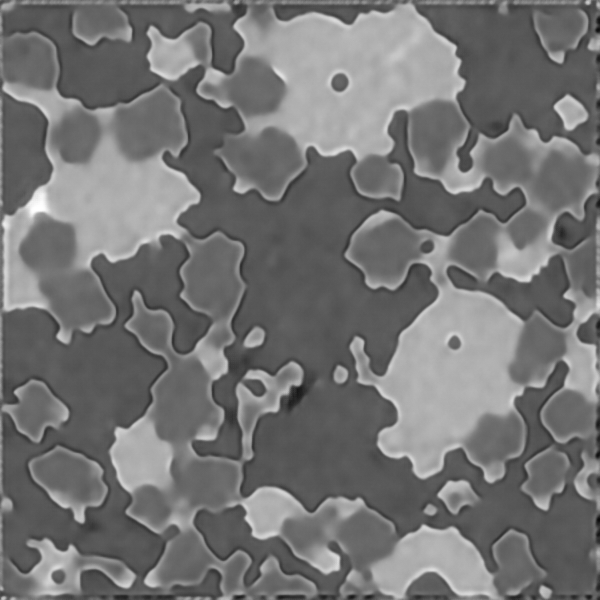}
\end{subfigure}%

\begin{subfigure}[t]{.14\textwidth}
  \includegraphics[width=\linewidth]{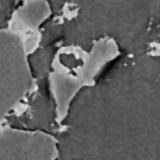}
    \caption*{HR image}
\end{subfigure}%
\hfill
\begin{subfigure}[t]{.14\textwidth}
  \includegraphics[width=\linewidth]{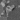}
  \caption*{LR image}
\end{subfigure}%
\hfill
\begin{subfigure}[t]{.14\textwidth}
  \includegraphics[width=\linewidth]{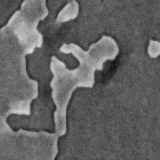}
  \caption*{Prediction 1}
\end{subfigure}%
\hfill
\begin{subfigure}[t]{.14\textwidth}
  \includegraphics[width=\linewidth]{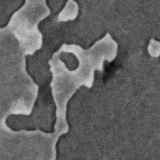}
  \caption*{Prediction 2}
\end{subfigure}%
\hfill
\begin{subfigure}[t]{.14\textwidth}
  \includegraphics[width=\linewidth]{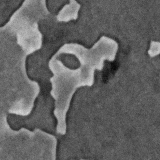}
  \caption*{Prediction 3}
\end{subfigure}%
\hfill
\begin{subfigure}[t]{.14\textwidth}
  \includegraphics[width=\linewidth]{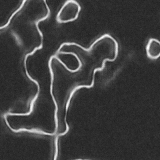}
  \caption*{Standard deviation}
\end{subfigure}%
\begin{subfigure}[t]{.14\textwidth}
  \includegraphics[width=\linewidth]{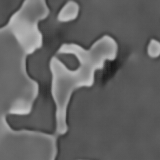}
    \caption*{Mean img}
\end{subfigure}%

\begin{subfigure}[t]{.14\textwidth}
  \includegraphics[width=\linewidth]{Results/UQ/x8_sic/img_hr_rectangle3.png}
\end{subfigure}%
\hfill
\begin{subfigure}[t]{.14\textwidth}
  \includegraphics[width=\linewidth]{Results/UQ/x8_sic/lr_img600_x8.png}
\end{subfigure}%
\hfill
\begin{subfigure}[t]{.14\textwidth}
  \includegraphics[width=\linewidth]{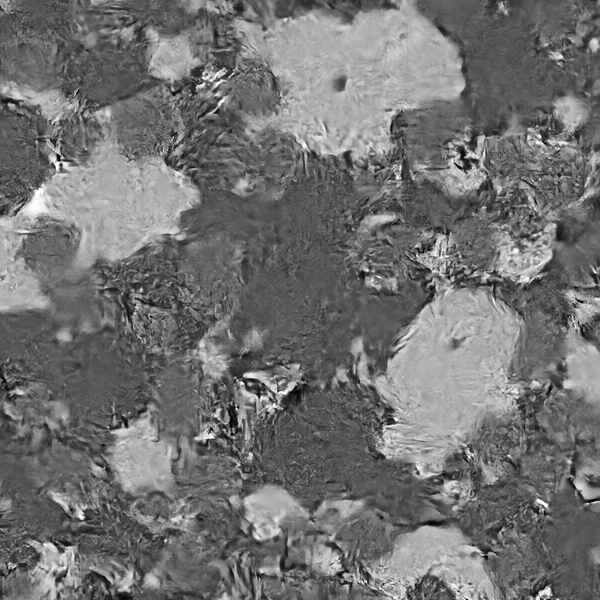}
\end{subfigure}%
\hfill
\begin{subfigure}[t]{.14\textwidth}
  \includegraphics[width=\linewidth]{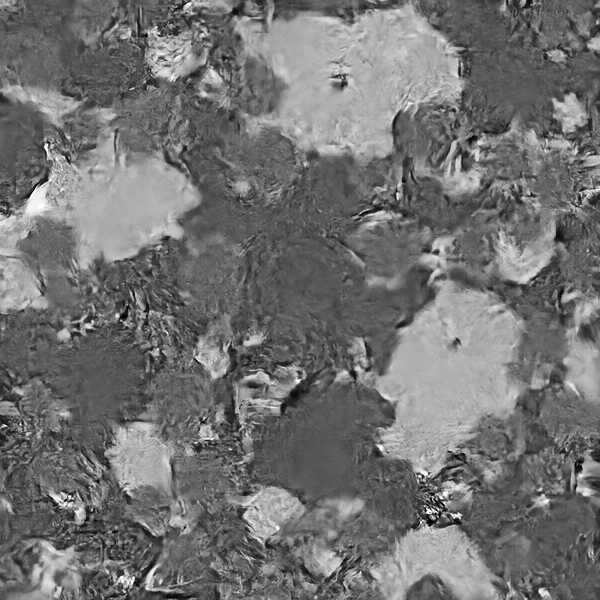}
\end{subfigure}%
\hfill
\begin{subfigure}[t]{.14\textwidth}
  \includegraphics[width=\linewidth]{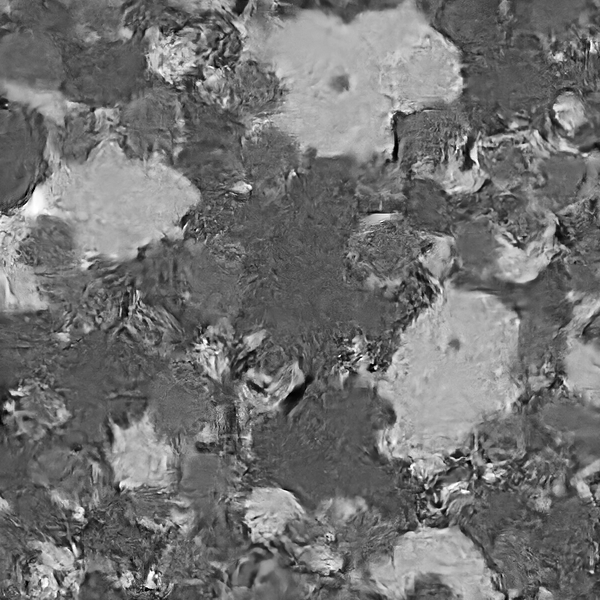}
\end{subfigure}%
\hfill
\begin{subfigure}[t]{.14\textwidth}
  \includegraphics[width=\linewidth]{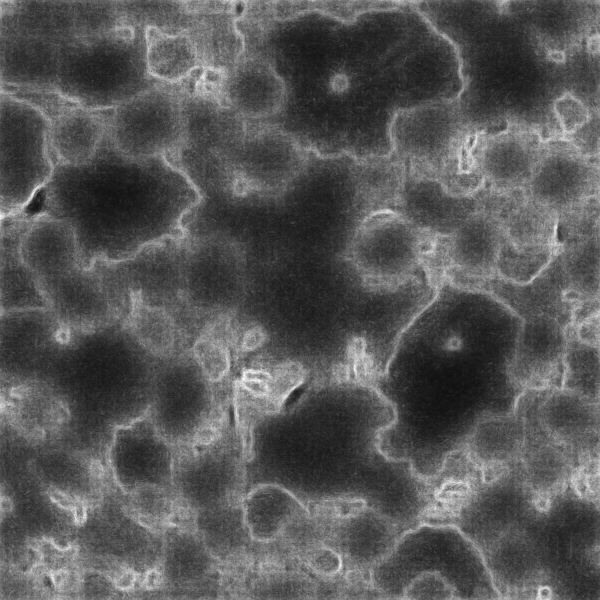}
\end{subfigure}%
\hfill
\begin{subfigure}[t]{.14\textwidth}
  \includegraphics[width=\linewidth]{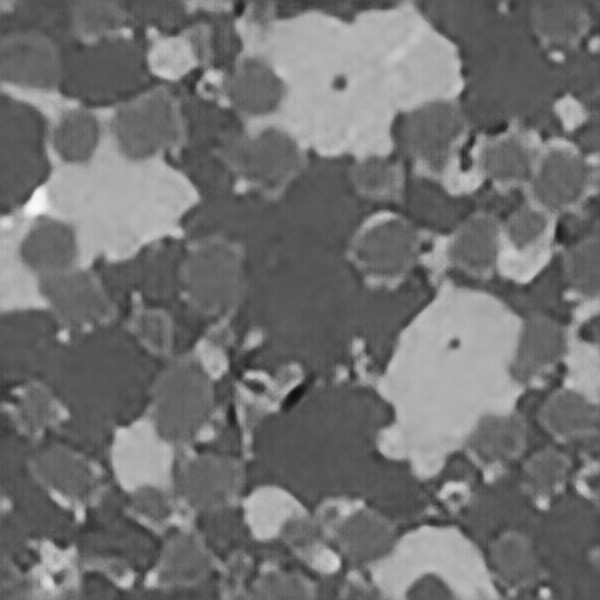}
\end{subfigure}%

\begin{subfigure}[t]{.14\textwidth}
  \includegraphics[width=\linewidth]{Results/UQ/x8_sic/hr_zoom.png}
    \caption*{HR image}
\end{subfigure}%
\hfill
\begin{subfigure}[t]{.14\textwidth}
  \includegraphics[width=\linewidth]{Results/UQ/x8_sic/lr_img_x8_zoom.png}
    \caption*{LR image}
\end{subfigure}%
\hfill
\begin{subfigure}[t]{.14\textwidth}
  \includegraphics[width=\linewidth]{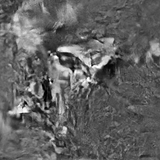}
  \caption*{Prediction 1}
\end{subfigure}%
\hfill
\begin{subfigure}[t]{.14\textwidth}
  \includegraphics[width=\linewidth]{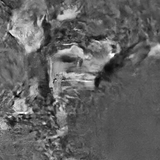}
  \caption*{Prediction 2}
\end{subfigure}%
\hfill
\begin{subfigure}[t]{.14\textwidth}
  \includegraphics[width=\linewidth]{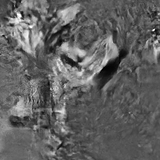}
  \caption*{Prediction 3}
\end{subfigure}%
\hfill
\begin{subfigure}[t]{.14\textwidth}
  \includegraphics[width=\linewidth]{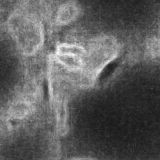}
  \caption*{Standard deviation}
\end{subfigure}%
\hfill
\begin{subfigure}[t]{.14\textwidth}
  \includegraphics[width=\linewidth]{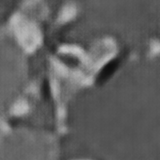}
    \caption*{Mean img}
\end{subfigure}%
\caption{Different WPPFlow (top) and SRFlow (bottom) reconstructions of the ground truth image with stride 8. \textit{Right}: The normalized standard deviation and the mean image.
The zoomed-in part is marked with a white box in the HR image.
} \label{Fig_cINN_x8}
\end{figure}
\section{Conclusion} \label{Section_conclusion}

We introduced WPPNets, which are CNNs trained with a new loss 
function based on comparisons of empirical patch distributions via the quadratic Wasserstein distance and
demonstrated its power by several numerical examples.
In particular, we observed that WPPNets are very stable under inaccurate operators appearing in real-world applications.
Due to the fact that WPPs require the knowledge of one high-resolution reference image, WPPs could also
be interpreted as a method for one-shot learning, see \cite{CHWC2021,WYKN2020} and references therein. However, as no low-resolution
correspondence to the reference image is given, we would call our WPP based methods an unsupervised learning method.
We measured the uncertainty within the reconstructions by combining WPPs with conditional normalizing flows. 

Our considerations could be extended in several directions:
\begin{itemize}
\item So far, we focused  on image superresolution.
However, the same methods can be applied for any inverse problem.
\item Other noise models than the Gaussian one can be easily incorporated.
\item Images of materials microstructures are often three-dimensional. 
Thus, in the future, we would extend the architecture of the WPPNets for three-dimensional images.
\item It is desirable to establish quantitative, mathematical estimates how inaccurate operators influence the results.
\end{itemize}

\section*{Acknowledgements}

Funding by German Research Foundation (DFG) within the project STE 571/16-1 and by the DFG excellence cluster Math+ within the project TP: EF3-7 are gratefully acknowledged.
The data from Section~\ref{sec_synchrotron} and the second example in Section~\ref{sec_inaccurate} has been acquired in the frame of the EU Horizon 2020 Marie Sklodowska-Curie Actions Innovative Training 
Network MUMMERING (MUltiscale, Multimodal and Multidimensional imaging for EngineeRING, Grant Number 765604) at the beamline 
TOMCAT of the SLS by A. Saadaldin, D. Bernard, and F. Marone Welford. We acknowledge the Paul Scherrer Institut, Villigen, 
Switzerland for provision of synchrotron radiation beamtime at the TOMCAT beamline X02DA of the SLS.    
Many thanks
to Antoine Houdard for generating Figure~\ref{fig_vis} and 
to Dang Phoung Lan Nguyen for 
registering the images from the second example in Section~\ref{sec_inaccurate}.
We would like to thank Gabriele Steidl and Paul Hagemann for fruitful discussions.
Moreover, we would like to thank the anonymous reviewers for the thoughtful evaluation  which helped to improve the paper.

\appendix
\section{Proof of Proposition \ref{thm_WPP_dist}}\label{app_proof_wpp_dist}
By \eqref{w_dual}, for each $x \in \R^d$ it holds that
\begin{align*}
W_2^2(\mu_x,\mu_{\tilde{x}})& = \max\limits_{\boldsymbol{\psi} \in \R^{\tilde{N}}}  \big( \frac{1}{N} \sum_{j=1}^N \psi^c (P_j (x)) + \frac{1}{\tilde{N}} \sum_{k=1}^{\tilde{N}} \psi_k \big)\\ &\ge \frac{1}{N} \sum_{j=1}^N \min_{k \in \{ 1,...,\tilde{N} \} } \Vert P_j (x) - P_k (\tilde{x}) \Vert_2^2,
\end{align*}
using $\boldsymbol{\psi} = 0$. Now,  for $x\in \R^d$ and $j\in\{1,...,N\}$,  let $\kappa(j,x)\in \{1,...,\tilde{N}\}$ be one element of 
$\argmin_{k \in \{ 1,...,\tilde{N} \} } \big( \Vert P_j (x) - P_k (\tilde{x}) \Vert_2^2 \big)$.\\ 
Moreover, let $j^*(x)\in\argmax_{j\in\{1,...,N\}}\|P_j (x)\|_\infty$ be a patch which contains the entry of $x$ with the largest absolute value. Then we have
\begin{align*}
W_2^2(\mu_x,\mu_{\tilde{x}}) &\ge \frac{1}{N} \sum_{j=1}^N  \Vert P_j (x) - P_{\kappa (j,x) } (\tilde{x}) \Vert_2^2\\
&\ge \frac{1}{N} \Vert P_{j^*(x)} (x) - P_{\kappa (j^*(x),x) } (\tilde{x}) \Vert_2^2 \\
&= \frac{1}{N} \sum_{l=1}^{p^2} \big((P_{j^*(x)} (x))_l - (P_{\kappa (j^*(x),x) } (\tilde{x}))_l\big)^2,
\end{align*}
where $p$ is the patch size. By considering just the summand $l$ with $(P_{j^*(x)} (x))_l=\|x\|_\infty$ and using that
$(P_{\kappa (j^*(x),x) } (\tilde{x}))_l\leq \|\tilde x\|_\infty$, we obtain
$$
W_2^2(\mu_x,\mu_{\tilde{x}})\geq\frac1N \big(\max(\|x\|_\infty-\|\tilde x\|_\infty,0)\big)^2\geq \frac1N \big(\max(c\|x\|_2-\|\tilde x\|_\infty,0)\big)^2,
$$
for some $c>0$.
Now define the compact set $K = \{ x \in \R^d : \| x \|_2 \le \frac{4}{c}\|\tilde x\|_\infty\}$. 
Then, it holds for all $x \in \R^d \setminus K$ that
\begin{equation}\label{properties_outside_K}
\big(\max(c\|x\|_2-\|\tilde x\|_\infty,0)\big)^2=(c\|x\|_2-\|\tilde x\|_\infty)^2\quad\text{and}\quad \tfrac{c^2}{2}\|x\|_2^2-2c\|x\|_2\|\tilde x\|_\infty \geq 0.
\end{equation}
We can split the integral over $\varphi$ as
$$
\int_{\R^d}|\varphi(x)|dx=\int_{K}|\varphi(x)|dx+\int_{\R^d \setminus K}|\varphi(x)|dx.
$$
Since the Wasserstein distance is non-negative, we have that $|\varphi(x)|\leq 1$. As $K$ is compact, we obtain that the first summand
in the above formula is finite. It remains to show that also the second summand is finite.
Indeed, it holds
\begin{align}
\int_{\R^d \setminus K}|\varphi(x)|dx&=\int_{\R^d \setminus K}\exp(-\rho W_2^2(\mu_x,\mu_{\tilde x}))dx\\
&\leq\int_{\R^d \setminus K}\exp\Big(-\frac{\rho}{N} \big(\max(c\|x\|_2-\|\tilde x\|_\infty,0)\big)^2\Big)dx
\end{align}
Since we are integrating over all $x \in \R^d \setminus K$, we obtain by \eqref{properties_outside_K} that
\begin{align}
\int_{\R^d \setminus K}|\varphi(x)|dx&\leq\int_{\R^d \setminus K}\exp\Big(-\tfrac{\rho}{N} \big(c^2\|x\|_2^2-2c\|x\|_2\|\tilde x\|_\infty+\|\tilde x\|_\infty^2\big)\Big)dx\\
&\leq\int_{\R^d \setminus K}\exp\big(-\tfrac{\rho}{2N} c^2\|x\|_2^2)\exp\Big(-\tfrac{\rho}{N} \big(\underbrace{\tfrac{c^2}{2}\|x\|_2^2-2c\|x\|_2\|\tilde x\|_\infty}_{\geq 0 \,\, \text{by \eqref{properties_outside_K}}}\big)\Big)dx\\
&\leq\int_{\R^d \setminus K}\exp\big(-\tfrac{\rho}{2N} c^2\|x\|_2^2)dx<\infty.
\end{align}
This finishes the proof.\hfill$\square$

\section{Implementation details of WPPNets}\label{sec_implementation_details}

All experiments are implemented in PyTorch. We run them on a single NVIDIA GeForce RTX 2060 GPU with 6 GB GPU memory.

\paragraph{Network Architecture}

We use a 16-layer CNN $G_\theta$ which is adapted from \cite{Tian21}.  
In \cite{Tian21} the authors propose a so-called asymmetric CNN (ACNN) for image super-resolution consisting of 23-layers. 
More specifically, the ACNN has a 17-layer asymmetric block, a 1-layer memory enhancement block and a 5-layer high-frequency 
feature enhancement block (for more details about the structure and the tasks of the individual blocks, see \cite[Section III]{Tian21}). As stated above, we modified the proposed ACNN and take a 10-layer asymmetric block 
(instead of 17-layer) in order to reduce the network complexity.

\paragraph{Training Details}

For the training of the WPPNet we use the Adam optimizer \cite{KB2015} with a learning rate of $0.0001$. The training and test data are independently chosen; we used 1000 low-resolution images of size $25 \times 25$ for training the network and have two pairs of validation images of size $600 \times 600$ and $150 \times 150$ for the high- and low-resolution image, respectively. For the training process the batch size is set to 25 and the number of epochs we trained for the reconstruction image is stated in Table ~\ref{table_epochs}. We choose the patch size 
to be $p=6$, i.e., $P_i (x)$ is a small sub-image of size $6 \times 6$ of an image $x$. 
We subsample the number of patches in the reference image to $|I|=10000$ accordingly to Remark~\ref{rem_computational_complexity}.

\begin{table}[t]
\begin{center}
\scalebox{.85}{
\begin{tabular}{c|ccccccc} 
          & Figure \ref{Comp_HRLRPred_texture_grass} & Figure \ref{Comp_HRLRPred_texture_floor} & Figure \ref{Comp_HRLRPred_SiC}  & Figure \ref{Comp_HRLRPred_FS}  & Figure \ref{fig_magnificationx6}  & Figure \ref{Fig_inaccurateforward} & Figure \ref{fig_estimatedforward} \\
\hline
epochs    & 420  & 270 & 450 & 420 & 570 & 420 & 150  \\
\end{tabular}}
\caption{Number of epochs to obtain the reconstructions visualized in the respective Figure.}                    
\label{table_epochs}
\end{center}
\end{table}

To obtain an approximation of the maximizer $\psi^{*}$ of $F$ we use $20$ 
iterations (except for Figure \ref{Comp_HRLRPred_texture_grass}, there we used 10 iterations) of a stochastic gradient ascent with a learning rate of $1$. 
Moreover, instead of starting with an arbitrary $\psi_{k}^{0}$ or choosing $\psi_{k}^{0} = 0$ for the optimization in epoch $k$, 
we save the approximated maximizer $\psi_{k-1}^{20}$ from the previous epoch $k-1$ and use it as 
the starting vector in epoch $k$, i.e., $\psi_{k-1}^{20} = \psi_k^{0}$. 
Herewith we reach a better approximation of the maximizer $\psi_k^{*}$ in a computationally efficient way.

\paragraph{Hyperparameter selection}

As for any regularized problem the hyperparameter $\lambda$ in \eqref{eq_loss} has to be chosen carefully.
Since the reference image is the only given high-resolution image, we us it as a validation image and perform a grid search of $\lambda$ on it. 
That is, we generate a synthetic low-resolution observation of the reference image by computing $\tilde y=f(\tilde x)+\xi$ for some noise $\xi$ and choose $\lambda$ such that the reconstruction quality with respect to the PSNR of $\tilde x$ from $\tilde y$ is optimal.

\section{Implementation Details for WPPFlows}\label{sec_WPPFlow_architecture}

\paragraph{Network Architecture} The conditional normalizing flow is adapted from \cite{Lugmayr20}. Here the authors propose a multiscale normalizing flow with 3 and 4 scales for a magnification factor 4 and 8, respectively, consisting of 16 flow steps, followed by a transition step for learning a better transition between the scales. Moreover, as a conditioning network a standard 23-block RRDB architecture \cite{WYWGLDQL19} is used to extract features from the given low-resolution image.

We modified the network to reduce the complexity. In particular, we do not use a conditioning network, but we use the low-resolution image itself and their bicubic interpolations for the respective scale. 
The input from the latent space $z \sim p_z$ is of the same size as the high-resolution reconstruction and then invertible reshaped to the size of the low-resolution. The downsample step is taken from \cite{Denker21}, which consists of an invertible reshaping, followed by a GlowCoupling block and an Actnorm layer. Then, as in \cite{Lugmayr20}, 16 and 10 flow steps and a transition step follow for magnification factor 4 and 8, respectively. Lastly, a conditional affine transform in the high-resolution scale is used. Note that we do not use a splitting of 50 \% in the channel dimension and we only used 2 and 3 scales for a magnification factor 4 and 8, respectively.

\paragraph{Training Details} Similar to WPPNet, we used the Adam optimizer \cite{KB2015} with a learning rate of $0.0001$. While 1000 low-resolution images of size $25\times 25$ are used for training WPPFlow with a magnification factor 4, for the magnification factor 8 we used 270 low-resolution images of size $20 \times 20$. For the training process the batch size is set to 10, the patch size is chosen to be 6 and both networks are trained for 450 epochs. We used the regularization parameter $\lambda = 100$ and the maximizer $\psi^*$ is computed similar to WPPNets.
We subsample the number of patches in the reference image to $|I|=10000$ accordingly to Remark~\ref{rem_computational_complexity}.
The selection of the hyperparameter $\lambda$ is done similar as for the WPPNets.

\section{Evaluation on a Larger Test Set}\label{sec_further_exp}

In order to make the experiments from Section~\ref{sec_synchrotron} more reliable, we apply the different methods onto 
a larger test set.
Three examplar images from the test set are given in Figure \ref{fig_additional_ground_truth}. 
The average of the errors are given in Table \ref{table_averagederrorMeasures}. 
Again, the WPPNet and WPP perform better than the other methods in terms of the considered quality measures. 
Overall, the results are comparable with the results from Section~\ref{sec_synchrotron}.

\begin{figure}[t]
\centering
\begin{subfigure}{0.33\textwidth}
  \centering
  \includegraphics[width=\linewidth]{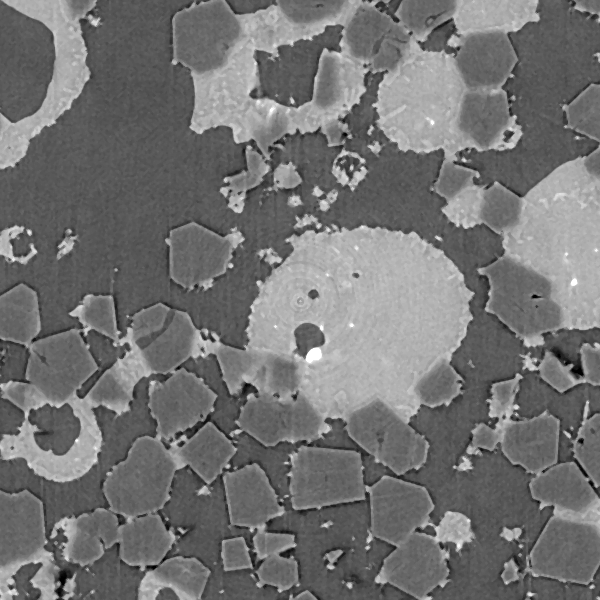}
\end{subfigure}%
\hfill
\begin{subfigure}{0.33\textwidth}
  \centering
  \includegraphics[width=\linewidth]{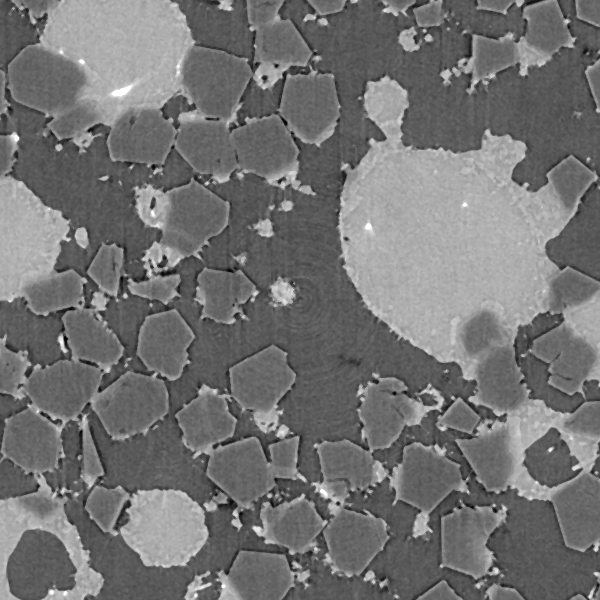}
\end{subfigure}%
\hfill
\begin{subfigure}{0.33\textwidth}
  \centering
  \includegraphics[width=\linewidth]{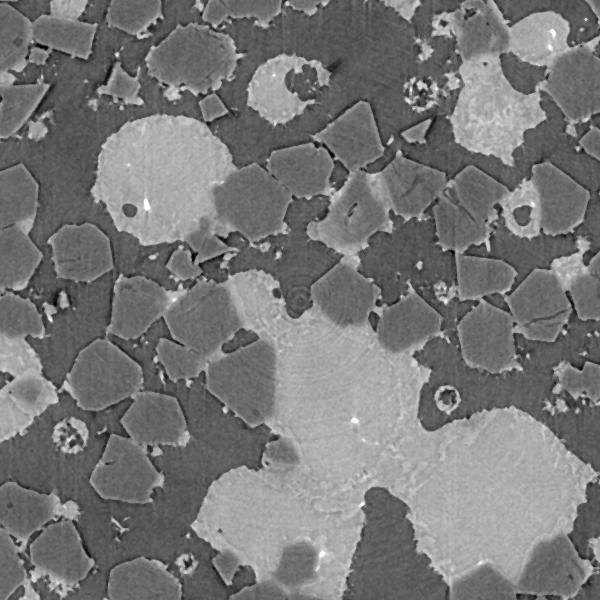}
\end{subfigure}%
\caption{Exemplary ground truth images used for further numerical examples.} \label{fig_additional_ground_truth}
\end{figure}

\begin{table}[t]
\begin{center}
\scalebox{.85}{
\begin{tabular}[t]{c|c|cccccc} 
             &              & bicubic & PnP-DRUNet& DIP+TV    & ACNN  & WPP    & WPPNet \\
\hline
             & PSNR         &  25.70  & 27.72  & \underline{\textbf{28.04}}  & \textbf{28.00}   & 27.73      &  \textbf{28.00} \\ 
SiC & Blur Effect  &  0.5342 & 0.4455 & 0.3967 & 0.3783  & \underline{\textbf{0.3596}}      &  \textbf{0.3663} \\
             & LPIPS        &  0.4110 & 0.3265 & 0.2141 & 0.2272  & \underline{\textbf{0.1591}}      &  \textbf{0.1830}  \\
             & SSIM        & 0.6944  & 0.7630 & \textbf{0.7677} & \underline{\textbf{0.7752}} & 0.7520 & 0.7654   \\
             & FSIM        & 0.8596  & 0.7903 & 0.8788 & 0.8267 & \underline{\textbf{0.9365}} & \textbf{0.9309} \\
             
\end{tabular}}
\caption{Averaged PSNR, Blur Effect and LPIPS value of the high-resolution reconstructions. The best two values are marked in bold, the best one is additionally underlined.}                    
\label{table_averagederrorMeasures}
\end{center}
\end{table}

\section{Estimation of the Forward Operator}\label{sec_estimation_of_forward}

In the following, we describe the estimation process of the forward operator $f$ for superresolution, which is used
in the second example of Section \ref{sec_inaccurate}. For the estimation, we assume that we have given a registered pair 
$(\tilde x,\tilde y)$ of a high-resolution and a low-resolution image and follow the lines of \cite{Hertrich21}.

We assume that our forward operator is given by $f(x)=S(k*x+b)$ for a $15\times 15$ blur kernel $k$, a bias $b\in\R$ 
and a downsampling operator $S$.

\paragraph{Definition of the downsampling operator}
Further, for the downsampling operator $S$, we make use of Fourier transforms.
Given an image $x\in\R^{n_x,n_y}$ the two-dimensional discrete Fourier transform (DFT)
is defined by $\mathcal F_{n_x,n_y}\coloneqq \mathcal F_{n_x}\otimes \mathcal F_{n_y}$, 
where we have $\mathcal F_n=(\exp(-2\pi i k l/n))_{k,l=0}^{n-1}$. Now, the downsampling operator $S\colon\R^{m_x,m_y}\to\R^{n_x,n_y}$ 
is given by
$$
S=\frac{n_xn_y}{m_xm_y}\mathcal F_{n_x,n_y}^{-1} D \mathcal F_{m_x,m_y},
$$
where for $x\in\C^{m_x,m_y}$ the $(i,j)$-th entry of $D(x)$ is given by
$x_{i',j'}$, where
\begin{align}
i'=\begin{cases}i, &$if $i\leq\frac{n_x}{2},\\i+m_x-n_x, &$otherwise.$\end{cases}
\end{align}
and $j'$ is defined analogously. 
Thus, the operator $S$ generates a downsampled version $S(x)$ of an image $x$ by removing the high-frequency part from $x$.
Note that even if the Fourier matrix $\mathcal F_{n_x,n_y}$ is complex valued, the range of $S$ is real-valued, 
as $D$ preserves Hermitian-symmetric spectra.

\paragraph{Estimation of Blur Kernel and Bias}

We assume that we have given images $\tilde x\in\R^{m_x,m_y}$ and $\tilde y\in\R^{n_x,n_y}$ related by $\tilde y\approx S(k*\tilde x+b)$, where the blur kernel $k\in\R^{15\times15}$ and the bias $b\in\R$ are unknown.
In the following, we aim to reconstruct $k$ and $b$ from $\tilde x$ and $\tilde y$. 
Here, we use the notations $N=n_xn_y$ and $M=m_xm_y$. Further let $\tilde k\in\R^{m_x,m_y}$ be the kernel $k$ padded with zeros such that it still corresponds to the same convolution as $k$, but has size $m_x\times m_y$.

Applying the DFT on both sides of $y=S(k*\tilde x+b)=S(\tilde k*\tilde x+b)$ and using the definition of $S$, we obtain that
$$
\hat y =\frac{N}{M}D(\hat k\odot \hat x + M b e)=\frac{N}{M}D(\hat k)\odot D(\hat x)+Nbe,
$$
where $\hat y=\mathcal F_{n_x,n_y}\tilde y$, $\hat x=\mathcal F_{m_x,m_y}\tilde x$, $\hat k=\mathcal F_{m_x,m_y}\tilde k$, $\odot$ is the elementwise product and $e$ denotes the first unit vector (i.e.,\ $e_{0,0}=1$ and all other entries are zero). Now, we can conclude that
$$
D(\hat k)=\frac{M}{N}\hat y\oslash D(\hat x)-  \frac{M b}{\hat x_{0,0}} e,
$$
where $\oslash$ is the elementwise quotient. In practice, we stabilize this quotient by increasing the absolute value of $D(\hat x)$ by $10^{-5}$ while retaining the phase.
Thus, assuming that the high-frequency part of $k$ is negligible (i.e., that $D^T D k=k$), we can approximate $\hat k$ by
$$
\hat k\approx \frac{M}{N} D^\tT \hat y \oslash D(\hat x) -  \frac{Mb}{\hat x_{0,0}} e.
$$
Applying the inverse DFT this becomes
$$
\tilde k\approx \mathcal F_{m_x,m_y}^{-1}\Big(\frac{M}{N} D^\tT (\hat y \oslash D(\hat x))\Big) - \frac{ b}{\hat x_{0,0}}.
$$
Using the assumption that $\tilde k$ is zero outside of the $15\times15$ patch, where $k$ is located, we can estimate $b$ by taking the mean over all pixels of $\mathcal F_{m_x,m_y}^{-1}\Big(\frac{M}{N} D^\tT (\hat y \oslash D(\hat x))\Big)$ outside of this $15\times15$ patch.
Afterwards, we estimate $k$ by reprojecting
$$
\mathcal F_{m_x,m_y}^{-1}\Big(\frac{M}{N} D^\tT (\hat y \oslash D(\hat x))\Big) - \frac{ b}{\hat x_{0,0}}
$$
to the set of all real $15\times15$ kernels.

\bibliographystyle{abbrv}
\bibliography{literatur_new}
\end{document}